\documentclass[acmsmall, screen, nonacm]{acmart}

\AtBeginDocument{%
  }

\usepackage{hyperref}       % hyperlinks
\usepackage{url}            % simple URL typesetting
\usepackage{booktabs}       % professional-quality tables
\usepackage{amsfonts}       % blackboard math symbols
\usepackage{nicefrac}       % compact symbols for 1/2, etc.
\usepackage{microtype}      % microtypography
\usepackage{color, colortbl}
\usepackage{xcolor}       % colors
\usepackage[capitalize]{cleveref}
\usepackage[skip=4ex]{caption}

\definecolor{red}{HTML}{E51400} %red
\definecolor{blue}{HTML}{0050EF} %cobalt
\definecolor{green}{HTML}{008A00} %emerald
\definecolor{purple}{HTML}{AA00FF} %violet
\definecolor{orange}{HTML}{FF7F00}
\definecolor{gray}{HTML}{848482}
\definecolor{Gray}{gray}{0.85}
\definecolor{LightGray}{gray}{0.96}

\usepackage{csquotes}

\usepackage{algpseudocode}

\usepackage{nicematrix}

\usepackage{url}            % simple URL typesetting
\usepackage{booktabs}       % professional-quality tables
\usepackage{amsfonts}       % blackboard math symbols
\usepackage{nicefrac}       % compact symbols for 1/2, etc.
\usepackage{microtype}      % microtypography
\usepackage[symbol]{footmisc}

\usepackage{amsthm}
\usepackage{mathtools}
\usepackage{chngcntr}
\usepackage{apptools}
\newtheorem{theorem}{Theorem}
\newtheorem*{theorem*}{Theorem}
\newtheorem{lemma}[theorem]{Lemma}

\newtheorem*{remark}{Remark}
\newtheorem*{lemma*}{Lemma}
\newtheorem{Proposition}{Proposition}
\newtheorem*{Proposition*}{Proposition}

\newtheorem{definition}{Definition}
\usepackage{graphicx}
\usepackage{graphics}
\usepackage{parskip}
\usepackage{amsmath}
\usepackage[ruled,vlined, linesnumbered, algo2e]{algorithm2e}
\usepackage{algorithmicx}  
\usepackage{algorithm}

\usepackage[caption=false,font=footnotesize,captionskip=-1.5mm,farskip=5mm]{subfig}

\usepackage{dsfont}
\usepackage{natbib}
\usepackage[leftcaption]{sidecap}

\usepackage[flushleft]{threeparttable}
\usepackage{array,booktabs,makecell}

\usepackage{amsmath}
\usepackage[flushleft]{threeparttable}
\usepackage{array,booktabs,makecell}
\usepackage[font=small]{caption}

\usepackage{wrapfig}
\usepackage[normalem]{ulem}
\usepackage{soul}
\usepackage[normalem]{ulem}               % to striketrhourhg text
\newcommand\redout{\bgroup\markoverwith
{\textcolor{red}{\rule[0.5ex]{2pt}{0.8pt}}}\ULon}

\makeatletter
\renewcommand{\Indentp}[1]{%
  \advance\leftskip by #1
  \advance\skiptext by -#1
  \advance\skiprule by #1}%
\renewcommand{\Indp}{\algocf@adjustskipindent\Indentp{\algoskipindent}}
\renewcommand{\Indm}{\algocf@adjustskipindent\Indentp{-\algoskipindent}}
\makeatother

\newcommand{\upbra}[1]{^{(#1)}}

\renewcommand{\le}{\leqslant}
\renewcommand{\geq}{\geqslant}
\renewcommand{\leq}{\leqslant}

\DeclareMathOperator*{\argmax}{arg\,max}

\usepackage{xspace}
\usepackage{fontawesome5}

\usepackage{todonotes}

\usepackage{threeparttable, array, float} % for adding note at the end of a table

\begin{document}

\title[Heterogeneous Multi-agent Multi-armed Bandit on Stochastic Block Models]{%Cooperative Multi-agent Multi-armed Bandit on Stochastic Block Models
  \texorpdfstring{Heterogeneous Multi-agent Multi-armed Bandits \\on Stochastic Block Models}{Heterogeneous Multi-agent Multi-armed Bandits on Stochastic Block Models}
  %\\
  %Multi-Armed Bandits with Multiple Clusters of Agents 
}


  \author{Mengfan Xu}
    \affiliation{%
    \institution{Department of Mechanical and Industrial Engineering, University of Massachusetts Amherst}
    \city{Amherst}
    \state{MA}
    \country{USA}}
  \email{mengfanxu@umass.edu}
  
  \author{Liren Shan}
  \email{lirenshan@ttic.edu}
  \affiliation{%
    \institution{Toyota Technological Institute at Chicago}
    \city{Chicago}
    \state{IL}
    \country{USA}
  }

  \author{Fatemeh Ghaffari }
  \affiliation{%
    \institution{Manning College of Information \& Computer Sciences, University of Massachusetts Amherst}
    \city{Amherst}
    \state{MA}
    \country{USA}}
  \email{fghaffari@umass.edu}

  \author{Xuchuang Wang}
  \affiliation{%
    \institution{Manning College of Information \& Computer Sciences, University of Massachusetts Amherst}
    \city{Amherst}
    \state{MA}
    \country{USA}}
  \email{xuchuangw@gmail.com}

  \author{Xutong Liu}
  \affiliation{%
    \institution{School of Computer Science, Carnegie Mellon University}
    \city{Pittsburgh}
    \state{PA}
    \country{USA}}
  \email{xutongl@andrew.cmu.edu}

  \author{Mohammad Hajiesmaili}
  \affiliation{%
    \institution{Manning College of Information \& Computer Sciences, University of Massachusetts Amherst}
    \city{Amherst}
    \state{MA}
    \country{USA}}
  \email{hajiesmaili@cs.umass.edu}


\begin{abstract}

We study a novel heterogeneous multi-agent multi-armed bandit problem with a cluster structure induced by stochastic block models, influencing not only graph topology, but also reward heterogeneity. Specifically, agents are distributed on random graphs based on stochastic block models - a generalized Erdos-Renyi model with heterogeneous edge probabilities: agents are grouped into clusters (known or unknown); edge probabilities for agents within the same cluster differ from those across clusters. In addition, the cluster structure in stochastic block model also determines our heterogeneous rewards. Rewards distributions of the same arm vary across agents in different clusters but remain consistent within a cluster, unifying homogeneous and heterogeneous settings and varying degree of heterogeneity, and rewards are independent samples from these distributions. The objective is to minimize system-wide regret across all agents. To address this, we propose a novel algorithm applicable to both known and unknown cluster settings. The algorithm combines an averaging-based consensus approach with a newly introduced information aggregation and weighting technique, resulting in a UCB-type strategy. It accounts for graph randomness, leverages both intra-cluster (homogeneous) and inter-cluster (heterogeneous) information from rewards and graphs, and incorporates cluster detection for unknown cluster settings. We derive optimal instance-dependent regret upper bounds of order $\log{T}$ under sub-Gaussian rewards. Importantly, our regret bounds capture the degree of heterogeneity in the system (an additional layer of complexity), exhibit smaller constants, scale better for large systems, and impose significantly relaxed assumptions on edge probabilities. In contrast, prior works have not accounted for this refined problem complexity, rely on more stringent assumptions, and exhibit limited scalability.

\end{abstract}

%\keywords{Multi-armed Bandit, Multi-agent Systems, Stochastic Block Model}
%\includegraphics[width=\textwidth]{sampleteaser}
%\caption{Seattle Mariners at Spring Training, 2010.}
%\Description{Enjoying the baseball game from the third-base
%seats. Ichiro Suzuki preparing to bat.}
%\label{fig:teaser}

\maketitle
\section{Introduction}
Multi-armed Bandit (MAB) \citep{auer2002finite, auer2002nonstochastic} is an online learning framework in which, during a sequential game, an agent, or decision maker, selects one arm from multiple arms, pulls the arm, and receives the reward observation of the pulled arm from an unknown environment at each time step. The objective is to maximize the cumulative received reward by identifying the best arm, a task also known as regret minimization when compared to the ideal case of knowing in advance which arm is the best. Recently, with the rapid development of real-world networks, multi-agent systems have become a major focus, motivating the study of Multi-agent Multi-armed Bandit (MA-MAB) \citep{xu2023, wangx2022achieving, landgren2021distributed, bistritz2018distributed,zhu2021federated,huang2021federated,mitra2021exploiting,reda2022near,yan2022federated}. In this context, multiple agents %\mo{is ``client'' chosen intentionally here? if not, I prefer agent since agents sounds more active entities but agents are passive, i.e., not decision maker.} 
exist within a system, each with its own arm set, playing a bandit game while communicating with others to exchange bandit information. It is well known that MAB is classified into stochastic (where reward observations come from a time-invariant distribution with reward mean values) \citep{auer2002finite} and adversarial (where reward observations are arbitrary) \citep{auer2002nonstochastic}. Here, we focus exclusively on the stochastic setting, consistent with the majority of existing work on MA-MAB. For simplicity, we refer to stochastic MA-MAB as MA-MAB for the remainder of this paper.

Depending on the application domain and environment properties, previous work studies several variants of MA-MAB settings. Among all variants, a widely studied one is a cooperative setting, where agents share the same arm set and aim to maximize the overall system's objective. Depending on how the rewards are generated for different agents, the cooperative MA-MAB can be further categorized into homogeneous and heterogeneous. In a homogeneous setting, the reward mean value of the same arm across different agents is identical. This implies that the locally optimal arm (with respect to an agent's own reward distribution) is also the globally optimal arm (with respect to the average reward mean values of the same arm across all agents). This scenario has been extensively studied \citep{landgren2016distributed,landgren2016distributed_2,landgren2021distributed,zhu2020distributed,martinez2019decentralized,agarwal2022multi, wangx2022achieving, wangp2020optimal, li2022privacy, sankararaman2019social, chawla2020gossiping}. However, it is common that in real applications, the agents' rewards are heterogeneous. For example, retail companies in different regions may have varying product return rates due to population heterogeneity. To address this, a line of research has focused on the heterogeneous setting \citep{xu2023decentralized, zhu2021federated, zhu2023distributed}, where the globally optimal arm can differ from the locally optimal arm. Nonetheless, these prior works assume a fully heterogeneous setting, treating all agents as distinct. It is important to note that, in practice, the setting is not necessarily fully heterogeneous; instead, different degrees of heterogeneity can exist within a heterogeneous setting. This concept, however, has not been well-defined or thoroughly studied, presenting a clear research gap.
 
Another central aspect of MA-MAB is how agents communicate. In a decentralized setting, agents are distributed on a graph (as vertices connected by edges) and can only communicate if an edge exists between them. In a sequential regime, while time-invariant graphs have been well studied, the advancement of several applications, e.g., wireless IoT networks, motivates the study of time-varying graphs, which introduces additional challenges. Random graphs \citep{erdds1959random} have been a promising approach to model time-varying graphs in MA-MAB \citep{xu2023decentralized, dubey2020cooperative} and other areas \citep{delarue2017mean, lima2008majority}, where the graph is randomly drawn from a distribution by sampling each edge based on a probability, akin to the reward generation process. However, current research in MA-MAB \citep{xu2023decentralized} is largely limited to Erdos-Renyi models~\citep{erdds1959random}, where the edge probability is homogeneous across all agents. Additionally, the graph generation process is assumed to be independent of the reward generation process, which may not always reflect practical scenarios. These limitations in the studied random graph models for MA-MAB highlight another important research gap.

Notably, a broad family of random graphs is formulated as Stochastic Block Models (SBM) \citep{abbe2015exact, abbe2018community, deshpande2018contextual}, where agents are grouped into clusters, and the edge probability for agents within the same cluster differs from that for agents in different clusters. The existence of cluster structures, which generalize the Erdos-Renyi models as commonly used in MA-MAB, surprisingly but naturally provides a framework for defining the degree of heterogeneity through graph topology. Consequently, considering SBM in the context of MA-MAB holds significant potential for addressing the aforementioned research gaps. However, SBM has not yet been explored in the cooperative learning context, particularly in MA-MAB, which motivates our work herein. 

%\mo{I would suggest moving the below paragraph to after the last paragraph in the contribution and highlight it as additional and independent contribution. }

In this paper, we study the following research problem: \emph{Can we address the heterogeneous multi-agent multi-armed bandit problem on Stochastic Block Models to bridge the following two gaps: 1) varying degrees of heterogeneity and 2) more general random graphs linked to reward dynamics?}
%\mo{if you moved the paragraph, remove (3). }

\subsection{Main Contributions}
We provide an affirmative answer to the above question through our main contributions, summarized as follows. First, we formulate a general heterogeneous multi-agent multi-armed bandit problem, where agents are grouped into clusters inspired by the stochastic block model. This cluster structure determines both graph topology and reward similarity. Specifically, the cluster structure introduces heterogeneity in edge probabilities and reward distributions across clusters, while maintaining homogeneity in edge probabilities and reward distributions within clusters, thereby linking reward dynamics with graph dynamics. This framework thus extends random graphs (with homogeneous edge probabilities) from Erdos-Renyi models to general stochastic block models (with heterogeneous edge probabilities). It also incorporates both homogeneity and heterogeneity in reward distributions into a unified setting, characterizing the degree of heterogeneity, which reflects an additional layer of complexity in heterogeneous settings. Existing work on either homogeneous or heterogeneous rewards can be viewed as special cases of this framework, demonstrating its consistency and generalization capability. A more detailed comparison with existing models is provided in Section \ref{sec:2.2}.

%\mo{generally the rest of this section needs one more careful pass to add more concrete statements, also if space is not constrained, i would suggest adding 1-2 paragraphs on technical novelties of the proposed algorithms.}

\textcolor{black}{Secondly, we propose a learning algorithm tailored to this new formulation that effectively leverages homogeneity to reduce the sample complexity of reward estimations and heterogeneity to efficiently learn the global objective. Specifically, we propose an algorithm, namely UCB-SBM, which consists of a burn-in period to collect local information and a learning period to leverage historical data to improve the arm-pulling strategy. During the burn-in period, the agents randomly pull arms to obtain reward estimators for each arm at an individual level. Then, during the learning period, the agents use a UCB-type strategy to pull the arm with the highest UCB index based on exchanged information and a new weighting technique. They communicate newly designed information to other agents, perform information updates based on newly proposed rules, and novelly run cluster detection using rewards as side information when the cluster structure is unknown. }

\textcolor{black}{Our algorithmically technical novelty is as follows.   Compared to the most relevant work \citep{xu2023decentralized}, our UCB-type strategy employs a newly constructed global estimator that integrates both homogeneity (inter-cluster estimators) and heterogeneity (intra-cluster estimators). Additionally, our information transmission involves sending cluster-level estimators instead of individual-level ones, thereby 1) achieving noise reduction in the estimators, and 2) relaxing the assumption on the edge probability, as it requires only one representative in the cluster to exchange cluster-level information, rather than requiring all agent pairs to communicate. The information update process is significantly different in that we construct estimators at the cluster level using a newly proposed weighted sum/average approach, and compute the global information based on the new weight technique over the cluster-level estimators, resulting in three layers of estimators: local, cluster, and global. In contrast, \citep{xu2023decentralized} considers only local and global layers. Lastly, the incorporation of a cluster detection method enables us to infer general unknown cluster structures and thus to leverage the cluster structure to design algorithms, which is completely omitted in \citep{xu2023decentralized}.} %\mo{make another pass and add more concrete ideas and intuitions behind the algorithm.}

\textcolor{black}{Thirdly, we establish precise instance-dependent regret upper bounds for the UCB-SBM algorithm, which are of order \(O(\log{T})\). Additionally, if we examine the coefficient of the regret bounds more closely, our regret bound accurately captures the relationship between the regret upper bounds and the newly defined degree of heterogeneity, \(\nicefrac{C}{M}\), which is the ratio between the number of clusters \(C\) and the number of agents \(M\). Specifically, the regret bound depends linearly on \(\nicefrac{C}{M}\), reflecting the problem's complexity. Moreover, this implies that the total regret depends on \(C \leq M\) instead of \(M\), scaling significantly better with the number of agents \(M\) in large-scale systems. In contrast, the existing algorithm for heterogeneous rewards results in a regret upper bound of order \(M^2\) \citep{xu2023decentralized}, as it neglects possible cluster structures. This bound may become unmanageable when the number of agents is comparable to the time horizon \(T\).} %\mo{add regret bounds if it is not too complicated. and make statements more concrete.}

\textcolor{black}{Fourthly, our results do not rely on strong assumptions about edge probabilities, making them more broadly applicable compared to prior work. More specifically, the lower bound on the edge probability in our case can be at most $
\frac{e}{e-1} \cdot \frac{C^2}{M^2} \cdot \frac{(C-l-1)!}{(C-2)!} < 1
 $ for $1 \leq l \leq C-1$, while the lower bound on the edge probability in \citep{xu2023decentralized} approaches $1$ as $T \to \infty$. Furthermore, this lower bound in~\citep{xu2023decentralized} increases more rapidly with $M$, as it depends on $M$ whereas ours only depends on $C \leq M$, and a larger lower bound implies more stringent assumptions on the problem setting as $M$ grows. Overcoming the limited scalability of problem setting and obtaining a lower bound on the edge probability that is strictly bounded away from $1$ is a highly non-trivial yet impactful contribution.
 }
 Fifth, our results apply to scenarios with both known and unknown cluster assignments, aligning with existing work on Stochastic Block Models. A comprehensive summary of the theoretical results is presented in Table \ref{tab:results}. Lastly, through experiments, we demonstrate that our algorithm achieves much lower actual regret (beyond regret bounds), with an improvement of at least $68.69\%$, highlighting its superior practical effectiveness.

Additionally, we make an independent contribution as follows. The regret bound under the new framework should reflect the degree of heterogeneity, which is essentially highlighted as an open problem in \citep{xu2023decentralized}. In that work, the authors numerically observe a dependency of regret on the level of heterogeneity in the problem setting, noting that regret increases monotonically with the level of heterogeneity. This suggests the potential to achieve smaller regret when the degree of heterogeneity is low. However, they do not formally define or analyze this dependency theoretically, leaving a research gap. Moreover, the results in \citep{xu2023decentralized} heavily rely on an assumption about the lower bound on the edge probability in the Erdos-Renyi model, which can become quite restrictive when $T$ is sufficiently large, thereby limiting its practical applicability. \textcolor{black}{How to relax these stringent assumptions remains unexplored, necessitating the development of new methods and analyses—another research gap that our work seeks to address.}

\paragraph{Paper Organization} The paper is presented as follows. We provide a comparison of our work with existing studies in Section \ref{sec:2.2}. In Section \ref{sec:formulation}, we introduce the notations and formulate the research question. In Section \ref{sec:app}, we provide the motivation for the formulation by highlighting some important real-world applications of the problem setting. Section \ref{sec:homo} begins by characterizing the framework through a simple case involving a single cluster, where the formulation reduces to homogeneous rewards on random graphs. Subsequently, in Section \ref{sec:heter}, we extend the framework to the main case involving multiple known clusters, presenting the proposed algorithm and its analysis (with improved regret bounds) under milder assumptions compared to existing work. In Section \ref{sec:heter-unknown}, we illustrate how the algorithm can be adapted to scenarios with multiple unknown clusters. Section \ref{sec:exp} demonstrates the numerical performance of the proposed algorithm. Last, we conclude the paper and suggest future research directions in Section~\ref{sec:conclusion}.

\iffalse
Point of considering stochastic block models
\begin{itemize}
    \item homogeneous settings are common in real-world and well studied 
    \item homogeneous settings fail to capture heterogeneity 
    \item existing work on the heterogeneous setting (fully heterogeneous)
    \item in practice, not necessarily fully heterogeneous
    \item different degree of heterogeneity can be possible in real world and needs to be defined and the regret bound should reflect the degree of heterogeneity
    \item this is essentially pointed out as an open problem in \citep{xu2023decentralized} where the authors numerically claim the dependency of the regret on heterogeneity, without formally defining and analyzing it.  
    \item Our result also solves an open problem as pointed out in \citep{xu2023decentralized} 
\end{itemize}
\fi 

\begin{table*}[t]
  \caption{Summary of the main results for UCB-SBM. }\label{tab:results}\vspace{-3mm}
  \centering
  \resizebox{1\columnwidth}{!}{
    \centering
    \begin{threeparttable}
      \begin{tabular}{|cccccc|}
        \hline
        % \multicolumn{5}{}{}
        % & \textbf{Summary }\\
        % \hline\hline
\textbf{Cluster} & \textbf{Thm.}                       & \textbf{Asm. 1$^*$} & \textbf{Asm. 2$^*$}   & \textbf{Worst-case$^{\S}$}                         & \textbf{Coef.$^{\ddagger}$}                                                                  \\
        \hline  
        known; $C=1$ & \ref{thm:homo-ucb} 
        & $p \in (0, 1] $ 
        & $q = p$ & $N/A$  
        & $O(
        \frac{1}{M})^{\ddagger}$                                   \\
        known/unknown$^\P$ & \ref{thm:3} & $p = 1 $ &   $q > 1 - (\frac{1}{2} - \frac{1}{2}\sqrt{1 - (\frac{\delta}{CT})^{\frac{2}{C-1}}})^{C^2/M^2} \qquad \quad $ & $1$ &  $O(\frac{C}{M})$  \\                                              known/unknown$^\P$ & \ref{thm:4} & $p = 1 $ & $q > 1-(1- \min\{(\frac{1}{2} + \frac{1}{2}\sqrt{1 - (\frac{\delta}{CT})^{\frac{2}{C-1}}}), 1 - \nicefrac{\delta (C-1)}{8CT}\})^{\nicefrac{C^2}{M^2}}$ & $1$  & $O(\frac{C}{M})$   
\\                                             known/unknown$^\P$ &  \ref{thm:6} & $p > \max\frac{(|c_M|-l-1)!}{(|c_M|-2)!}(1 - \frac{\delta (|c_M|-1)}{8c_MT}), \frac{(|c_M|-l-1)!}{(|c_M|-2)!}(\frac{3}{4})^{\frac{1}{l}}$ &  $q > \frac{e}{e-1}\frac{C^2}{M^2}\max\frac{(C-l-1)!}{(C-2)!}(1 - \frac{\delta (C-1)}{8CT}), \frac{(C-l-1)!}{(C-2)!}(\frac{3}{4})^{\frac{1}{l}} $ & $\frac{e}{e-1}\frac{C^2}{M^2}\frac{(C-l-1)!}{(C-2)!}$  & $O(\frac{C}{M})$   \\ 
\hline
          \hline
      \end{tabular}
      \begin{tablenotes}[para, online,flushleft]
        \footnotesize%\smallskip
        This table assumes reward distributions is sub-Gaussian and the regret bound is of order $\log{T}$.
        \item[]\hspace*{-\fontdimen5\font}$^{*}$ Asm. 1 refers to the assumption on the edge probability $p$ for agents within the same cluster and Asm. 2 refers to the assumption on the edge probability $q$ for agents belonging to different clusters
        %\item[]\hspace*{-\fontdimen2\font}$^{**}$ The value of $C_1$ is $8\sigma^2\max\{12\frac{M(M+2)}{M^4}$ and $\Delta_i$ is the sub-optimality gap of arm $i$. The degree of heterogeneity is $\frac{C}{M}$. 
        %\item[]\hspace*{-\fontdimen2\font}$^\dagger$ Corollary \ref{cor:6} is the extension of Theorem \ref{thm:6} when the cluster is unknown. Likewise we can establish the extensions of Theorem \ref{thm:2}-\ref{thm:5} which are presented in Appendix. 
        \item[]\hspace*{-\fontdimen2\font}$^{\ddagger}$ We observe that going from $1$ to $C$ clusters, the regret grows linearly with $C$, which represents the dependency between the regret bound and the degree of heterogeneity. \item[]\hspace*{-\fontdimen2\font}$^{\S}$ The worst case scenario refers to the case when $T \to \infty$, i.e. the upper bound on the Asm. 2. The value in the existing work \citep{xu2023decentralized} is also $1$.  \item[]\hspace*{-\fontdimen2\font}$^{\P}$ We impose additional assumptions on $p-q$ in the unknown cluster case.  

        % \item[]\hspace*{-\fontdimen2\font}$^{\S}$ Almost all applications satisfy $B_2=\Theta(B_1\sqrt{K})$.

      \end{tablenotes}
    \end{threeparttable}
  }
  \vspace{-0.1in}
\end{table*}

\iffalse
\begin{tikzpicture}
        \draw [<->, line width= 2, draw=blue] (0,0)--(8,0);
        % \draw [line width= 1, draw=blue] (0,-0.1) -- (0,0.1) node[above] {\(0\)};
        \node at (0,-0.3) {\(0\)};
        \node at (8,-0.3) {\(1\)};
        % \draw [line width= 1, draw=blue] (8,-0.1) -- (8,0.1) node[above] {\(1\)};
        % \draw [line width= 1, draw=blue] (4,-0.1) -- (4,0.1) node[above] {\(\frac 1 2\)};
        \draw (-0.4, 0) node[align=center] {\textcolor{blue}{\(\alpha\)}};

        \draw (0, -1) node[align=center]
        {Constant};
        % \draw (0, -1) node[align=center]
        % {by DMED};

        \draw (8, -1) node[align=center]
        {Constant};
        % \draw (8, -1) node[align=center]
        % {by RMED};

        \draw (0, 0.5) node[align=center]
        {\(O(\log T)\)};
        % \draw (0, -1) node[align=center]
        % {by DMED};

        \draw (8, 0.5) node[align=center]
        {\(O(\log T)\)};

        \draw (4, 0.5) node[align=center]
        {\(O( \log T )\)};
        \draw (4, 1) node[align=center]
        {Alg.~2};

        \draw (4, -1) node[align=center]
        {\(O( \log T )\)};
        \draw (4, -0.5) node[align=center]
        {Alg 1};

      \end{tikzpicture}
\fi

\section{Related Work}\label{sec:2.2}

Our proposed model differs significantly from existing work on multi-agent multi-armed bandits. Specifically, we outline these differences in comparison to the existing lines of research on: 1) homogeneous cooperative multi-agent multi-armed bandits, 2) heterogeneous cooperative multi-agent multi-armed bandits, 3) multi-agent multi-armed bandits with clusters of agents, and \textcolor{black}{4) multi-agent multi-armed bandits with time-varying graphs.}

\paragraph{Homogeneous Cooperative Multi-agent Multi-armed Bandit} 
There has been extensive work on cooperative multi-agent bandits, with most studies assuming homogeneous rewards, where the reward distribution for the same arm is identical across all agents \citep{landgren2016distributed,landgren2016distributed_2,landgren2021distributed,zhu2020distributed,martinez2019decentralized,agarwal2022multi,wangx2022achieving,wangp2020optimal,li2022privacy,sankararaman2019social,chawla2020gossiping}. In contrast, our model incorporates heterogeneous reward distributions for agents in different clusters, while maintaining homogeneous reward distributions for agents within the same cluster. Notably, when there is only one cluster, our model reduces to the case of the homogeneous reward. 

\paragraph{Heterogeneous Cooperative Multi-agent Multi-armed Bandit} Although some studies have explored heterogeneous rewards \citep{xu2023decentralized, zhu2023distributed, zhao2020privacy}, they treat rewards as entirely heterogeneous, without considering the possibility of a framework that bridges heterogeneity and homogeneity, along with the associated problem complexity. In our work, we define and systematically characterize the degree of heterogeneity using the cluster structure. Our model also aligns with existing heterogeneous cases when each agent belongs to a different cluster, and thus, there are $M$ clusters. In summary, our work bridges the gap between homogeneous and heterogeneous rewards by integrating both paradigms and fully characterizing every possible degree of heterogeneity.

\paragraph{Clusters by Stochastic Block Models} 
One key to our framework is considering cluster structure motivated by the Stochastic Block Model (SBM), which was previously a separate line of work. SBM, introduced by~\citep{holland1983stochastic}, is known as a foundational framework for modeling community (referred to as clusters herein) structures in networks. It has been extensively studied for cluster detection, with detailed analyses providing exact phase transitions and efficient algorithms for different recovery settings~\citep{abbe2018community}. However, it has not yet been coupled with MA-MAB to model and leverage the agent structure to additionally decide on the reward distribution, and thus bridge the gaps. Besides modeling, we also consider scenarios where the cluster assignment is unknown, inspired by the Contextual Stochastic Block Model (CSBM) proposed by~\citep{deshpande2018contextual}, which generalizes SBM by incorporating side information—namely node covariates—that depend on cluster assignments. Building on that, recent work provide algorithms to leverage both graph structure and contextual attributes to enhance cluster detection and recovery~\citep{deshpande2018contextual,abbe2022,braun2022iterative,dreveton2024exact}. Specifically, \citep{abbe2022} rigorously study the case where node covariates are generated from a Gaussian Mixture Model (GMM) and propose an algorithm for two-cluster networks. More generally, \citep{braun2022iterative} develop an iterative clustering algorithm and derive the exact recovery threshold for multiple balanced clusters. Notably, none of them consider reward information as side information unique to MA-MAB.

\paragraph{Multi-agent Multi-armed Bandit with Clusters of Agents} 
Another related line of research incorporates cluster structures into multi-armed bandits, commonly referred to as the online clustering of bandits (CLUB)~\cite{gentile2014online,nguyen2014dynamic,li2016online,li2016collaborative,korda2016distributed,li2018online,li2019improved,gentile2017context,li2023clustering,ban2024meta,8440090,liu2022federated,wu2021clustering,blaser2024federated,yang2024federated,li2025demystifying, pal2024blocked}. These studies assume that agents can be grouped into clusters, with each group sharing similar reward distributions for each arm, a concept that aligns with our setting. However, there are three significant differences between CLUB and our work.
First, while CLUB primarily focuses on contextual bandit scenarios and provides instance-independent regret bounds, our work addresses the canonical multi-agent MAB setting and establishes finer-grained, instance-dependent regret bounds.
Second, most CLUB approaches assume a central server within a star-shaped communication graph~\cite{liu2022federated,blaser2024federated,yang2024federated}. To our knowledge, only \citet{korda2016distributed} consider peer-to-peer networks, where agents can communicate with any other agent using a gossip protocol. In contrast, our work involves a more realistic and challenging scenario: communication is constrained by a random communication graph modeled by a stochastic block structure. In this setting, only agents connected by an edge can exchange information, significantly increasing the problem's complexity.
Finally, CLUB aims to identify the optimal arm for each individual agent, whereas our work focuses on finding a \textit{globally} optimal arm across all agents . Consequently, our framework requires each agent not only to learn its own reward distribution but also to infer the reward distributions of other agents . This added complexity is particularly demanding under the constraints of a random communication graph.

\paragraph{Multi-agent Multi-armed Bandit with Graphs} 
Recently, the study of multi-agent bandit problems, where agents are distributed on a graph that constrains their communication, has gained significant attention. Most existing works focus on time-invariant graphs, where the graph remains constant over time \citep{wang2021multitask, jiang2023multi, zhu2020distributed,zhu2021decentralized,zhu2021federated}. However, there is an emerging need to address time-varying graphs, which capture more general scenarios where the graph changes over time, motivated by wireless ad-hoc networks in IoT \citep{roman2013features}. It is worth noting that existing work on time-varying graphs either considers Erdos-Renyi graphs with homogeneous edge probabilities \citep{xu2023decentralized} or focuses on connected graphs \citep{zhu2023distributed}, without exploring heterogeneous edge probabilities or the relationship between graph dynamics and reward dynamics. Notably, we are the first to bridge this gap by introducing stochastic block models, which are more general than Erdos-Renyi graphs, and by relating graph topology to reward heterogeneity through a cluster structure. Furthermore, existing work on Erdos-Renyi graphs \citep{xu2023decentralized} imposes strong assumptions on edge probabilities, which may be highly impractical. We address this limitation by leveraging cluster information and significantly relaxing these assumptions.

%\paragraph{Multi-agent Bandit with heterogeneous rewards}

\section{Problem Formulation}\label{sec:formulation}
In this section, we introduce the notations and formally present the problem formulation.
 We start by introducing the notations. Consistent with the basic MAB setting, we consider $K$ arms, labeled as $1,2, \ldots, K$. Let us denote each time step as $1 \leq t \leq  T$, where $T$ is the length of the time horizon. Let us denote $M$ as the number of agents in this multi-agent setting. These agents are distributed on a time-dependent graph $G_t$ represented by vertex set $V = \{1, 2, \ldots, M\}$ and edge set $E_t$. We use $X_{i,j}^t$ to denote whether an agent pair $(i,j) \in E_t$. We use $\mathcal{N}_m(t)$ to denote the neighbor set of agent~$m$ at time $t$, where agent $j$ is called to be in the neighbor set  $\mathcal{N}_m(t)$ if and only if there is an edge between them, i.e., $(m,j) \in E_t$. The graph $G_t$ is independent and identically distributed samples from the stochastic block models that extend the Erdos-Renyi model in~\citep{xu2023decentralized} as described below.

\begin{definition}[Stochastic Block Models]
We consider a stochastic block model, where the set of agents (vertices) with a cluster structure, each agent $1 \leq i \leq M$, belongs to a cluster $c_i \in {1, 2, \ldots, C}$. Additionally, there exists an \textbf{unknown} probability matrix $\{p(m,n)\}_{1 \leq m \leq C}^{1 \leq n \leq C}$ associated with the clusters, where $p(m,n)$ represents the probability of having an edge between an agent pair $(i,j)$, where agent $i \in m$ and agent $j \in n$. Notably, $p(m,n) \neq p(m,m)$ for $m \neq n$, meaning the probability of having an edge between two agents within the same cluster differs from the probability of having an edge between two agents across different clusters, implying heterogeneous random graphs. Then we sample $X_{i,j}^t$ based on $\{p(m,n)\}_{1 \leq m \leq C}^{1 \leq n \leq C}$, for $\forall i,j \in V$, and $E_t = \{(i,j): X_{i,j}^t = 1, \forall i,j \in V\}$ .
\end{definition}

It is worth noting that when $C=1$, $X_{i,j}^t$ are sampled according to a Bernoulli distribution with a uniform edge probability, which precisely aligns with the definition of Erdos-Renyi models.

Besides the graph setting, we consider the reward setting characterized by clusters based on stochastic block models. Let $\mu_k^i$ denote the reward mean value of arm $1 \leq k \leq K$ for agent $1 \leq i \leq M$. Notably, the reward mean values for the same arm are identical for agents within the same cluster, i.e., $\mu_k^i = \mu_k^m \doteq \mu_k^{c_m}$, for agent $m$ and $j$ that meet $c_m = c_i$
while differing for agents in different clusters. This framework effectively bridges the gap between homogeneous and heterogeneous MA-MAB settings. Moreover, we propose a new definition of the degree of heterogeneity as follows. 
\begin{definition}[Degree of Heterogeneity]
We define $h_{M,C} = \nicefrac{C}{M} = \nicefrac{1}{c_M}$ where $c_M$ represents the average number of agents in one cluster, which is scale-invariant and bounded by $1$, i.e., $0 < h_{M,C} \leq 1$. 
\end{definition}
We argue the rationality as follows. This metric quantifies the variety in the reward/edge distributions across clusters relative to the total number of agents. When $C=M$, it is fully heterogeneous, aligning with \citep{xu2023decentralized}, and when $C=1$, it is fully homogeneous, consistent with \citep{wangx2022achieving}. 

The reward of arm $k$ at agent $i$ at time step $t$, denoted as $\{r_k^i(t)\}_{k,i,t}$, follows a $\sigma^2$-sub-Gaussian distribution with a time-invariant mean value $\{\mu_k^i\}_{k,i}$.
\begin{remark}
While we assume sub-Gaussian reward distributions for illustrative purposes, we highlight that the formulation and results established herein can be extended to sub-exponential cases through straightforward analysis, as this does not require changes to the communication or information update mechanisms. We omit the details and focus on the sub-Gaussian case in this work. 
\end{remark}

Let $a_i^t$ represent the arm selected by agent $i$ at time $t$, and let $n_{i,k}(t)$ denote the number of pulls of arm $k$ at agent $i$ up to time $t$. We consider a cooperative setting where the goal of all agents is to select the globally optimal arm, defined as $k^* = \arg\max_k \sum_{i=1}^{M} \mu_k^i$. The objective of the system is to maximize the pulls of the globally optimal arm, thereby minimizing regret, which is defined with respect to the globally optimal arm as follows. Formally, the regret and total regret are given by 
\begin{align*}
   & R_T = \frac{1}{M}\sum_{k=1}^K\sum_{i = 1 }^{M}\Delta_k n_{i,k}(T), \quad R_T^{M} = M \cdot R_T = \sum_{k=1}^K\sum_{i = 1 }^{M}\Delta_k n_{i,k}(T).
\end{align*}
respectively, where $\Delta_k = \nicefrac{(\sum_{i=1}^{M}\mu_{k^*}^{i} - \sum_{i=1}^{M}\mu_{k}^{i})}{M}$.

%We next introduce some definitions unique to our proposed formulation, from both the graph and the reward perspective. 

\iffalse
\begin{definition}[graph generation]
  For any two agents $i$ and $j$, they know whether $\mu_{k}^{i}$ equals $\mu_{k}^{j}$. The environment generates the independent and identical samples by generating rewards based on $\{\mu_k^i\}_{k}^{i}$.
\end{definition}

\begin{definition}[reward generation]
  For any two agents $i$ and $j$, they know whether $\mu_{k}^{i}$ equals $\mu_{k}^{j}$. The environment generates the independent and identical samples by generating rewards based on $\{\mu_k^i\}_{k}^{i}$ and sub-Gaussian distributions. 
\end{definition}
\fi

\section{Real-world Applications}\label{sec:app}

%\hl{Citations}

%\mo{it would be great if you could add another application in computing domains, something like CDN server/content placement, online advertisement, or llm model selection. @Xutong, can you help with this part?}

In this section, we motivate our problem formulation, which bridges existing gaps by considering agents on stochastic block models, through a range of real-world applications. Here stochastic block models capture agents with specific probabilities of being connected, where agents can observe edges but not the underlying edge probabilities—a scenario that often reflects real-world conditions.  

\paragraph{Collaborative Content Placement in Content Delivery Networks (CDNs)}
Online content delivery in content delivery networks (CDNs) is a critical component of modern network applications, including video streaming, web browsing, and software distribution~\cite{yang2018content,chen2018spatio,dai2024axiomvision}. Unlike traditional architectures that rely on a single central server, CDNs distribute and cache content across multiple edge servers, allowing end users to retrieve data from the nearest server. This distributed architecture significantly reduces latency and enhances the reliability of content delivery.
A key challenge in CDNs lies in dynamically placing content across thousands of edge servers to optimize user service-level objectives (e.g., latency, packet loss) and quality of experience (QoE). In this context, each edge server can be modeled as an agent, with its arms representing candidate content placement policies. The reward for each arm corresponds to the number of successfully delivered and precached contents, which ultimately reduce users' loading time.
Given the heterogeneity in user preferences and network conditions, edge servers may form clusters where only agents within the same cluster share similar rewards, modeled by a stochastic block model. Furthermore, the large number of edge servers and candidate policies necessitates collaboration among servers to learn optimal policies. However, due to communication bandwidth constraints, servers can only communicate randomly, governed by a random graph.
Our framework effectively models this problem, enabling the identification of the \textit{global} optimal content placement policy that maximizes the reward across all edge servers.

\paragraph{Collaborations in Social Networks} Examples include scientific collaboration networks of biologists and physicists, where an edge represents a collaboration, defined as co-authorship of one or more scientific articles during the study period, and a cluster refers to working in the same main research area \citep{cugmas2020scientific}. \textcolor{black}{The collaboration network is highly dynamic, as collaborations change over time and are modeled by time-varying graphs.} These scientists may select the most important research topic from a few options, referred to as arms, with the reward of an arm being the impact of the research topic (scientists working in the same area will have the same impact). In the context of collaboration networks of movie actors \citep{el2024community}, an edge represents appearing in the same movie, \textcolor{black}{which again varies in movies released at different times and is therefore time-varying}, and a cluster refers to club membership. These actors may choose the best club activity among several options, referred to as arms, with the reward of an arm being the engagement in the activity (actors in the same club will have the same reward distribution). Similarly, in a network of directors of Fortune 1000 companies \citep{battiston2004statistical}, an edge between two directors indicates that they served on the same board \textcolor{black}{that can change over time as the board committee itself may change}, and a cluster again refers to club membership. Here, the arms are the start-up candidates for investment, and the reward of an arm is the return on investment (ROI) (directors in the same club will have the same reward distribution). %\mo{if possible, some justifications about time-varying nature of graph can help people buy the model.}

\paragraph{Protein-to-Protein in AI-enabled Biology} 
\textcolor{black}{In the context of protein-to-protein biology research, AI-enabled proteins embedded in a patient act as agents/nodes, and the physical connections or interactions between proteins are represented as edges within a cell, namely a protein-to-protein interactions network~\citep{airoldi2006mixed}. It is worth noting that these interactions change over time, resulting in time-varying graphs. AI-enabled Proteins functioning similarly in the protein-protein interaction network belong to the same cluster. The task of the proteins is to transport different nutrients (arms) in the body, and the rewards of the arms are the corresponding health conditions of a patient, e.g., blood pressure or blood sugar levels, resulting from different nutrients (the reward distribution for proteins within the same cluster is identical).}  %\mo{this is short and hard to follow, another pass to make it more connected to model is needed.}

\paragraph{Recommendation in E-commerce with Filtering} In e-commerce, filtering has been an effective approach where customers utilize others' information to make decisions \citep{stanley2019stochastic, duchemin2023reliable}. In collaborative filtering, customers are represented as nodes/agents, an edge between two customers indicates similar behaviors, and the most similar customers (those with the highest degrees) form clusters. The arms represent product candidates, and the reward of an arm corresponds to the experience with the product (hence, the reward distribution for the same arm is identical for agents within the same cluster). In item-to-item collaborative filtering, product producers are represented as nodes/agents, an edge signifies similar properties, and the most similar products (again, those with the highest degrees) form clusters. The arms are the warehouse options for the products, and the reward corresponds to the quality of the product after being stored in the warehouse (the reward distributions for the same warehouse are identical for products within the same cluster).

\section{Warm-up: Single Cluster (Homogeneous Clients)}\label{sec:homo}

This section studies the single-cluster multi-agent MAB, focusing on in-cluster learning and serving as a didactic warm-up for the multi-cluster scenarios discussed in later sections. Here, all clients belong to the same cluster and share a homogeneous reward environment. Although there is existing work on homogeneous multi-agent multi-armed bandits \citep{wang2023achieve,wang2020optimal}, our model introduces a time-varying random graph $G_t$, which has not yet been studied.
In this homogeneous setting, all clients have the same reward distribution for each arm. Consequently, observations from different clients can be combined to improve the estimation of an arm's reward distribution, leading to a more efficient exploration-exploitation trade-off compared to the single-agent case.
We first present a simple UCB algorithm in Section~\ref{subsec:homo-ucb}, followed by its regret upper bound in Section~\ref{subsec:homo-ucb-bound}.

\subsection{Algorithm}\label{subsec:homo-ucb}

Since clients are homogeneous, we propose a simple cooperative Upper Confidence Bounds (UCB) algorithm. Over the whole learning process, every client \(m\) maintains the UCB index for each arm \(k\) as follows, $u_{k,t}\upbra{m} \coloneqq \hat{\mu}_{k,t}\upbra{m} + \sqrt{\nicefrac{\log t}{\tilde N_{k,t}\upbra{m}}},$
where \(
\tilde N_{k,t}\upbra{m} =
N_{k,t}\upbra{m}
+ \sum_{m'\in\mathcal{M}\setminus \{m\}} N_{k,\tau_t\upbra{m\leftrightarrow m'}}\upbra{m'}
\)
% \todo{consider the impact of message propagation from multi-hop. Maybe directly give a \(N_{k,t}\upbra{-m}\) to roughly represent all received the messages due to propagation.}
is the total number of observations of arm \(k\) that client \(m\) collects, including its own \(N_{k,t}\upbra{m}\) local observations and the \(N_{k,\tau_t\upbra{m\leftrightarrow m'}}\upbra{m'}\) observations collected from its neighbors \(m'\in\mathcal{M}\setminus \{m\}\) at the latest communication time slot \(\tau_t\upbra{m\leftrightarrow m'}\) between these two clients \(m\) and \(m'\) on or before time slot \(t\).
The empirical mean \(\hat\mu_{k,t}\upbra{m}\) is also the average of all \(\tilde N_{k,t}\upbra{m}\) observations of arm \(k\) that client \(m\) collects. The clients pull the arm with the highest UCB index, i.e. $a_m^t = argmax_{k}u_{k,t}^m$, and receive the reward.

\subsection{Regret Analysis}\label{subsec:homo-ucb-bound}

%Theorem~\ref{thm:homo-ucb} presents the regret upper bound of the method.

\begin{theorem}\label{thm:homo-ucb}
    Executing the above algorithm leads to $ \mathbb{E}[R_T] \le O\left(\sum_{k\neq k^*} \frac{\log T}{M\Delta_k}
        + \frac{K}{p^{M^2}} 
        \right).~\refstepcounter{equation}(\theequation)\label{eq:homo-ucb}$ 
\end{theorem}
\vspace{-5mm}
\begin{proof}[Proof Sketch]
    The full proof is in Appendix \ref{app:proof}; the main intuition is as follows. Fix a suboptimal arm \(k\). After the total number of observations for this arm \(k\) exceeds the sample complexity threshold, in expectation, it takes \(\frac{1}{p^{M^2}}\) time slots for all agents to get the information of this arm \(k\).
    After that, no more regret will be incurred on this arm \(k\). As a result, the proof first makes an assumption to reduce the problem to a standard cooperative UCB for homogeneous agents residing on a complete graph with communication delays.
    Then, we show that this assumption can be fulfilled in the single cluster scenario.
\end{proof}

The leading \(O(\log T)\) term of the total regret $R_T^M = M \cdot R_T$ by multiplying $M$ and ~\eqref{eq:homo-ucb} is independent of the number of agents \(M\), highlighting the advantage of multi-agent cooperation. Meanwhile, it is worth noting that in heterogeneous setting in \citep{xu2023decentralized}, the upper bound of the total regret $R_T^M$ is of order $O(M^2\log{T})$ (though it is not tight as illustrated in Section \ref{sec:heter}), rather than $O(\log{T})$ which emphasizes the regret reduction by our analysis in homogeneous settings. The second term of~\eqref{eq:homo-ucb}, however, has a dependence of \(p^{-M}\) where $p = p(m,m)$. It suggests the importance of the edge generation probability \(p\), which will be thoroughly addressed in Section \ref{sec:heter} and Section \ref{sec:heter-unknown}.

\section{Heterogeneous - Multiple Known Clusters}\label{sec:heter}
In this section, we consider the general model where the agents are distributed on stochastic block models with multiple clusters, capturing the dependency between reward and graph dynamics. Here, we assume that the cluster information $\{c_i\}_{i=1}^{M}$ is known to the agents, and we generalize the results to a more practical setting where the cluster information is unknown in Section \ref{sec:heter-unknown}. We would like to highlight that the probability matrix of the model is always unknown to the agents. The algorithm is presented in Section \ref{sec:heter-sub-1}, followed by the corresponding regret analyses in Section \ref{sec:heter-sub-2}.
 %Let us also assume that the rewards are $\sigma^2$-sub-Gaussian distributed (the case of sub-exponential can be directly derived following a similar logic as in \citep{xu2023decentralized}).

\subsection{Algorithm}\label{sec:heter-sub-1}

The newly proposed algorithm is presented as follows and consists of two stages. In the first stage, all agents pull arms one by one without communication to accumulate local information, referred to as the burn-in period. In the second stage, agents use intelligent strategies (based on Upper Confidence Bounds) to pull arms and communicate with one another following the graph structure to collect global information, referred to as the learning period. The corresponding algorithms are provided as Algorithm \ref{alg:burn-in} (see Appendix \ref{apx:IR-LSS}) and Algorithm \ref{alg:dr}, collectively referred to as UCB-SBM (\textbf{U}pper \textbf{C}onfidence \textbf{B}ounds for \textbf{S}tochastic \textbf{B}lock \textbf{M}odels).

We note that there is no difference between the algorithm in the burn-in period herein and that in \citep{xu2023decentralized}, and thus we show the pseudo code in Appendix \ref{apx:IR-LSS}, except that we do not need $\tau_1$. The reason is that there is no intelligence during this stage. Specifically at $t$, each agent $m$ pull each arm $a_m^t = t \mod K$ one by one and update the average reward as local reward estimators $\bar{\mu}_i^m(t)$. It also updates the edge frequency $P_t(m,j) = \nicefrac{((t-1)P{t-1}(m,j)+X_{m,j}^t)}{t}$ for each $j \in V$, and communicates $\bar{\mu}_i^m(t)$ to agent $j \in \mathcal{N}_m(t)$. Then at the end of the burn-in period, it outputs the initial global estimator $\tilde{\mu}_i^m(L+1)$, which is the weighted average of $\bar{\mu}_i^m(L)$ where weights are $P_t(m,j)$.%The purpose of this stage is to collect sufficient information about local rewards, thereby providing a relatively accurate local estimator for information integration in the subsequent learning period.

\begin{algorithm2e}[ht]
  \SetAlgoLined
  \caption{UCB-SBM: Learning period}\label{alg:dr}
  \textbf{Initialization:} For each agent $m$ and arm $i \in \{1,2,\ldots, K\}$, we have $\Tilde{\mu}^m_i(L+1)$, $\Tilde{N}_{m,i}(L+1), N_{m,i}(L+1) = n_{m,i}(L)$; $\tau=1$, all other values at $L+1$ are initialized as $0$;\par
  \For{$t = L + 1 , L + 2, \ldots,T$}{
  %\For{each agent m}{
  %\For{$i =1,2,\ldots,K$}{
  %Set $p_{m,i}(t) = 1$ if $i = argmax \hat{\mu}_{m,i}(t) + \sqrt{\frac{\ln{t}}{n_{m,i}(t)}}$ and 0 otherwise\;
  %Choose $a_{m}^t$ according to the distribution $p_{m,1}(t),\ldots, p_{m,K}(t)$\;
  %Receive reward $r_{m,a_m^t}(t)$\;}}
  \For(\tcp*[f]{UCB}){each agent m}{
    \eIf{there is no arm $i$ such that $N_{m,i}(t) \leq \Tilde{N}_{m,i}(t) - K$}{ $a_m^t = \arg\max_i \Tilde{\mu}_{m,i}(t) + F(m,i,t)$}
    { Randomly sample an arm $a_m^t$.
    }
    Pull arm $a_m^t$ and receive reward $r_{m,a_m^t}(t)$\;
  }
  The environment samples $G_t = (V,E_t)$ based on SBM;\tcp*[f]{Env} \par
  Each agent $m$ sends $r_i^m(t),N_{j,i}(t),\bar{\mu}_i^m(t), \Tilde{\mu}_i^m(t)$ to each agent in $\mathcal{N}_m(t)$;\label{line:send}\par
  Each agent $m$ receives $r_i^j(t),N_{j,i}(t),\bar{\mu}_i^j(t), \Tilde{\mu}_i^j(t)$ from all agents $j \in \mathcal{N}_m(t)$ and stores them as $\hat{\mu}_{i,j}^{m}(t),\hat{N}_{i,j}^{m}(t),\hat{\bar{\mu}}_{i,j}^{m}(t), \hat{\Tilde{\mu}}_{i,j}^{m}(t)$; \tcp*[f]{Transmission}\par
  %given a limited time period\;
  %\For{$i=1,\ldots,K$}{
  %$\tilde{r}_i(t)$ is estimated by $\frac{\sum_{j\in \mathcal{N}_M} r_{j,i}(t) \cdot \mathds{1}_{a_j^t = i}}{\sum_{j\in \mathcal{N}_M} 1 \cdot \mathds{1}_{a_j^t = i}}$\;
  %}
  \For{each agent m}{
  \For{ $i =1, \ldots, K$}{
  %$\hat{\mu}_{m,i}(t) = \frac{\hat{\mu}_{m,i}(t) \cdot n_{m,a_m^t}(t) + \tilde{r}_{i}(t) \cdot \mathds{1}_{\{\exists j \in {\mathcal{N}_m}, a_j^t = i\}}}{n_{m,i}(scre)}$\;
  %$n_{m,i}(t+1) = n_{m,i}(t) + \mathds{1}_{\{\exists j \in {\mathcal{N}_m}, a_j^t = i\}}$
  Update $P_t$ for $1 \leq j \leq M$ by $P_t(m,j) = \nicefrac{(t-1)P_{t-1}(m,j)+X_{m,j}^t}{t}$\;\par
  Update $P_t^{\prime}$ for $1 \leq j \leq M$ by $
    P_t^{\prime}(m,j) = 1 \text{if } P_t(m,j) > 0 \text{ and } 0 \text{ o.w.}$\;\par
  \eIf{$t \mod \tau = 0$ }{ Update $n_{m,i}(t), N_{m,i}(t), \Tilde{N}_{m,i}(t)$ and $\Tilde{\mu}^m_i(t)$ based on \textbf{Rule 1} or \textbf{Rule 2}}
  {  Update $n_{m,i}(t)$ as $n_{m,i}(t) = n_{m,i}(t) + 1_{a_m^t = i}$
  }
  }}
  }
\end{algorithm2e}
Subsequently, we proceed to the learning stage using either Rule 1 or Rule 2, which define how agents update and aggregate information. Rule 1 is consistent with \citep{xu2023decentralized}, as it does not consider the cluster structure, whereas Rule 2 is newly proposed and leverages the cluster information. The pseudo-code is presented in Algorithm \ref{alg:dr}, which includes several stages in the order outlined below.

\paragraph{Arm selection} 
During this stage, the agents use a UCB-based criterion to decide which arm to pull. More specifically, if there is no arm \(i\) such that \(N_{m,i}(t) \leq \Tilde{N}_{m,i}(t) - K\), where \(N\) and \(\Tilde{N}\) represent the in-cluster and across-cluster number of pulls, respectively, then agent \(m\) pulls \(a_m^t = \arg\max_i \Tilde{\mu}_{m,i}(t) + F(m,i,t)\), where \(\Tilde{\mu}_{m,i}(t)\) is the network-wide estimator for arm \(i\) of agent \(m\) and \(F(m, i, t) = \sqrt{\frac{C_1\ln{t}}{N_{m,i}(t)}}\) (\(C_1\) is specified later) quantifies the uncertainty in \(\Tilde{\mu}_{m,i}(t)\). Otherwise, the agents randomly pull an arm by specifying \(a_m^t = t \mod K\).

\paragraph{Transmission} The agents communicate with their neighbors and integrate information from other agents . Specifically, each agent sends its own information and receives information from agents in its time-dependent neighborhood. The information includes sample counts and reward estimators, covering local, cluster, and global levels, denoted as \(r_i^j(t), N_{j,i}(t), \tilde{N}_{j,i}(t), \bar{\mu}_i^j(t), \Tilde{\mu}_i^j(t)\) defined below.

\paragraph{Information update} With such information, agent $m$ updates estimators as in \textbf{Rule 1} or \textbf{Rule 2}.
{\allowdisplaybreaks
\begin{align}\label{eq:Eq_1}
   & \textbf{Rule 1: } t_{m,j} = max_{s \geq \tau_1} \{(m,j) \in E_s)\} \text{ and } 0  \text{ if such an $s$ does not exist}                                                                                                          \\
   & \indent   N_{m,i}(t+1) = n_{m,i}(t+1) = n_{m,i}(t) + \mathds{1}_{a_m^t = i}, \tilde{N}_{m,i}(t+1) = \max \{N_{m,i}(t+1), \tilde{N}_{j,i}(t), j \in \mathcal{N}_m(t)\} \notag                                                             \\
   & \bar{\mu}^m_i(t+1) = \nicefrac{(\bar{\mu}^m_i(t) \cdot n_{m,i}(t) + r_{m,i}(t) \cdot \mathds{1}_{a_m^t = i})}{n_{m,i}(t+1)}, P^{\prime}_t(m,j) = \nicefrac{(M-1)}{M^2} \text{ if } P_t(m,j) > 0 \text{ and } 0 \text{ otherwise}   \notag                                                                                                                        \\
   & \Tilde{\mu}^m_i(t+1) = \textstyle \sum_{j=1}^M P^{\prime}_t(m,j)\hat{\Tilde{\mu}}^m_{i,j}(t_{m,j}) + d_{m,t}\textstyle \sum_{j }\hat{\bar{\mu}}^m_{i,j}(t_{m,j})  \text{ with } d_{m,t} = \nicefrac{(1- \sum_{j=1}^M P^{\prime}_t(m,j))}{M} \notag
\end{align}
}
\vspace{-3mm}
{\allowdisplaybreaks
\begin{align}\label{eq:Eq_2}
   & \textbf{Rule 2: } t_{m,j} = \max_{s \geq \tau_1} \{(m,j) \in E_s)\} \text{ and } 0  \text{ if such an $s$ does not exist} \notag                                                                                                                  \\
   & \text{Local and Cluster sample counts: } n_{m,i}(t+1) = n_{m,i}(t) + \mathds{1}_{a_m^t = i}, N_{m,i}(t) = N_{c_m,i}(t) = \textstyle \sum_{j}n_{j,i}(t_{m,j}) \notag                                                                                                                                             \\
   & \text{Global sample counts: } \Tilde{N}_{m,i}(t+1) = \max \{N^m_{i}(t), \Tilde{N}^j_{i}(t), (m,j) \in E_t\} \notag                                                                                                                             \\
   & \text{Local estimator: } \bar{\mu}^m_i(t+1) = \nicefrac{(\bar{\mu}^m_i(t) \cdot n_{m,i}(t) + r_{m,i}(t) \cdot \mathds{1}_{a_m^t = i})}{n_{m,i}(t+1)} \notag                                                                                          \\
   & \text{Cluster estimator (local): } \hat{\mu}^{c_m}_i(t+1) = \hat{\mu}^m_i(t+1) = \nicefrac{(\sum_{j\in c_m}\bar{\mu}^j_i(t_{m,j}))}{|c_m|} \notag                                                                                                        \\
   & \text{Cluster estimator (network-wise): } \Tilde{\bar{\mu}}^{c_m}_i(t+1) = \Tilde{\bar{\mu}}^m_i(t+1) = \nicefrac{(\sum_{j\in c_m}\Tilde{\mu}^j_i(t_{m,j}))}{|c_m|} \notag                                                                               \\
   & P_t(c_m,c_j) = \nicefrac{(\sum_{s\leq t, p \in c_m, q \in c_j} 1_{(p,q) \in E_s})}{t},  P^{\prime}_t(m,j) = \nicefrac{(M-1)}{M^2} \text{ if } P_t(c_m,c_j) > 0 \text{ and } 0 \text{ otherwise}                                                                                                                                              \\
   & \text{Global estimator: } \Tilde{\mu}^m_i(t+1) = \notag                                                                                                                                                                                        \\
   & \sum_{j=1}^M P^{\prime}_t(m,j)\tilde{\mu}^j_{i}(t_{m,j}) + d_{m,t}\sum_{j \in c_m}\hat{\mu}^j_{i}(t_{m,j}) + d_{m,t}\sum_{j \not \in c_m}\hat{\mu}^j_{i}(t_{m,j})   \text{ with } d_{m,t} = \nicefrac{(1- \sum_{j=1}^M P^{\prime}_t(m,j))}{M} \notag
\end{align}
}
The difference between Rule 1 and Rule 2 is that Rule 1 is the same as in \citep{xu2023decentralized} and leverage no cluster information, whereas Rule 2 is newly proposed herein. Rule 2 considers the stochastic block model structure (agents within a cluster aggregate $n$ and $\bar{\mu}$ to obtain $N$ and $\hat{\mu}$), communicates at the cluster level ($\hat{\mu}$ instead of $\bar{\mu}$), and utilizes the cluster information to improve the estimators $\tilde{\mu}$ (with 3 sources: local, cluster, and global information). As shown in Section \ref{sec:heter-sub-2}, they result in different regret bounds, with Rule 2 achieving a smaller regret and requiring less stringent assumptions.

\subsection{Regret Analyses}\label{sec:heter-sub-2}

Next, we prove the effectiveness of the proposed algorithm through analyzing the theoretical regret induced by the algorithm. For illustration purposes, let us assume a balanced model where the number of agents in each cluster is the same for all clusters, i.e., $|c_i| = \frac{M}{C} \doteq c_M$. We highlight that the case of imbalanced clusters (with respect to cluster size) can be addressed in our analyses by using the smallest number of agents in a single cluster, $\min_{1 \leq i \leq M} |c_i|$, as the universal cluster size.

As a starting point, we consider the regret of Algorithm 2 with Rule 1, which aligns with existing work on heterogeneous rewards without characterizing the cluster structure. A straightforward result is presented below, which is a by-product of Theorem 2 in \citep{xu2023decentralized}, but with potentially different edge probabilities for different agent pairs, and it reads as follows.

\begin{theorem}\label{thm:2}
  Let us assume that $\min_{m,n}p(m,n) \geq (\frac{1}{2} + \frac{1}{2}\sqrt{1 - (\frac{\delta}{MT})^{\frac{2}{M-1}}})$. For every $0 < \epsilon < 1$ and $0 < \delta
    < \frac{1}{2} + \frac{1}{4}\sqrt{1 - (\frac{\epsilon}{MT})^{\frac{2}{M-1}}}$, the regret of Algorithm \ref{alg:dr} with Rule 1 is upper bounded by with probability $1-7\epsilon$, $E[R_T | A_{\epsilon, \delta}] \leq L + \sum_{i \neq i^*}\Delta_i(\max{\{[\frac{4C_1\log T}{\Delta_i^2}], 2(K^2+MK) \}} +  \frac{2\pi^2}{3P(A_{\epsilon, \delta})} + K^2 + (2M-1)K)$
  where $A_{\epsilon, \delta} = A_2 \cap A_3$ with $A_2 = \{\exists t_0, \forall t \geq L, \forall j, \forall m, t+1 - \min_jt_{m,j} \leq t_0 \leq c_0\min_{l}n_{l,i}(t+1)\}$ and $A_3 = \{\forall t \geq L, G_t \text{ is connected}\}$, the length of the burn-in period is explicitly $L  = \max\{\nicefrac{{}\ln{\nicefrac{T}{2\epsilon}}}{2\delta^2}, \nicefrac{4K\log_{2}T}{c_0}\},$
  $c_0$ $=$ $c_0(K, \min_{i \neq i^*}\Delta_i, M, \epsilon, \delta)$, and the instance-dependent constant
  $ 
    C_1 = 8\sigma^2C = \max\{\nicefrac{4(M+2)(1 - \frac{1 - c_0}{2(M+2)})^2}{3M(1-c_0)}, (M+2)(1 + 4Md^2_{m,t})\}$.
\end{theorem}

\begin{proof}[Proof sketch]
The complete proof is provided in Appendix \ref{app:proof}; the main proof logic is as follows. The proof of Theorem 1 parallels that of Theorem 2 in \citep{xu2023decentralized} for the Erdos-Renyi graph, except that the edge probability is agent-dependent in this case. Interestingly, we find that as long as the minimal edge probability satisfies the condition on the edge probability in the Erdos-Renyi model as specified in \citep{xu2023decentralized}, the entire proof remains valid. The key observation is that the original analysis relies solely on the lower bound of the edge probability in the Erdos-Renyi model. 
\end{proof}

While the above regret bound depends logarithmically on the time horizon $T$, there are two limitations: 1) the assumption on $\min_{m,n}p_{m,n}$ is stringent and may not always hold in practice, and 2) the total regret ($M \cdot R_T$) depends linearly on $M$, which may not scale well in large-scale systems. This is because the analysis does not leverage the homogeneity within clusters, leading to high sample complexity for the reward estimators. To address these limitations, we next present an approach that exploits the cluster structure using Rule 2 in Algorithm 2, which takes advantage of the homogeneity within clusters induced by stochastic block models.

Intuitively, agents using Rule 2 first aggregate the rewards and sample counts of agents within the same cluster, and then communicate these at a cluster level, meaning they only share cluster-wide information. The aggregation reduces sample complexity because the variance of the averaged estimator is smaller than that of a single agent's estimator. Consequently, this can potentially lead to smaller regret in terms of $M$. Additionally, cluster-level communication reduces the need for pairwise (every agent pair) communication. As long as there exists an agent in one cluster and another agent in a different cluster with an edge between them, the two clusters can communicate, rather than requiring every agent in one cluster to be connected to every agent in the other cluster. This approach reduces the connectivity requirements of the graph and relaxes the assumption on $\min_{m,n}p_{m,n}$, as a larger $\min_{m,n}p_{m,n}$ always implies better connectivity (in the high probability sense).

Formally, we consider communication at a cluster level by defining the subgraph generated by the clusters as $G_t^C$ as follows.

\begin{definition}
A sub-graph $G_t^C$ is represented by the vertex set $\{1, 2, \ldots, C\}$ of clusters and the edge set $E_t^{\prime}$, where the pair of clusters $x$ and $y$, namely $(x,y)$, belongs to $E_t^{\prime}$ if and only if there exists an agent $i \in x$ and an agent $j \in y$ such that $(i,j) \in E_t$.
\end{definition}

It holds true that the sub-graph $G_t^C$ is much denser compared to the original graph $G_t$ as it has a higher probability of having an edge (cluster pair) and thus a lower requirement on the graph topology of $G_t$. First, we consider the case where the graph induced by the clusters, $G_t^C$, is a connected graph and the edge probability within one cluster is $1$, and derive a better regret bound with relaxed assumptions. For illustration purposes, it is natural to assume that the edge probability within the same cluster is $1$, and we relax this assumption later (Theorem \ref{thm:6}) in this section. The formal statement is summarized as Theorem \ref{thm:3}, which reads as follows.

\begin{theorem}\label{thm:3}
  Let us assume that $p(m,m) = 1$ for any $1 \leq m \leq C$. Let us further assume that $\min_{m,n}p(m,n) \geq 1 - (\frac{1}{2} - \frac{1}{2}\sqrt{1 - (\frac{\delta}{CT})^{\frac{2}{C-1}}})^{\nicefrac{C^2}{M^2}}$. The regret bound of Algorithm 2 with Rule 2 reads as with probability $1-7\epsilon$ $E[R_T|A_{\epsilon,\delta}^{\prime}]                        \leq L + \sum_{i \neq i^*}\Delta_i(\max{\{\frac{C}{M} \cdot [\frac{4C_1\log T}{\Delta_i^2}], 2(K^2+MK) \}} +  \nicefrac{2\pi^2}{3P(A_{\epsilon, \delta})} + K^2 + (2M-1)K)  = O(\nicefrac{C\log{T}}{M})$
where $A_{\epsilon,\delta}^{\prime} = A_2 \cap A_3^{\prime}$, $A_2 = \{\exists t_0, \forall t \geq L, \forall j, \forall m, t+1 - \min_jt_{m,j} \leq t_0 \leq c_0\min_{l}n_{l,i}(t+1)\}$ and $A_3^{\prime} = \{\forall t \geq L, G_t^C \text{ is connected}\}$, the length of the burn-in period is explicitly $L  = \nicefrac{C}{M}\max{\{\nicefrac{{}\ln{\frac{T}{2\epsilon}}}{2\delta^2}, \nicefrac{4K\log_{2}T}{c_0}}\},$
$c_0$ $=$ $c_0(K, \min_{i \neq i^*}\Delta_i, M, \epsilon, \delta)$, and $
  C_1 = \max\{\nicefrac{4(M+2)(1 - \frac{1 - c_0}{2(M+2)})^2}{3M(1-c_0)}, (M+2)(1 + 4Md^2_{m,t})/M\}
$.
\end{theorem}
\begin{proof}[Proof Sketch]
The full proof is deferred to Appendix \ref{app:proof}; we present the main logic herein. We note that by Lemma~\ref{lem:edge_prob_1}, the edge probability $p(c_m, c_n)$ of the sub-graph $G_t^C$ is $1-(1-p(m,n))^{M^2/C^2}$. In other words, as long as $p(c_m, c_n)$ meets the condition of Theorem \ref{thm:2}, i.e., $\min_{m,n}p(c_m,c_n) \geq (\frac{1}{2} + \frac{1}{2}\sqrt{1 - (\frac{\delta}{CT})^{\frac{2}{C-1}}})$, we achieve the same regret bound, where everything is with respect to the sub-graph instead of the original graph. Subsequently, we derive that the condition is equivalent to $\min_{m,n}p(m,n) \geq 1 - (\frac{1}{2} - \frac{1}{2}\sqrt{1 - (\frac{\delta}{CT})^{\frac{2}{C-1}}})^{C^2/M^2}$. Hence, the regret bound in Theorem \ref{thm:2} holds. 
\end{proof}

\begin{lemma}\label{lem:edge_prob_1}
    For any pair of vertices $c_m, c_n$ in the sub-graph $G^C_t$, the probability that $c_m$ and $c_n$ is connected in $G^C_t$ is $p(c_m,c_n) = 1-(1-p(m,n))^{M^2/C^2}$. 
\end{lemma}
%The proof of the two Lemmas are in Appendix. 

% \begin{lemma}
%     For any pair of vertices $c_m, c_n$ in the sub-graph $G^C_t$, the probability that $c_m$ and $c_n$ is connected in $G^C_t$ is $p(c_m,c_n) \geq z$ if $p(m,n) \geq \frac{C^2}{M^2}\ln(\frac{1}{1-z})$. 
% \end{lemma}

% \begin{proof}
%     Since all clusters have the same size, the clusters $c_m$ and $c_n$ have $\frac{M}{C}$ agents each. Thus, there are $\frac{M^2}{C^2}$ pairs of vertices between clusters $c_m$ and $c_n$. Since each pair of such vertices is connected with probability $p(m,n)$ independently, the probability that there is at least one edge between clusters $c_m$ and $c_n$ is $1-(1-p(m,n))^{M^2/C^2}$.
%     If $p(m,n) \geq \frac{C^2}{M^2}\ln(\frac{1}{1-z})$, then we have
%     $$
%     (1-p(m,n))^{M^2/C^2} \leq (1 - \frac{C^2}{M^2}\ln(\frac{1}{1-z}))^{M^2/C^2} \leq e^{-\ln(\frac{1}{1-z})} = 1-z.
%     $$
%     Therefore, we have 
%     $$
%     p(c_m,c_n) = 1 - (1-p(m,n))^{M^2/C^2} \geq z.
%     $$
% \end{proof}

\paragraph{Discussion on the total regret} We novelly derive a regret bound on $R_T$ that depends on the degree of heterogeneity $h_{M,C} = \frac{C}{M}$, reflecting the problem complexity related to the stochastic block model and reward heterogeneity we consider, and highlighting the comprehensiveness of our regret bound. It is worth noting that our result also resolves an open problem identified in \citep{xu2023decentralized}, where the authors numerically claimed a dependency of the regret on heterogeneity without formally defining or analyzing it. The advantage of having this explicit dependency compared to the aforementioned established results is as follows. When $C=1$, it is consistent with Theorem \ref{thm:homo-ucb} in Section \ref{sec:homo}, further supporting our claim. We note that both the total regret bound ($M^2 \cdot R_T$) in \citep{xu2023decentralized} and our result in Theorem \ref{thm:2} depend on $M^2$, while the one in Theorem \ref{thm:3} depends linearly on $C \leq M$. This demonstrates an improvement over the existing result in \citep{xu2023decentralized} when $C=M$ (no homogeneity) and further implies a significant reduction in the regret bound when $C < M$. This improvement is particularly significant in large-scale systems where $C \ll M$ (a common scenario in real-world applications, e.g., people in different regions where all individuals are agents and their clusters are defined by the regions). It highlights the practical importance of our proposed algorithm and provides practitioners with more effective tools.

\paragraph{Discussion on the assumption} Notably, the lower bound on $\min_{m,n}p_{m,n}$ is $1 - (\frac{1}{2} - \frac{1}{2}\sqrt{1 - (\frac{\delta}{CT})^{\frac{2}{C-1}}})^{\nicefrac{C^2}{M^2}}$. When $C=M$, i.e. in a fully heterogeneous setting, this lower bound is the same as the lower bound in Theorem \ref{thm:2}, i.e. \citep{xu2023}, implying consistency. When $C < M$, this term is smaller by noting that $\frac{C^2}{M^2} < 1$, and thus by the monotone property of the function $1 - (\frac{1}{2} - \frac{1}{2}\sqrt{1 - (\frac{\delta}{CT})^{\frac{2}{C-1}}})^{\nicefrac{C^2}{M^2}}$ we have $1 - (\frac{1}{2} - \frac{1}{2}\sqrt{1 - (\frac{\delta}{CT})^{\frac{2}{C-1}}})^{\nicefrac{C^2}{M^2}} < 1   - (\frac{1}{2} - \frac{1}{2}\sqrt{1 - (\frac{\delta}{CT})^{\frac{2}{C-1}}}) = (\frac{1}{2} + \frac{1}{2}\sqrt{1 - (\frac{\delta}{CT})^{\frac{2}{C-1}}}) < (\frac{1}{2} + \frac{1}{2}\sqrt{1 - (\frac{\delta}{MT})^{\frac{2}{M-1}}}) $ which suggests that our assumption is less stringent compared to Theorem \ref{thm:2} herein and the original statement of Theorem 2 in \citep{xu2023decentralized}.

Moreover, we next show that the lower bound on $p(m,n)$ can be further reduced by modifying the proof, purely from an analytical perspective. This modification also applies to the setting in \citep{xu2023decentralized}, thereby providing an improvement to the result therein as well, as part of our contribution. The formal statement reads as follows.

\begin{theorem}\label{thm:4}
  Let us assume that $p(m,m) = 1$ for any $1 \leq m \leq C$. Let us further assume that $\min_{m,n}p(m,n) \geq  1-(1- \min\{(\frac{1}{2} + \frac{1}{2}\sqrt{1 - (\frac{\delta}{CT})^{\frac{2}{C-1}}}), 1 - \nicefrac{\delta (C-1)}{8CT}\})^{\nicefrac{C^2}{M^2}}$. The regret bound of Algorithm 2 with Rule 2 reads as with probability $1-7\epsilon$ $\allowbreak E[R_T|A_{\epsilon,\delta}^{\prime}]     \leq L + \sum_{i \neq i^*}(\max(\frac{C}{M} \cdot [\frac{4C_1\log T}{\Delta_i^2}], 2(K^2+MK)) +  \nicefrac{2\pi^2}{3P(A_{\epsilon, \delta})} + K^2 + (2M-1)K) = O(\nicefrac{C\log{T}}{M})$
where $A_{\epsilon,\delta}^{\prime}$, $L, c_0, C_1 $ are specified in Theorem \ref{thm:3}. 
\end{theorem}
\begin{proof}[Proof Sketch]
The complete proof is in Appendix \ref{app:proof}; we introduce the key logic here. The proof mostly follows from the proof of Theorem \ref{thm:3}, with the only exception being that the concentration inequality used to prove the following proposition, which in part guarantees the regret bounds by ensuring communication effectiveness given random graphs, considers both Chernoff's Bound (leading to the lower bound $(\frac{1}{2} + \frac{1}{2}\sqrt{1 - (\frac{\delta}{CT})^{\frac{2}{C-1}}})$) and Chebyshev's inequality (resulting in $1 - \frac{\delta (C-1)}{8CT}$). In contrast, Theorem \ref{thm:3} only considers Chernoff's Bound. The proposition reads as follows, characterizing the probability of event $A_3^{\prime}$. Assume the edge probability $p(m,n)$ where $c_m \neq c_n)$ meets the condition $1 \geq p(m,n) \geq \min\{(\frac{1}{2} + \frac{1}{2}\sqrt{1 - (\frac{\delta}{CT})^{\frac{2}{C-1}}}), 1 - \frac{\delta (C-1)}{8CT}\},$
    where $0 < \epsilon < 1$. Then, with probability $1 - \epsilon$, event $A_3^{\prime}$ holds.
\end{proof}
The assumption on $\min_{m,n}p(m,n)$ originates from the requirement to establish that $G_t^C$ is connected with high probability. We observe that, in practice, it may not always be feasible or necessary to have a connected graph $G_t^C$ at every time step. This assumption has been notoriously hard to overcome in MA-MAB. The situation becomes even more challenging in our context and in existing work related to random graphs, as the assumption essentially implies that the lower bound on the edge probability in Theorem \ref{thm:3} and existing work~\citep{xu2023decentralized} can approach $1$ when $T$ is large enough. This heavily constrains the applicability of the established results. 

However, this is no longer a concern in what follows, addressed by our new techniques, which represent a significant advancement in this line of work on multi-agent systems with random graphs, and thus highlight our contributions. Surprisingly, we find that as long as the subgraph $G_t^C$ is $l$-periodically connected, based on the following definition, the above regret bound holds, further relaxing the assumption on $\min_{m,n}p(m,n)$, which is shown to be strictly bounded away from $1$. We introduce the definitions related to $l$-periodically connected graphs and present the formal statement below.

\begin{definition}[Composition of graphs] Let us assume graphs $G^1, G^2, \ldots, G^l$  have the same vertex set $V$ and possibly different edge set $E_1, E_2, \ldots, E_l$. Then the composition of these $l$ graphs $G^1, G^2, \ldots, G^l$, known as $G = G_{1} \otimes G_{2} \otimes G_{3} \otimes \cdots \otimes G_{l}$, is uniquely defined by vertex set $V$ and edge set $E$ where $(i,j) \in E$ if and only if there exist vertex $v_2, v_3, v_{l-1}$ such that $(i, v_2) \in E_1, (v_1, v_2) \in E_2, (v_2, v_3) \in E_3, \ldots, (v_{l-1}, j) \in E_l$.

\end{definition}

\begin{definition}[$l$-periodically connected \citep{zhu2023distributed}]
  A sequence of time-dependent sub-graphs $G_t^C$ is said to be $l$-periodically connected in the sense that the composition of any $l \geq 1$ consecutive sub-graphs $G_{t_1}^C, G_{t_1+1}^C, G_{t_1+2}^C, \ldots, G_{t_1+l-1}^C$ is a connected graph, formally expressed as $G = G_{t_1}^C \otimes G_{t_1+1}^C \otimes G_{t_1+2}^C \otimes \cdots \otimes G_{t_1+l-1}^C$ is connected.
\end{definition}

It is worth noting that, in our context, $l$ is at most $C-1$ since there are $C$ vertices. Also, by definition, $l$ is a positive integer. Next, we demonstrate that any $l$ within this range specifies a lower bound on $\min_{m,n}p(m,n)$, and when this lower bound holds for $\min_{m,n}p(m,n)$, the same regret bound as in Theorem \ref{thm:4} also applies, but under much less stringent assumptions. 

First, based on the $l$-periodical connectivity, we have the following lemma hold which characterizes the relationship between $p(m,n)$ and $p(c_m, c_n)$ and is much stronger compared to Lemma \ref{lem:edge_prob_1}. 

\begin{lemma}\label{lem:edge_prob_2}
    For any pair of vertices $c_m, c_n$ in the sub-graph $G^C_t$, the probability that $c_m$ and $c_n$ is connected in $G^C_t$ is $p(c_m,c_n) \geq (1-\frac{1}{e})\min\{1, \frac{M^2}{C^2} \cdot p(m,n)\}$. 
\end{lemma}

The formal statement on regret is as follows. Here we assume $l \geq 2$, as $l=1$ implies connectivity. 
\begin{theorem}\label{thm:5}
  Let us assume that $p(m,m) = 1$ for any $1 \leq m \leq C$. Given any $C \geq 4$, $2 \leq l \leq C-1$, let us further assume that $\min_{m,n} p(m,n) \geq \frac{e}{e-1}\cdot\frac{C^2}{M^2}\cdot \max\{\frac{(C-l-1)!}{(C-2)!}(1 - \frac{\delta (C-1)}{8CT}), \frac{(C-l-1)!}{(C-2)!}(\frac{3}{4})^{\frac{1}{l}}\}$. The regret bound of Algorithm 2 with Rule 2 reads as with probability $1-7\epsilon$ $E[R_T|A_{\epsilon,\delta}^{\prime}]             \leq L + \sum_{i \neq i^*}\Delta_i(\max{\{\frac{C}{M} \cdot [\frac{4C_1\log T}{\Delta_i^2}], 2(K^2+MK) \}} +  \frac{2\pi^2}{3P(A_{\epsilon, \delta})} + K^2 + (2M-1)K)  = O(\frac{C\log{T}}{M})$ where $A_{\epsilon,\delta}^{\prime} = A_2 \cap A_3^{\prime}$, $A_2 = \{\exists t_0, \forall t \geq L, \forall j, \forall m, t+1 - \min_jt_{m,j} \leq t_0 \leq c_0\min_{l}n_{l,i}(t+1)\}$ and $A_3^{\prime} = \{\forall t \geq L, G_t^C \text{ is $l$-periodically connected}\}$, $L, c_0, C_1 $ are specified in Theorem \ref{thm:3}. 
\end{theorem}
\begin{proof}[Proof Sketch]
We refer to Appendix \ref{app:proof} for the detailed proof; the intuitive logic is as follows. The relaxation is obtained by proving the following: for any $m,i,t > L$, if $n_{m,i}(t) \geq 2(K^2+KM+M)$ and the sub-graph $G_t$ is $l$-periodically connected, then we have
 $\hat{N}_{m,i}(t) \leq 2\min_{j}\hat{N}_{j,i}(t),$ where the minimum is taken over all clusters, not just the neighbors. This is equivalent to proving that the agents stay on the same page and achieve consensus, which is guaranteed by $l$-periodically connected sub-graphs, not limited to connected sub-graphs. Based on the definition of the composition of graphs, the probability of having an edge is much larger than the edge probability $p(m,n)$. Thus, the degree of the composition graph is more likely to exceed $\frac{C-1}{2}$, which is a sufficient condition for connectivity, i.e., meeting the condition of being $l$-periodically connected.
\end{proof}

\paragraph{Choice of $l$} A natural question is how to specify such $2 \leq l \leq C-1$. Since the result in Theorem \ref{thm:5} holds for any $l$, an optimal choice of $l$ is the one that minimizes the lower bound on $\min_{m,n}p(m,n)$, which reads as $l^* = \arg\min_{l}\frac{e}{e-1}\frac{C^2}{M^2}\max\{\frac{(C-l-1)!}{(C-2)!}(1 - \frac{\delta (C-1)}{8CT}), \frac{(C-l-1)!}{(C-2)!}(\frac{3}{4})^{\frac{1}{l}}\}$. Nevertheless, we can always use any $l$, which improves the previous results, as stated in the following.

%\hl{this discussion needs to be changed as well} The existence of clusters relaxes the requirement on pair-wise communication and on the sample complexity of rewards.

\paragraph{Comparison with Theorem \ref{thm:3} \& \ref{thm:4}} The lower bound on the edge probability in Theorem \ref{thm:5} is given by 
$\min_{m,n}p(m,n) \allowbreak \geq \frac{e}{e-1}\frac{C^2}{M^2}\max\{\frac{(C-l-1)!}{(C-2)!}(1 - \frac{\delta (C-1)}{8CT}), \frac{(C-l-1)!}{(C-2)!}(\frac{3}{4})^{\frac{1}{l}}\}$. In contrast, the corresponding lower bounds in Theorems \ref{thm:3} and \ref{thm:4} are $1 - (\frac{1}{2} - \frac{1}{2}\sqrt{1 - (\frac{\delta}{CT})^{\frac{2}{C-1}}})^{\nicefrac{C^2}{M^2}}$ and $1-(1- \min\{(\frac{1}{2} + \frac{1}{2}\sqrt{1 - (\frac{\delta}{CT})^{\frac{2}{C-1}}}), 1 - \nicefrac{\delta (C-1)}{8CT}\})^{\nicefrac{C^2}{M^2}}$, respectively. When \(l > 1\), the lower bound in Theorem \ref{thm:5} can be significantly smaller than the lower bound in Theorem \ref{thm:4}. Specifically:
- The additional term \(\frac{(C-l-1)!}{(C-2)!}\) is always less than 1 (and substantially so due to factorial decay).
- The factor \(\nicefrac{C^2}{M^2}\) is also smaller when \(C < M\). These differences contribute significantly to relaxing the assumption on edge probabilities. Notably, in Theorems \ref{thm:3} and \ref{thm:4}, the lower bounds converge to 1 as \(T \to \infty\), implying that the graph becomes fully connected. However, this is no longer a concern in Theorem \ref{thm:5}. For the second term, consider the case where \(C = M\) is large. In this scenario:
$\frac{1}{2} + \frac{1}{2}\sqrt{1 - (\frac{\delta}{CT})^{\frac{2}{C-1}}} \approx 1 - \frac{1}{2}\frac{\delta}{CT} > 1 - \frac{\delta (C-1)}{8CT}$.
Thus, the lower bound in Theorem \ref{thm:4} becomes approximately: $1 - \frac{\delta (C-1)}{8CT}^{\nicefrac{C^2}{M^2}} = 1- \frac{\delta (C-1)}{8CT}$.
This demonstrates the improvement of Theorem \ref{thm:5} over Theorem \ref{thm:4} in this case. However, when \(C\) is small or \(C < M\), the improvement mainly comes from the additional term in Theorem \ref{thm:5}, as the difference between the second terms in Theorems \ref{thm:4} and \ref{thm:5} becomes negligible in comparison.

%When $\frac{1}{2} + \frac{1}{2}\sqrt{1 - (\frac{\delta}{CT})^{\frac{2}{C-1}}} < 1 - \frac{\delta (C-1)}{8CT}$, we have 

%This is not necessarily true when $C$ is small or $C < M$, which means in this case, the improvement mainly comes from the additional term.  

%Now we consider the case when $C$ is small. The result in Theorem \ref{thm:3} is obtained by the Chernoff's bound. However, when $l > 1$, $1_{m,n}$ is dependent on $1_{m,j}$, so we need to use Chebyshev's inequality instead of the Chernoff bound, which requires i.i.d. random variables $1_{m,n}$. It is known that for small $|C|$, the Chebyshev's inequality is stronger. That being said, the lower bound $(\frac{1}{2} + \frac{1}{2}\sqrt{1 - (\frac{\delta}{CT})^{\frac{2}{C-1}}})$ is larger than $\max{\{\frac{(C-l-1)!}{(C-2)!}(1 - \frac{\delta (C-1)}{8CT}), \frac{(C-l-1)!}{(C-2)!}(\frac{3}{4})^{\frac{1}{l}}\}}$, again implying the improvement of Theorem \ref{thm:4} over Theorem \ref{thm:3}.

It is worth noting that the above three results assume that the in-cluster probability $p(m,m) = 1$, which may be violated in some cases. To this end, we further relax this assumption by considering more general scenarios. We derive similar results, but only require a lower bound on $p(m,m)$ by additionally considering the $l$-periodically connected sub-graphs induced by the agents within the same cluster (motivated by considering the sub-graphs across clusters), rather than assuming $p(m,m) = 1$. Methodologically, in Algorithm \ref{alg:dr}, we set $\tau = l$ instead of $1$, which implies that the frequency of updating the within cluster information aligns with $l$-periodical connectivity. We next present the corresponding theoretical result below.

\begin{theorem}\label{thm:6}
  Let us assume that $\min_{m}p(m,m) \geq \max\{\frac{(|c_M|-l-1)!}{(|c_M|-2)!}(1 - \frac{\delta (|c_M|-1)}{8c_MT}), \frac{(|c_M|-l-1)!}{(|c_M|-2)!}(\frac{3}{4})^{\frac{1}{l}}\}$ for any $1 \leq m \leq C$, and that $\min_{m \neq n}p(m,n) \geq \frac{e}{e-1}\cdot\frac{C^2}{M^2}\cdot\max\{\frac{(C-l-1)!}{(C-2)!}(1 - \frac{\delta (C-1)}{8CT}), \frac{(C-l-1)!}{(C-2)!}(\frac{3}{4})^{\frac{1}{l}}\}$. The regret bound of Algorithm 2 with Rule 2 reads as with probability $1-7\epsilon$ $E[R_T|A_{\epsilon,\delta}^{\prime}] \leq L + \sum_{i \neq i^*}\Delta_i(\max{\{\frac{C}{M} \cdot [\frac{4C_1\log T}{\Delta_i^2}], 2(K^2+MK) \}} +  \frac{2\pi^2}{3P(A_{\epsilon, \delta})} + K^2 + (2M-1)K) + l  = O(\frac{C\log{T}}{M})$
where $A_{\epsilon,\delta}^{\prime}$ is defined in Theorem \ref{thm:5}, $L, c_0, C_1 $ are specified in Theorem \ref{thm:3}.
\end{theorem}
\begin{proof}[Proof Sketch]
The complete proof is referred to in Appendix \ref{app:proof}; here we illustrate the main proof logic. Instead of considering a complete sub-graph within one cluster, we consider $l$-periodically connected sub-graphs, which implies that the delay to receive all other agents ' information within the cluster is at most $l$. The assumption of $\min_{m}p(m,m) \geq \max{\{\frac{(|c_M|-l-1)!}{(|c_M|-2)!}(1 - \frac{\delta (|c_M|-1)}{8CT}), \frac{(|c_M|-l-1)!}{(|c_M|-2)!}(\frac{3}{4})^{\frac{1}{l}}\}}$ guarantees that, with high probability, the sub-graph induced by the agents in one cluster is $l$-periodically connected at all times. The proof then follows from the statement for sub-graphs induced by the clusters. The regret incurred in between is at most $l$. Otherwise, once an agent collects the information from all agents in the same cluster, the algorithm proceeds as before, where the cluster can be treated as a complete graph (since all information is available). Subsequently, the previous regret bound also holds, which provides the regret bound as stated in Theorem \ref{thm:6}.
\end{proof}

When interpreting the above result in more detail, we surprisingly find that it provides useful insights into the proposed framework, which incorporates both homogeneity and heterogeneity. We next illustrate these insights. If $|c_M| < C$, i.e., $M \leq C^2$, the lower bound for the edge probability $p(m,m)$ within one cluster is larger than that for the edge probability $p(m,n)$ across clusters where $m \neq n$. This implies that for small-scale systems with a large number of clusters, sufficient edges within a cluster (homogeneity) are required to achieve the order reduction in the regret bound. Conversely, when $|c_M| \gg C$, i.e., $M \gg C^2$, the reverse holds: the lower bound for the edge probability $p(m,m)$ within one cluster becomes smaller than that for the edge probability $p(m,n)$ across clusters where $m \neq n$. Hence, for large-scale systems, it is more important to ensure sufficient information is gathered across clusters (heterogeneity) to guarantee the regret bound.

%Lastly, we conclude this section with a discussion on the communication efficiency and computational complexity of the proposed algorithm, which further demonstrates its advantages.

%\paragraph{Discussion on the communication efficiency and computational complexity} 

\section{Heterogeneous - Multiple Unknown Clusters}\label{sec:heter-unknown}

In practice, the cluster structure is often unknown due to the complexity of real-world data. Simply assuming no cluster structure, as in \citep{xu2023decentralized}, is an overly simplistic approach that ignores the potential presence of clusters. This neglect not only oversimplifies the problem but also divert leveraging these structures to design more refined and effective algorithms. To overcome this limitation, and unlike the case where the cluster assignment $\{i \to c_i\}$ is known beforehand as in Section \ref{sec:heter}, we demonstrate that our methodological framework can be extended to scenarios where such information is unavailable. This extension opens the door to tackling more practical and complex real-world problems and providing robust tools.  We show that, with only minor modifications, our algorithm seamlessly integrates with a cluster detection algorithm (Section \ref{sec:heter-unknown-1}). We establish the corresponding regret bound under these extended conditions in Section \ref{sec:heter-unknown-2}.

\subsection{Algorithm}\label{sec:heter-unknown-1}

The method involves two steps: first, incorporating a cluster detection method (Algorithm \ref{alg:IR-LSS}) to estimate the cluster structure with reward information given by the burn-in period (Algorithm \ref{alg:burn-in}); and second, running Algorithm \ref{alg:dr} using this estimation in a plug-and-play fashion.

\subsubsection{Cluster Detection}\label{sec:heter-unknown-1-1}

To estimate the cluster structure from the graph and the reward information, we use the algorithm for cluster detection in~\cite{braun2022iterative}, which extends the stochastic block model to contextual stochastic block models by incorporating node covariates as side information.

\begin{definition}
    The Contextual Symmetric Stochastic Block Model (CSSBM) is a special stochastic block model with node covariates. Given the graph size $M$, number of clusters $C$, edge connection probabilities $p,q$, $C$ vectors $\mu^1,\cdots, \mu^C \in \mathbb{R}^K$ and variance $\sigma^2$, it generates a graph with $M$ nodes and $C$ balanced clusters from a stochastic block model with edge probability $p(m,m) = p$ for any $m$ and $p(m,n) = q$ for any $m \neq n$. The node covariates are generated from a Gaussian mixture model, where for each node $i$ in cluster $c_m$, its covariates $v^i$ are sampled from Gaussian distribution $\mathcal{N}(\mu^m, \sigma^2 I_K)$. 
\end{definition}

Braun et al.~\cite{braun2022iterative} proposed an iterative clustering algorithm IR-LSS to recover the cluster structure under CSSBM by incorporating both graph structure and vertex covariates (rewards in our case).
We provide the pseudo-code of the IR-LSS algorithm (Algorithm 3) in Appendix~\ref{apx:IR-LSS}. 
The IR-LSS algorithm iteratively refines clusters by alternating between the following two steps - 1) Parameter Estimation: Using the current clustering, it estimates model parameters, such as the connectivity probability and covariate means, and 2) Clustering Refinement: It reassigns vertices to clusters by minimizing a least-squares criterion, based on estimated model parameters.

This iterative approach effectively leverages both graph and reward information, addressing cases where clusters are difficult to recover from either source alone, aligning perfectly with MA-MAB.

\subsubsection{Full Algorithm} 
The full algorithm first runs Algorithm \ref{alg:burn-in} to accumulate reward information and then executes Algorithm 3 to detect the cluster information, both of which are burn-in steps. Note that for each agent to run the clustering algorithm locally, an additional $O(M)$ steps are required to propagate the estimated reward information and local edge information across the entire graph. Thus, the new burn-in steps have a total length of $L_1 = O(L + M)$, where $L = O(\frac{C}{M} \cdot K \cdot \log T)$ represents the length of Algorithm~\ref{alg:burn-in}. Subsequently, the full algorithm proceeds to the learning period by running Algorithm \ref{alg:dr}, as before, to design UCB-based strategies.

\subsection{Analyses}\label{sec:heter-unknown-2}

Again, let us assume that the cluster structure is balanced for illustration purposes. As before, the imbalanced case can be easily handled by using the smallest cluster size.

We first present the key results regarding the cluster detection algorithm. By using the iterative refinement clustering algorithm by~\cite{braun2022iterative}, the cluster structure of the CSSBM can be exactly recovered in the following regime efficiently with high probability. First, we define the Signal-to-Noise Ratio (SNR) to be $\mathrm{SNR} = \frac{1}{8\sigma^2} \min_{m\neq n}\|\mu^m - \mu^n\|_2^2 + \frac{\log M}{C}(\sqrt{p'} - \sqrt{q'})^2, $
where $p = p' \frac{\log M}{M}$ and $q = q' \frac{\log M}{M}$.

\begin{lemma}[Braun et al.~\cite{braun2022iterative}]\label{lemma:CSSBM}
    Consider a CSSBM with $M$ nodes, $C$ clusters, and signal-to-noise ratio $\mathrm{SNR}$. Suppose $\mathrm{SNR} > 2\log M$ and $C^{3} \leq \mathrm{SNR} \cdot \delta$ for a small constant $\delta < 1/2$. Then, with probability at least $1 - 1/\mathrm{poly}(M)$, Algorithm~\ref{alg:IR-LSS} exactly recovers the cluster structure.
    %\label{lem:recoverCSSBM}
\end{lemma}

Then, we use the graph information and the reward estimation for each agent from the burn-in period with $L = O(\frac{C}{M}\cdot K\log T)$ rounds as the input for the above clustering algorithm. 
We assume that the reward for each arm $i \in [K]$ of agent $m \in [M]$ is sampled from a Gaussian distribution $\mathcal{N}(\mu^m_i,\sigma^2)$.
Then, we can recover the cluster structure exactly under the following conditions. 

\begin{lemma}\label{lem:recover}
    Consider the graph is generated from a stochastic block model with $M$ agents and $C$ balanced clusters with edge connection probability $p(m,m) = p$ for all $m \in [C]$ and $p(m,n) = q$ for all $m \neq n$, where $p = p'\frac{\log M}{M}$ and $q = q' \frac{\log M}{M}$.
    The reward for each arm $i \in [K]$ of agent $m \in [M]$ is sampled from a Gaussian distribution $\mathcal{N}(\mu^m_i, \sigma^2)$.
    Suppose $\mathrm{SNR} > 2 \log M$ and $C^3 \leq \mathrm{SNR} \cdot  \delta$ for a small constant $\delta < 1/2$, where $\mathrm{SNR} = \frac{C\log T}{8 M\sigma^2} \min_{m\neq n}\|\mu^m - \mu^n\|_2^2 + \frac{K\log T \log M}{M}(\sqrt{p'} - \sqrt{q'})^2.$
    Then, with probability at least $1 - 1/\mathrm{poly}(M)$, Algorithm~\ref{alg:IR-LSS} exactly recovers the cluster structure.
\end{lemma}

The proof of Lemma \ref{lem:recover} is provided in Appendix \ref{app:proof-lemma}. We would like to emphasize that, given the high probability of detecting the cluster structure, the regret bounds introduced in Section \ref{sec:heter} remain valid with certain modifications. Specifically, the probability of the regret bound holds as the product of the original probability and the probability of exact cluster detection. Additionally, the regret bound includes an extra term of order $M$, arising from the burn-in period $L_1$ during which the cluster detection algorithm is executed.

Formally, we present the following corollary of Theorem \ref{thm:6}, removing the assumption of known clusters from the theorem. We show the corollary of Theorem \ref{thm:6} for illustrative purposes, as Theorem \ref{thm:6} represents the most general form of the regret bound under the least stringent assumptions. Likewise, the corollaries of Theorems \ref{thm:2}, \ref{thm:3}, \ref{thm:4}, and \ref{thm:5} also hold, as presented in Appendix \ref{app:Theory}.

We first define event \( A_{\epsilon, \delta, \tau} = A_{\epsilon, \delta} \cap \{c_i = c_i^{\prime}, \forall 1 \leq i \leq M\} \), where \( c_i^{\prime} \) is the cluster label for agent \( i \) recovered by Algorithm 3. It is straightforward to show that \( P(A_{\epsilon, \delta, \tau}) \geq 1 - 7\epsilon - \tau \), where \( \tau = P(\{c_i \neq c_i^{\prime}, \exists 1 \leq i \leq M\}) = 1 - 1/\mathrm{poly}(M)\) is the failure probability for the cluster recovery.

\begin{corollary}[Extension of Theorem \ref{thm:6}]\label{cor:6}
Let the assumptions in Theorem \ref{thm:6} hold except that $\{c_i\}$ are unknown. Let the assumptions in Lemma \ref{lem:edge_prob_2} hold. Then the regret of Full Algorithm reads as with probability $1-7\epsilon - 1/\mathrm{poly}(M)$, $E[R_T|A_{\epsilon,\delta, \tau}] \leq L_1 + \sum_{i \neq i^*}(\max\{\frac{C}{M} \cdot \frac{4C_1\log T}{\Delta_i^2}, 2(K^2+MK) \} +  \frac{2\pi^2}{3P(A_{\epsilon, \delta, \tau})} + K^2 + (2M-1)K) + l \leq O(\frac{C}{M}\log T)$.
\end{corollary}

\begin{proof}
 The proof is straightforward, combining the above lemma, which implies that after \( L_1 \), with probability \( 1 - \tau \), the cluster structure is correctly identified. From \( L_1 \) to \( T \), the agents run the same algorithm as before, resulting in the regret bound in Theorem \ref{thm:6}, under \( A_{\epsilon, \delta, \tau} \).
\end{proof}

This corollary demonstrates that our algorithm and analysis are general, making them applicable to scenarios where the cluster information is unknown and, thus, valuable for many real-world applications, as mentioned in Section \ref{sec:app}.

\section{Numerical Experiments}\label{sec:exp}

In this section, we numerically evaluate the performance of the proposed algorithm. Specifically, we compare the regret performance of Algorithm 2 over time with existing benchmark methods studied in the literature. We present the regret curves by computing the algorithms' exact regret as the average over 25 runs, along with the corresponding 95\% confidence intervals (CI). We present the results on both synthetic datasets and a real-world dataset in this section. Details about numerical experiments are referred to Appendix \ref{app:exp}.  %\textcolor{black}{Appendix~\ref{?}}. 

%We conduct experiments on both synthetic datasets and real-world datasets, corresponding to settings with known clusters and unknown clusters, respectively.

%With the synthetic dataset, motivated by the fact that the problem setting and the established results depend heavily on the parameters of the problem formulation, including \(M\), \(C\), and \(c_M\), we examine the performance across different settings with varying parameters. This allows us to characterize the dependency of the exact regret (rather than the regret bounds in the theorems) on the problem setting and gain insights into the algorithm's practical regret behavior. %We present our results in Section \ref{sec:exp.1}.

%Additionally, we consider real-world datasets, consistent with existing work on stochastic block models. \hl{To do - introduction of this dataset}. The work \hl{reference} provides the ground truth for a specific setting where the cluster structure is learned. We evaluate the regret performance of these algorithms and report the findings in Section \ref{sec:exp.2}.

The benchmarks include DrFed-UCB~\citep{xu2023decentralized}, GoSInE~\citep{chawla2020gossiping}, Gossip\_UCB~\citep{zhu2021federated}, and Dist\_UCB~\citep{zhu2023distributed}. Notably, GoSInE and Gossip\_UCB are tailored for time-invariant graphs, whereas Dist\_UCB and DrFed-UCB are recent works designed for time-varying graphs. Among these, we emphasize that DrFed-UCB is the most recent and highly relevant to our framework, as it considers Erdos-Renyi models and pure heterogeneity, while GoSInE focuses solely on homogeneity.

%\subsection{Synthetic Datasets}\label{sec:exp.1}

\begin{figure}
\vspace{-5mm}
\centering
\hspace{8mm}\includegraphics[width=0.9\textwidth]{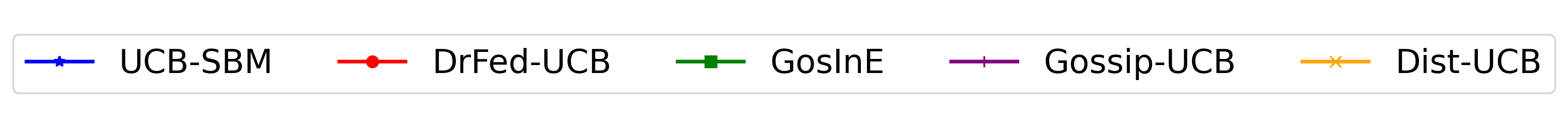}\label{fig:leg}
\vspace{-8mm}
\newline
\subfloat[][Synthetic data]
{\includegraphics[width=0.33\textwidth]{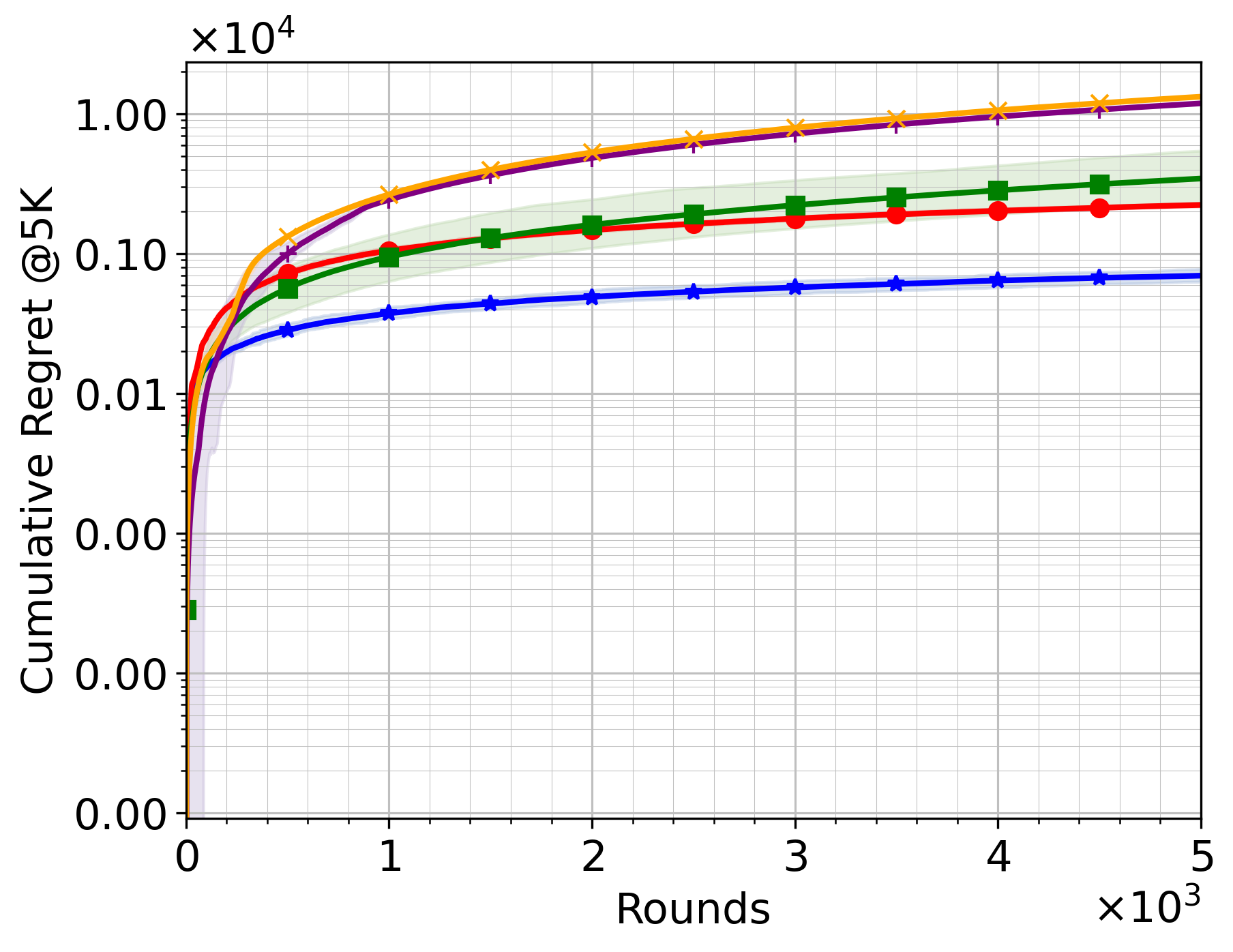}\label{fig:syn}}
\hfill
\subfloat[][Real-world data]
{\includegraphics[width=0.33\textwidth]{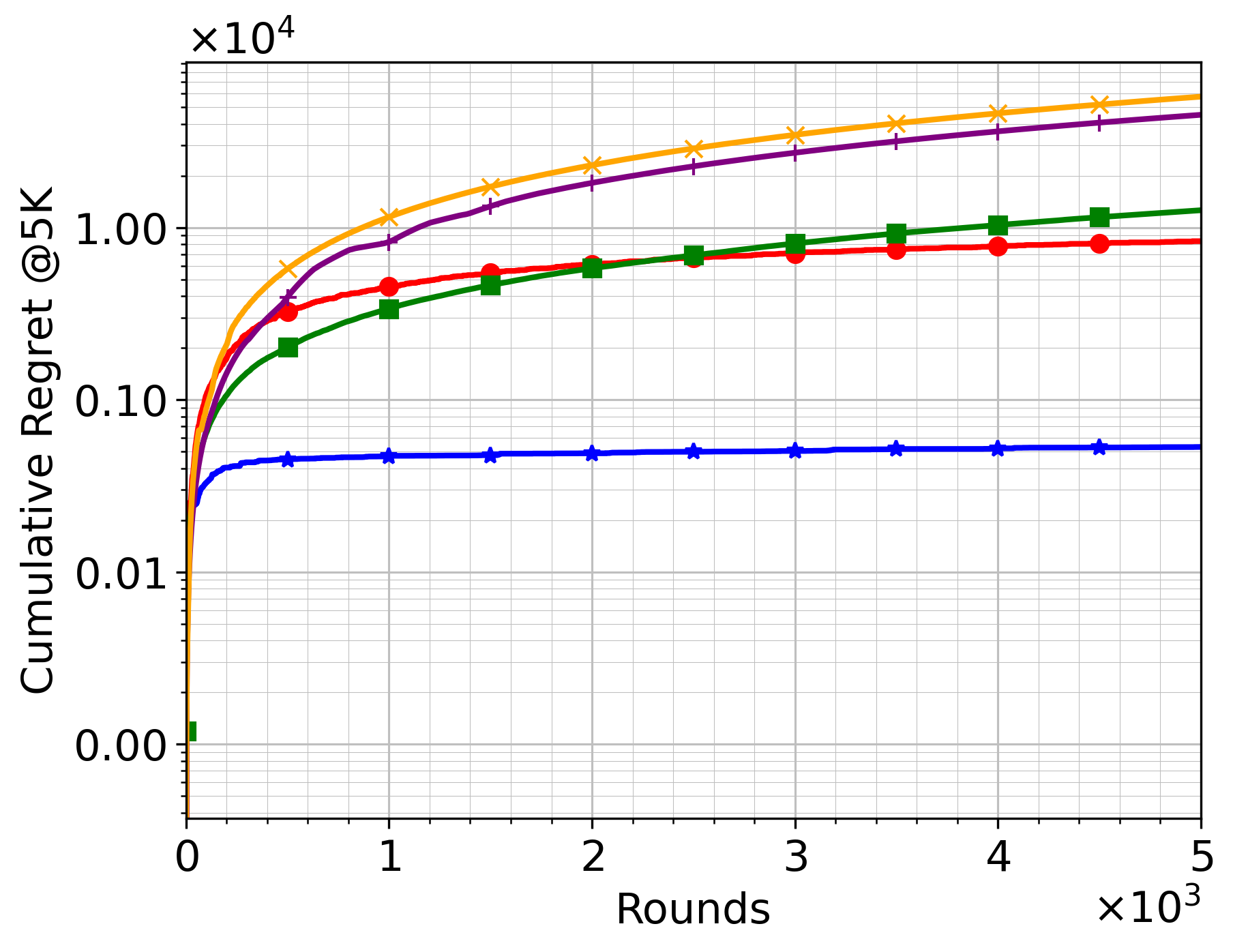}\label{fig:r}}
\hfill
\subfloat[][Changing $M$]
{\includegraphics[width=0.33\textwidth]{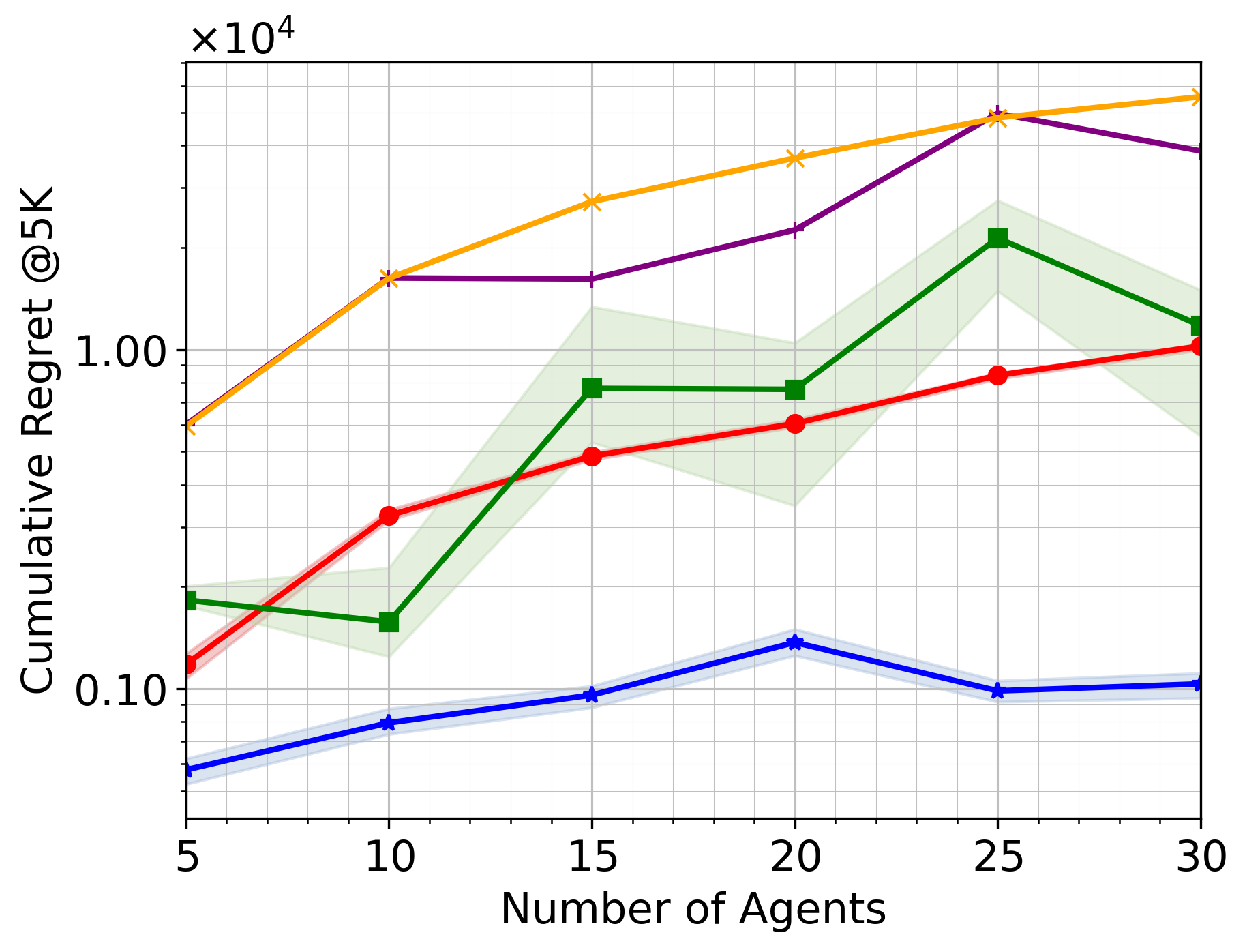}\label{fig:M}}
\vspace{-6mm}
\newline
\subfloat[][Changing $C$]
{\includegraphics[width=0.33\textwidth]{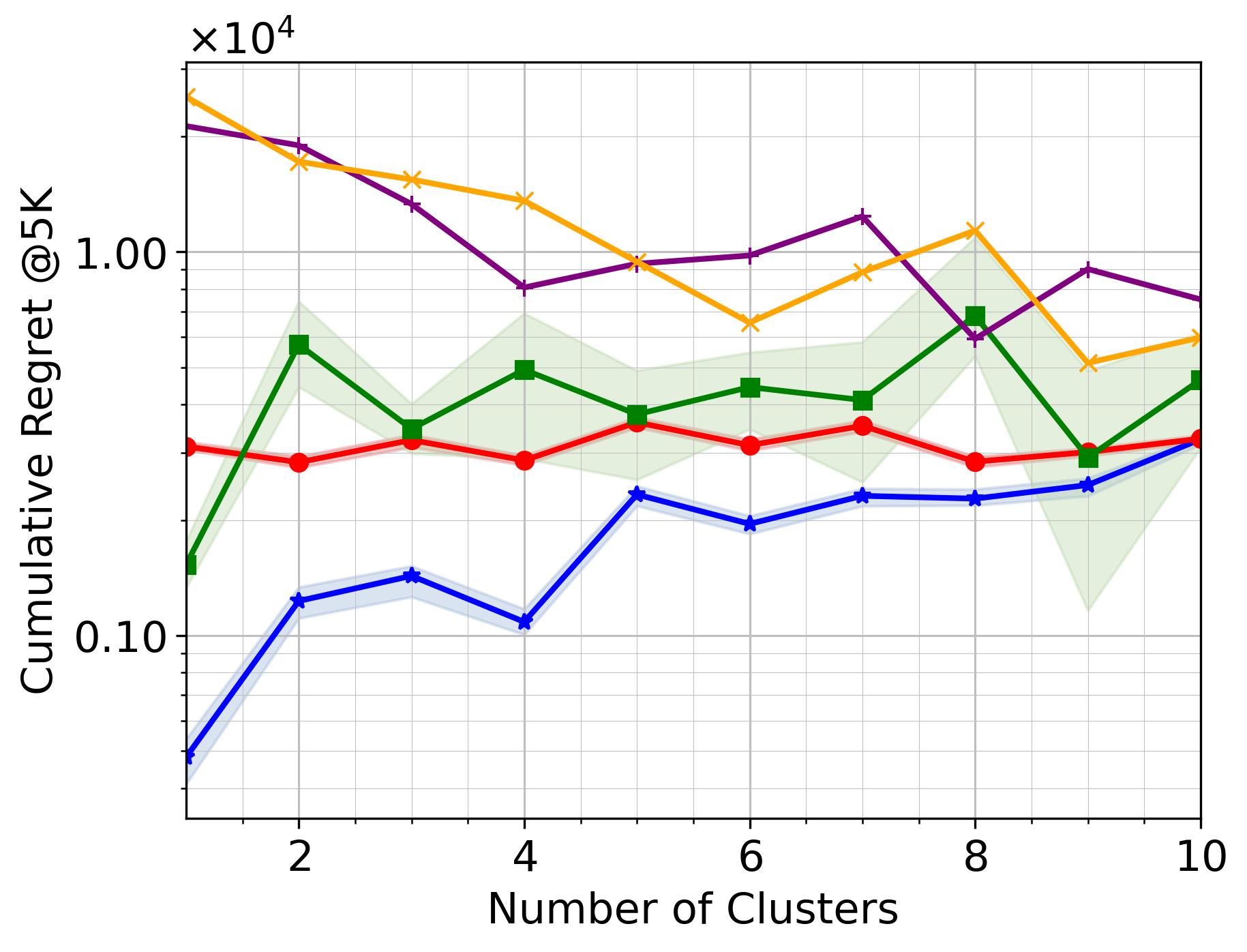}\label{fig:C}}
\hfill
\subfloat[][Changing $p(m,m)$]
{\includegraphics[width=0.33\textwidth]{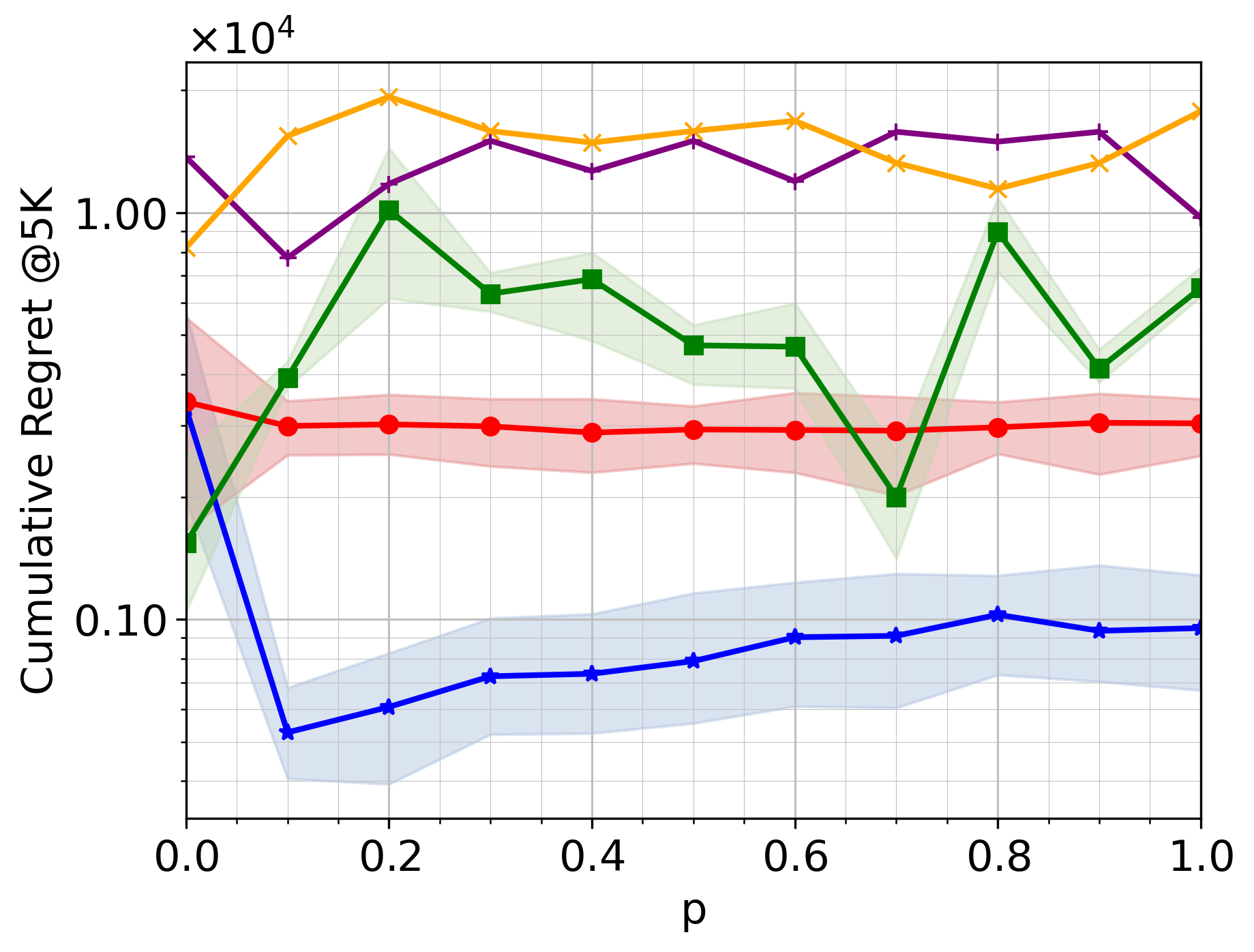}\label{fig:p}}
\hfill
\subfloat[][Changing $q(m,n)$]
{\includegraphics[width=0.33\textwidth]{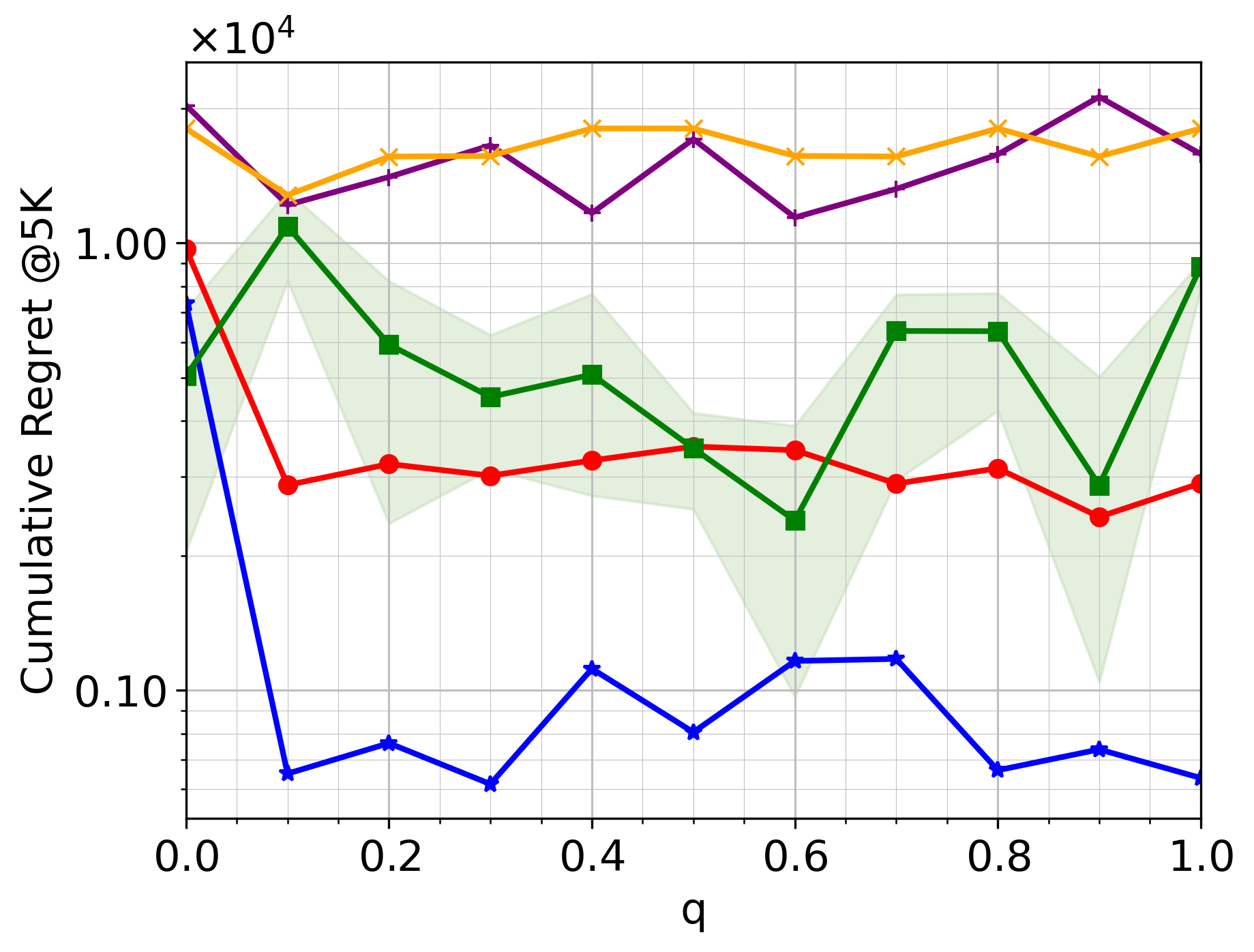}\label{fig:q}}
\vspace{-5mm}
\caption{ The regret of different methods across different settings}
\label{fig 1: comparison}
\end{figure}
%\vspace{-5mm}
\paragraph{Benchmark Comparison Results} The comparison of the regret curves on the synthetic dataset between UCB-SBM (our method) and benchmarks is shown in Fig. \ref{fig:syn}, where the shaded area denotes the CI. Here, the $x$-axis and $y$-axis represent the time $t$ and the cumulative regret on a log scale $\log{R_t}$, respectively. Among these, UCB-SBM achieves the smallest regret. We observe that UCB-SBM consistently demonstrates significantly smaller regret, showcasing notable improvements and highlighting the advantages of considering homogeneity within clusters and relaxing the assumption on edge probability. More precisely, the improvements in average regret $R_t$ compared to DrFed-UCB, GoSInE, Gossip\_UCB, and Dist\_UCB are \textbf{68.79\%}, \textbf{79.80\%}, \textbf{94.14\%}, and \textbf{94.75\%}, respectively. The comparison with DrFed-UCB emphasizes the heavy performance degradation of DrFed-UCB when neglecting the cluster structure. Meanwhile, our algorithm exhibits small variances (with GoSInE showing the largest), indicating stability even with time-varying graphs. Likewise, we draw similar conclusions from the real-world dataset, as presented in \textcolor{black}{Fig. \ref{fig:r}, wherein our improvement is even more significant.} %\mo{why not 1b?}

%In terms of time complexity, \hl{XX}. 

%Lastly, the communication costs of DrFed-UCB, Gossip\_UCB, and Dist\_UCB are of order $T$, while those of GoSInE exhibit only $o(T)$. This points toward a potential direction for optimizing communication costs, along with privacy guarantee.

\paragraph{Regret Dependency Results} Moreover, we demonstrate how the actual regret of UCB-SBM depends on several parameters associated with the problem setting, including the number of agents $M$, the number of clusters $C$, and the parameters $p = p(m,m)$ and $q = p(m,n)$ (which also affect $|p(m,m) - p(m,n)|$). While other parameters, such as $K$ and the difference in mean values $h$, are important, they have been studied in \citep{xu2023decentralized}. The aforementioned parameters are unique to our problem setting and necessitate examining how the actual regret changes with them beyond the theoretical upper bounds. The results of varying $M$, $C$, $p$, and $q$ are shown in Figs.~\ref{fig:M}, \ref{fig:C}, \ref{fig:p}, and~\ref{fig:q}, respectively. First, across all possible settings, our algorithm consistently achieves the smallest regret, and its performance does not change dramatically with different parameters, demonstrating both the effectiveness and robustness of the algorithm. In Fig. \ref{fig:M}, we observe that, except for UCB-SBM, all other algorithms exhibit an increasing trend as $M$ increases, while UCB-SBM remains steady, implying that UCB-SBM scales much better with $M$ and is thus more practical. In Fig. \ref{fig:C}, UCB-SBM's regret increases with $C$, consistent with the theoretical bound's dependence on $C$. Notably, when $C=10=M$, i.e., the fully heterogeneous case, UCB-SBM and DrFed-UCB achieve the same regret, validating the consistency of the results. The regret of UCB-SBM increases and then decreases with $p$ and $q$, as shown in Fig. \ref{fig:p} and Fig. \ref{fig:q}. This is possibly because lower bounds on $p$ and $q$ are necessary for the theoretical regret bounds to hold, as this pattern holds true for DrFed-UCB. Establishing an explicit dependency of $R_T$ on $p$ and $q$ is left for future exploration. %Lastly, we also provide a summary of the paper and discuss some exciting directions for future work in Appendix \ref{app:conclusion}.
  %.In Fig. \ref{fig:p},  We conclude how regret \hl{varies with the number of agents and the number of clusters, respectively, when fixing one or the other}. \hl{Additionally, we examine how the regret changes with $p(m,m)$, $p(m,n)$, and their difference $|p(m,m) - p(m,n)|$ when both $M$ and $C$ are fixed}.

\section{Conclusion and Future Work}\label{sec:conclusion}

In this paper, we novelly study the multi-agent multi-armed bandit (MA-MAB) problem, where agents are distributed on random graphs induced by a cluster structure, namely stochastic block models, and their reward mean values also depend on the cluster structure. This introduces both homogeneity and heterogeneity in edge probabilities and rewards, within and across clusters. The cluster assignment can be either known or unknown. This is the first framework that unifies the existing formulations of both homogeneous MA-MAB (1 cluster) and heterogeneous MA-MAB ($M$ clusters), smoothly capturing more general cases in between and reflecting different degrees of heterogeneity. Algorithmically, we propose a new method where agents within one cluster aggregate their information to achieve sample complexity reduction, communicate with other clusters to collect heterogeneous information, integrate this information to estimate the globally optimal arm, and pull arms based on newly designed UCB indices. When the cluster assignment is unknown, the agents leverage a cluster detection algorithm to estimate the cluster assignment, and our algorithm operates in a plug-and-play fashion, demonstrating its generalization ability. This approach leads to significantly improved results under less stringent assumptions. Theoretically, we show that the regret bound has a constant reduction of $\frac{C}{M}$, uncovering how the regret bound changes with the degree of heterogeneity and improving upon existing work \citep{xu2023decentralized}, beyond solely $T$. Moreover, the assumption on the minimal edge probability of the random graph is significantly relaxed, scaling better with $T$. Notably, while the minimal edge probability in existing work can approach $1$ as $T \to \infty$, our approach bounds it by much smaller values (e.g., $\frac{e}{e-1}\frac{C^2}{M^2}$ and $\frac{e}{e-1}\frac{C^2}{M^2} \frac{(C-l-1)!}{(C-2)!}$). Numerically, we demonstrate the superior performance of the proposed algorithm by comparing it with benchmarks. Consistently, our algorithm shows significant regret improvement, with the relative improvement percentage being at least $68\%$. We also examine how actual regret changes with parameters unique to the framework, consistent with the theoretical findings.

Moving forward, we identify several promising directions for future work. First, while we assume a balanced cluster structure or use the minimal cluster size to run the algorithm in unbalanced cases, it would be interesting to explore how to fully leverage unbalanced cluster structures instead of relying solely on the minimal cluster size. Additionally, while we assume that the reward distribution is sub-Gaussian (and can be extended to sub-exponential cases), more general heavy-tailed distributions present another direction for future research. Lastly, exploring other types of cluster structures, beyond stochastic block models, and characterizing how regret changes with these structures would be of great interest to both theorists and practitioners.

%\subsection{Real-world Datasets}\label{sec:exp.2}

%\paragraph{Benchmark Comparison Results} Again, we leverage average regret over \hl{XX} runs and the corresponding confidence intervals as evaluation metrics. Here, the experimental setting is determined by the real-world dataset. The comparison of these algorithms is presented in Figure \hl{XX}. Notably, our UCB-SBM outperforms all other algorithms, demonstrating superior performance in real-world applications and highlighting its practical potential. The exact regret improvements compared to DrFed-UCB, GoSInE, Gossip\_UCB, and Dist\_UCB are \hl{XX\%, XX\%, XX\%, and XX\%, respectively.}

%\hl{check the numerical experiment and the writing}

%We also present a summary of the paper and discuss some exciting future work in Appendix. 

\newpage

\bibliography{software}

\begin{thebibliography}{67}
\providecommand{\natexlab}[1]{#1}
\providecommand{\url}[1]{\texttt{#1}}
\expandafter\ifx\csname urlstyle\endcsname\relax
  \providecommand{\doi}[1]{doi: #1}\else
  \providecommand{\doi}{doi: \begingroup \urlstyle{rm}\Url}\fi

\bibitem[Abbe(2018)]{abbe2018community}
E.~Abbe.
\newblock Community detection and stochastic block models: recent developments.
\newblock \emph{Journal of Machine Learning Research}, 18\penalty0 (177):\penalty0 1--86, 2018.

\bibitem[Abbe et~al.(2015)Abbe, Bandeira, and Hall]{abbe2015exact}
E.~Abbe, A.~S. Bandeira, and G.~Hall.
\newblock Exact recovery in the stochastic block model.
\newblock \emph{IEEE Transactions on information theory}, 62\penalty0 (1):\penalty0 471--487, 2015.

\bibitem[Abbe et~al.(2022)Abbe, Fan, and Wang]{abbe2022}
E.~Abbe, J.~Fan, and K.~Wang.
\newblock An lp theory of pca and spectral clustering.
\newblock \emph{The Annals of Statistics}, 50\penalty0 (4):\penalty0 2359--2385, 2022.

\bibitem[Agarwal et~al.(2022)Agarwal, Aggarwal, and Azizzadenesheli]{agarwal2022multi}
M.~Agarwal, V.~Aggarwal, and K.~Azizzadenesheli.
\newblock Multi-agent multi-armed bandits with limited communication.
\newblock \emph{The Journal of Machine Learning Research}, 23\penalty0 (1):\penalty0 9529--9552, 2022.

\bibitem[Airoldi et~al.(2006)Airoldi, Blei, Fienberg, Xing, and Jaakkola]{airoldi2006mixed}
E.~M. Airoldi, D.~M. Blei, S.~E. Fienberg, E.~P. Xing, and T.~Jaakkola.
\newblock Mixed membership stochastic block models for relational data with application to protein-protein interactions.
\newblock In \emph{Proceedings of the international biometrics society annual meeting}, volume~15, page~1, 2006.

\bibitem[Auer et~al.(2002{\natexlab{a}})Auer, Cesa-Bianchi, and Fischer]{auer2002finite}
P.~Auer, N.~Cesa-Bianchi, and P.~Fischer.
\newblock Finite-time analysis of the multiarmed bandit problem.
\newblock \emph{Machine Learning}, 47\penalty0 (2-3):\penalty0 235--256, 2002{\natexlab{a}}.

\bibitem[Auer et~al.(2002{\natexlab{b}})Auer, Cesa-Bianchi, Freund, and Schapire]{auer2002nonstochastic}
P.~Auer, N.~Cesa-Bianchi, Y.~Freund, and R.~E. Schapire.
\newblock The nonstochastic multiarmed bandit problem.
\newblock \emph{SIAM Journal on Computing}, 32\penalty0 (1):\penalty0 48--77, 2002{\natexlab{b}}.

\bibitem[Ban et~al.(2024)Ban, Qi, Wei, Liu, and He]{ban2024meta}
Y.~Ban, Y.~Qi, T.~Wei, L.~Liu, and J.~He.
\newblock Meta clustering of neural bandits.
\newblock In \emph{Proceedings of the 30th ACM SIGKDD Conference on Knowledge Discovery and Data Mining}, pages 95--106, 2024.

\bibitem[Battiston and Catanzaro(2004)]{battiston2004statistical}
S.~Battiston and M.~Catanzaro.
\newblock Statistical properties of corporate board and director networks.
\newblock \emph{The European Physical Journal B}, 38:\penalty0 345--352, 2004.

\bibitem[Bistritz and Leshem(2018)]{bistritz2018distributed}
I.~Bistritz and A.~Leshem.
\newblock Distributed multi-player bandits-a game of thrones approach.
\newblock \emph{Advances in Neural Information Processing Systems}, 31, 2018.

\bibitem[Blaser et~al.(2024)Blaser, Li, and Wang]{blaser2024federated}
E.~Blaser, C.~Li, and H.~Wang.
\newblock Federated linear contextual bandits with heterogeneous clients.
\newblock In \emph{International Conference on Artificial Intelligence and Statistics}, pages 631--639. PMLR, 2024.

\bibitem[Braun et~al.(2022)Braun, Tyagi, and Biernacki]{braun2022iterative}
G.~Braun, H.~Tyagi, and C.~Biernacki.
\newblock An iterative clustering algorithm for the contextual stochastic block model with optimality guarantees.
\newblock In \emph{International Conference on Machine Learning}, pages 2257--2291. PMLR, 2022.

\bibitem[Chawla et~al.(2020)Chawla, Sankararaman, Ganesh, and Shakkottai]{chawla2020gossiping}
R.~Chawla, A.~Sankararaman, A.~Ganesh, and S.~Shakkottai.
\newblock The gossiping insert-eliminate algorithm for multi-agent bandits.
\newblock In \emph{International conference on artificial intelligence and statistics}, pages 3471--3481. PMLR, 2020.

\bibitem[Chen et~al.(2018)Chen, Xu, Ren, and Zhou]{chen2018spatio}
L.~Chen, J.~Xu, S.~Ren, and P.~Zhou.
\newblock Spatio--temporal edge service placement: A bandit learning approach.
\newblock \emph{IEEE Transactions on Wireless Communications}, 17\penalty0 (12):\penalty0 8388--8401, 2018.

\bibitem[Cugmas et~al.(2020)Cugmas, Mali, and {\v{Z}}iberna]{cugmas2020scientific}
M.~Cugmas, F.~Mali, and A.~{\v{Z}}iberna.
\newblock Scientific collaboration of researchers and organizations: a two-level blockmodeling approach.
\newblock \emph{Scientometrics}, 125\penalty0 (3):\penalty0 2471--2489, 2020.

\bibitem[Dai et~al.(2024)Dai, Zhang, Yang, Xu, Liu, and Lui]{dai2024axiomvision}
X.~Dai, Z.~Zhang, P.~Yang, Y.~Xu, X.~Liu, and J.~C. Lui.
\newblock Axiomvision: Accuracy-guaranteed adaptive visual model selection for perspective-aware video analytics.
\newblock In \emph{Proceedings of the 32nd ACM International Conference on Multimedia}, pages 7229--7238, 2024.

\bibitem[Delarue(2017)]{delarue2017mean}
F.~Delarue.
\newblock Mean field games: A toy model on an {E}rd{\"o}s-{R}enyi graph.
\newblock \emph{ESAIM: Proceedings and Surveys}, 60:\penalty0 1--26, 2017.

\bibitem[Deshpande et~al.(2018)Deshpande, Sen, Montanari, and Mossel]{deshpande2018contextual}
Y.~Deshpande, S.~Sen, A.~Montanari, and E.~Mossel.
\newblock Contextual stochastic block models.
\newblock \emph{Advances in Neural Information Processing Systems}, 31, 2018.

\bibitem[Dreveton et~al.(2024)Dreveton, Fernandes, and Figueiredo]{dreveton2024exact}
M.~Dreveton, F.~Fernandes, and D.~Figueiredo.
\newblock Exact recovery and bregman hard clustering of node-attributed stochastic block model.
\newblock \emph{Advances in Neural Information Processing Systems}, 36, 2024.

\bibitem[Dubey and Pentland(2730--2739, 2020)]{dubey2020cooperative}
A.~Dubey and A.~Pentland.
\newblock Cooperative multi-agent bandits with heavy tails.
\newblock In \emph{International Conference on Machine Learning}, 2730--2739, 2020.

\bibitem[Duchemin(2023)]{duchemin2023reliable}
Q.~Duchemin.
\newblock Reliable prediction in the markov stochastic block model.
\newblock \emph{ESAIM: Probability and Statistics}, 27:\penalty0 80--135, 2023.

\bibitem[El~Haj(2024)]{el2024community}
A.~El~Haj.
\newblock Community detection in multiplex continous weighted nodes networks using an extension of the stochastic block model.
\newblock \emph{Computing}, 106\penalty0 (11):\penalty0 3711--3725, 2024.

\bibitem[ERDdS and R\&wi(1959)]{erdds1959random}
P.~ERDdS and A.~R\&wi.
\newblock On random graphs i.
\newblock \emph{Publ. math. debrecen}, 6\penalty0 (290-297):\penalty0 18, 1959.

\bibitem[Gentile et~al.(2014)Gentile, Li, and Zappella]{gentile2014online}
C.~Gentile, S.~Li, and G.~Zappella.
\newblock Online clustering of bandits.
\newblock In \emph{International conference on machine learning}, pages 757--765. PMLR, 2014.

\bibitem[Gentile et~al.(2017)Gentile, Li, Kar, Karatzoglou, Zappella, and Etrue]{gentile2017context}
C.~Gentile, S.~Li, P.~Kar, A.~Karatzoglou, G.~Zappella, and E.~Etrue.
\newblock On context-dependent clustering of bandits.
\newblock In \emph{International Conference on machine learning}, pages 1253--1262. PMLR, 2017.

\bibitem[Holland et~al.(1983)Holland, Laskey, and Leinhardt]{holland1983stochastic}
P.~W. Holland, K.~B. Laskey, and S.~Leinhardt.
\newblock Stochastic blockmodels: First steps.
\newblock \emph{Social networks}, 5\penalty0 (2):\penalty0 109--137, 1983.

\bibitem[Huang et~al.(2021)Huang, Wu, Yang, and Shen]{huang2021federated}
R.~Huang, W.~Wu, J.~Yang, and C.~Shen.
\newblock Federated linear contextual bandits.
\newblock \emph{Advances in Neural Information Processing Systems}, 34:\penalty0 27057--27068, 2021.

\bibitem[Jiang and Cheng(1--33, 2023)]{jiang2023multi}
F.~Jiang and H.~Cheng.
\newblock Multi-agent bandit with agent-dependent expected rewards.
\newblock \emph{Swarm Intelligence}, 1--33, 2023.

\bibitem[Korda et~al.(2016)Korda, Szorenyi, and Li]{korda2016distributed}
N.~Korda, B.~Szorenyi, and S.~Li.
\newblock Distributed clustering of linear bandits in peer to peer networks.
\newblock In \emph{International conference on machine learning}, pages 1301--1309. PMLR, 2016.

\bibitem[Landgren et~al.(2016{\natexlab{a}})Landgren, Srivastava, and Leonard]{landgren2016distributed}
P.~Landgren, V.~Srivastava, and N.~E. Leonard.
\newblock On distributed cooperative decision-making in multiarmed bandits.
\newblock In \emph{2016 European Control Conference}. 243--248. IEEE, 2016{\natexlab{a}}.

\bibitem[Landgren et~al.(2016{\natexlab{b}})Landgren, Srivastava, and Leonard]{landgren2016distributed_2}
P.~Landgren, V.~Srivastava, and N.~E. Leonard.
\newblock Distributed cooperative decision-making in multiarmed bandits: Frequentist and {B}ayesian algorithms.
\newblock In \emph{2016 IEEE 55th Conference on Decision and Control}. 167--172. IEEE, 2016{\natexlab{b}}.

\bibitem[Landgren et~al.(2021)Landgren, Srivastava, and Leonard]{landgren2021distributed}
P.~Landgren, V.~Srivastava, and N.~E. Leonard.
\newblock Distributed cooperative decision making in multi-agent multi-armed bandits.
\newblock \emph{Automatica}, 125:\penalty0 109445, 2021.

\bibitem[Li et~al.(2023)Li, Zhao, Yu, Wu, and Li]{li2023clustering}
Q.~Li, C.~Zhao, T.~Yu, J.~Wu, and S.~Li.
\newblock Clustering of conversational bandits with posterior sampling for user preference learning and elicitation.
\newblock \emph{User Modeling and User-Adapted Interaction}, 33\penalty0 (5):\penalty0 1065--1112, 2023.

\bibitem[Li and Zhang(2018)]{li2018online}
S.~Li and S.~Zhang.
\newblock Online clustering of contextual cascading bandits.
\newblock In \emph{Proceedings of the AAAI Conference on Artificial Intelligence}, volume~32, 2018.

\bibitem[Li et~al.(2016{\natexlab{a}})Li, Gentile, Karatzoglou, and Zappella]{li2016online}
S.~Li, C.~Gentile, A.~Karatzoglou, and G.~Zappella.
\newblock Online context-dependent clustering in recommendations based on exploration-exploitation algorithms.
\newblock \emph{ArXiv, abs/1608.03544}, 2016{\natexlab{a}}.

\bibitem[Li et~al.(2016{\natexlab{b}})Li, Karatzoglou, and Gentile]{li2016collaborative}
S.~Li, A.~Karatzoglou, and C.~Gentile.
\newblock Collaborative filtering bandits.
\newblock In \emph{Proceedings of the 39th International ACM SIGIR conference on Research and Development in Information Retrieval}, pages 539--548, 2016{\natexlab{b}}.

\bibitem[Li et~al.(2019)Li, Chen, and Leung]{li2019improved}
S.~Li, W.~Chen, and K.-S. Leung.
\newblock Improved algorithm on online clustering of bandits.
\newblock \emph{arXiv preprint arXiv:1902.09162}, 2019.

\bibitem[Li and Song(2022)]{li2022privacy}
T.~Li and L.~Song.
\newblock Privacy-preserving communication-efficient federated multi-armed bandits.
\newblock \emph{IEEE Journal on Selected Areas in Communications}, 40\penalty0 (3):\penalty0 773--787, 2022.

\bibitem[Li et~al.(2025)Li, Liu, Dai, and Lui]{li2025demystifying}
Z.~Li, M.~Liu, X.~Dai, and J.~Lui.
\newblock Demystifying online clustering of bandits: Enhanced exploration under stochastic and smoothed adversarial contexts.
\newblock \emph{arXiv preprint arXiv:2501.00891}, 2025.

\bibitem[Lima et~al.(2008)Lima, Sousa, and Sumuor]{lima2008majority}
F.~W. Lima, A.~O. Sousa, and M.~Sumuor.
\newblock Majority-vote on directed {E}rd{\H{o}}s--{R}{\'e}nyi random graphs.
\newblock \emph{Physica A: Statistical Mechanics and its Applications}, 387\penalty0 (14):\penalty0 3503--3510, 2008.

\bibitem[Liu et~al.(2022)Liu, Zhao, Yu, Li, and Lui]{liu2022federated}
X.~Liu, H.~Zhao, T.~Yu, S.~Li, and J.~C. Lui.
\newblock Federated online clustering of bandits.
\newblock In \emph{Uncertainty in Artificial Intelligence}, pages 1221--1231. PMLR, 2022.

\bibitem[Mart{\'\i}nez-Rubio et~al.(2019)Mart{\'\i}nez-Rubio, Kanade, and Rebeschini]{martinez2019decentralized}
D.~Mart{\'\i}nez-Rubio, V.~Kanade, and P.~Rebeschini.
\newblock Decentralized cooperative stochastic bandits.
\newblock \emph{Advances in Neural Information Processing Systems}, 32, 2019.

\bibitem[Mitra et~al.(2021)Mitra, Hassani, and Pappas]{mitra2021exploiting}
A.~Mitra, H.~Hassani, and G.~Pappas.
\newblock Exploiting heterogeneity in robust federated best-arm identification.
\newblock \emph{arXiv preprint arXiv:2109.05700}, 2021.

\bibitem[Nguyen and Lauw(2014)]{nguyen2014dynamic}
T.~T. Nguyen and H.~W. Lauw.
\newblock Dynamic clustering of contextual multi-armed bandits.
\newblock In \emph{Proceedings of the 23rd ACM international conference on conference on information and knowledge management}, pages 1959--1962, 2014.

\bibitem[Pal et~al.(2024)Pal, Suggala, Shanmugam, and Jain]{pal2024blocked}
S.~Pal, A.~Suggala, K.~Shanmugam, and P.~Jain.
\newblock Blocked collaborative bandits: online collaborative filtering with per-item budget constraints.
\newblock \emph{Advances in Neural Information Processing Systems}, 36, 2024.

\bibitem[R{\'e}da et~al.(2022)R{\'e}da, Vakili, and Kaufmann]{reda2022near}
C.~R{\'e}da, S.~Vakili, and E.~Kaufmann.
\newblock Near-optimal collaborative learning in bandits.
\newblock In \emph{2022-36th Conference on Neural Information Processing System}, 2022.

\bibitem[Roman et~al.(2013)Roman, Zhou, and Lopez]{roman2013features}
R.~Roman, J.~Zhou, and J.~Lopez.
\newblock On the features and challenges of security and privacy in distributed internet of things.
\newblock \emph{Computer networks}, 57\penalty0 (10):\penalty0 2266--2279, 2013.

\bibitem[Sankararaman et~al.(2019)Sankararaman, Ganesh, and Shakkottai]{sankararaman2019social}
A.~Sankararaman, A.~Ganesh, and S.~Shakkottai.
\newblock Social learning in multi agent multi armed bandits.
\newblock \emph{Proceedings of the ACM on Measurement and Analysis of Computing Systems}, 3\penalty0 (3):\penalty0 1--35, 2019.

\bibitem[Stanley et~al.(2019)Stanley, Bonacci, Kwitt, Niethammer, and Mucha]{stanley2019stochastic}
N.~Stanley, T.~Bonacci, R.~Kwitt, M.~Niethammer, and P.~J. Mucha.
\newblock Stochastic block models with multiple continuous attributes.
\newblock \emph{Applied Network Science}, 4:\penalty0 1--22, 2019.

\bibitem[Wang et~al.(2020{\natexlab{a}})Wang, Proutiere, Ariu, Jedra, and Russo]{wang2020optimal}
P.-A. Wang, A.~Proutiere, K.~Ariu, Y.~Jedra, and A.~Russo.
\newblock Optimal algorithms for multiplayer multi-armed bandits.
\newblock In \emph{International Conference on Artificial Intelligence and Statistics}, pages 4120--4129. PMLR, 2020{\natexlab{a}}.

\bibitem[Wang et~al.(2020{\natexlab{b}})Wang, Proutiere, Ariu, Jedra, and Russo]{wangp2020optimal}
P.-A. Wang, A.~Proutiere, K.~Ariu, Y.~Jedra, and A.~Russo.
\newblock Optimal algorithms for multiplayer multi-armed bandits.
\newblock In \emph{International Conference on Artificial Intelligence and Statistics}, pages 4120--4129. PMLR, 2020{\natexlab{b}}.

\bibitem[Wang et~al.(2019)Wang, Zeng, Zhou, Li, Iyengar, Shwartz, and Grabarnik]{8440090}
Q.~Wang, C.~Zeng, W.~Zhou, T.~Li, S.~S. Iyengar, L.~Shwartz, and G.~Y. Grabarnik.
\newblock Online interactive collaborative filtering using multi-armed bandit with dependent arms.
\newblock \emph{IEEE Transactions on Knowledge and Data Engineering}, 31\penalty0 (8):\penalty0 1569--1580, 2019.
\newblock \doi{10.1109/TKDE.2018.2866041}.

\bibitem[Wang et~al.(2022)Wang, Yang, Chen, Liu, Hajiesmaili, Towsley, and Lui]{wangx2022achieving}
X.~Wang, L.~Yang, Y.-Z.~J. Chen, X.~Liu, M.~Hajiesmaili, D.~Towsley, and J.~C. Lui.
\newblock Achieving near-optimal individual regret \& low communications in multi-agent bandits.
\newblock In \emph{The Eleventh International Conference on Learning Representations}, 2022.

\bibitem[Wang et~al.(2023)Wang, Yang, Chen, Liu, Hajiesmaili, Towsley, and Lui]{wang2023achieve}
X.~Wang, L.~Yang, Y.-Z.~J. Chen, X.~Liu, M.~Hajiesmaili, D.~Towsley, and J.~C. Lui.
\newblock Achieve near-optimal individual regret \&amp; low communications in multi-agent bandits.
\newblock In \emph{International Conference on Learning Representations}, 2023.

\bibitem[Wang et~al.(1531--1539, 2021)Wang, Zhang, Singh, Riek, and Chaudhuri]{wang2021multitask}
Z.~Wang, C.~Zhang, M.~K. Singh, L.~Riek, and K.~Chaudhuri.
\newblock Multitask bandit learning through heterogeneous feedback aggregation.
\newblock In \emph{International Conference on Artificial Intelligence and Statistics}, 1531--1539, 2021.

\bibitem[Wu et~al.(2021)Wu, Zhao, Yu, Li, and Li]{wu2021clustering}
J.~Wu, C.~Zhao, T.~Yu, J.~Li, and S.~Li.
\newblock Clustering of conversational bandits for user preference learning and elicitation.
\newblock In \emph{Proceedings of the 30th ACM International Conference on Information \& Knowledge Management}, pages 2129--2139, 2021.

\bibitem[Xu and Klabjan(2023{\natexlab{a}})]{xu2023}
M.~Xu and D.~Klabjan.
\newblock Regret lower bounds in multi-agent multi-armed bandit.
\newblock \emph{arXiv preprint arXiv:2308.08046}, 2023{\natexlab{a}}.

\bibitem[Xu and Klabjan(2023{\natexlab{b}})]{xu2023decentralized}
M.~Xu and D.~Klabjan.
\newblock Decentralized randomly distributed multi-agent multi-armed bandit with heterogeneous rewards.
\newblock \emph{Advances on Neural Information Processing Systems}, 2023{\natexlab{b}}.

\bibitem[Yan et~al.(2022)Yan, Xiao, Chen, and Tajer]{yan2022federated}
Z.~Yan, Q.~Xiao, T.~Chen, and A.~Tajer.
\newblock Federated multi-armed bandit via uncoordinated exploration.
\newblock In \emph{IEEE International Conference on Acoustics, Speech and Signal Processing}. 5248--5252. IEEE, 2022.

\bibitem[Yang et~al.(2024)Yang, Liu, Wang, Xie, Lui, Lian, and Chen]{yang2024federated}
H.~Yang, X.~Liu, Z.~Wang, H.~Xie, J.~C. Lui, D.~Lian, and E.~Chen.
\newblock Federated contextual cascading bandits with asynchronous communication and heterogeneous users.
\newblock In \emph{Proceedings of the AAAI Conference on Artificial Intelligence}, volume~38, pages 20596--20603, 2024.

\bibitem[Yang et~al.(2018)Yang, Zhang, Zhang, Yu, Zhang, and Shen]{yang2018content}
P.~Yang, N.~Zhang, S.~Zhang, L.~Yu, J.~Zhang, and X.~Shen.
\newblock Content popularity prediction towards location-aware mobile edge caching.
\newblock \emph{IEEE Transactions on Multimedia}, 21\penalty0 (4):\penalty0 915--929, 2018.

\bibitem[Zachary(1977)]{zachary1977information}
W.~W. Zachary.
\newblock An information flow model for conflict and fission in small groups.
\newblock \emph{Journal of anthropological research}, 33\penalty0 (4):\penalty0 452--473, 1977.

\bibitem[Zhao et~al.(2020)Zhao, Zhao, Jiang, Tan, Niyato, Li, Lyu, and Liu]{zhao2020privacy}
Y.~Zhao, J.~Zhao, L.~Jiang, R.~Tan, D.~Niyato, Z.~Li, L.~Lyu, and Y.~Liu.
\newblock Privacy-preserving blockchain-based federated learning for iot devices.
\newblock \emph{IEEE Internet of Things Journal}, 8\penalty0 (3):\penalty0 1817--1829, 2020.

\bibitem[Zhu and Liu(2023)]{zhu2023distributed}
J.~Zhu and J.~Liu.
\newblock Distributed multi-armed bandits.
\newblock \emph{IEEE Transactions on Automatic Control}, 2023.

\bibitem[Zhu et~al.(2020)Zhu, Sandhu, and Liu]{zhu2020distributed}
J.~Zhu, R.~Sandhu, and J.~Liu.
\newblock A distributed algorithm for sequential decision making in multi-armed bandit with homogeneous rewards.
\newblock In \emph{59th IEEE Conference on Decision and Control}. 3078--3083. IEEE, 2020.

\bibitem[Zhu et~al.(2021)Zhu, Mulle, Smith, and Liu]{zhu2021decentralized}
J.~Zhu, E.~Mulle, C.~S. Smith, and J.~Liu.
\newblock Decentralized multi-armed bandit can outperform classic upper confidence bound.
\newblock \emph{arXiv preprint arXiv:2111.10933}, 2021.

\bibitem[Zhu et~al.(3--4, 2021)Zhu, Zhu, Liu, and Liu]{zhu2021federated}
Z.~Zhu, J.~Zhu, J.~Liu, and Y.~Liu.
\newblock Federated bandit: A gossiping approach.
\newblock In \emph{Abstract Proceedings of the 2021 ACM SIGMETRICS/International Conference on Measurement and Modeling of Computer Systems}, 3--4, 2021.

\end{thebibliography}

\newpage

\appendix

\section{Pseudo Code of Algorithms}\label{apx:IR-LSS}

\subsection{Burn-in Peirod}
The full algorithm of the burn-in period described in Section \ref{sec:heter} is shown as follows.

\begin{algorithm2e}[H]
  \SetAlgoLined
  \caption{UCB-SBM: Burn-in period \citep{xu2023decentralized}}\label{alg:burn-in}
  Initialization: The length of the burn-in period is $L$;  the estimates are initialized as $\bar{\mu}_i^m(0) = 0$, $n_{m,i}(0) = 0$, $\hat{\bar{\mu}}_{i,j}^m(0) = 0 $, and $P_0(m,j) = $ for any arm $i$ and agents $m, j$\;
  \For{$1 < t \leq L$}{
  \For{each agent $m$}{
  Sample arm $a^m_t = (t \mod K)$\;
  Receive rewards $r_{a^m_t}^m(t)$ and update $n_{m,i}(t) = n_{m,i}(t-1) + \mathds{1}_{a_m^t = i}$\;
  Update the local estimates for any arm $i$: $\bar{\mu}_i^m(t) = \frac{n_{m,i}(t-1)\bar{\mu}_i^m(t-1) + r_{a^m_t}^m(t) \cdot 1_{a^m_t = i}}{n_{m,i}(t-1) + 1_{a^m_t = i}}$\;
  Update the maintained matrix $P_t(m,j) = \nicefrac{((t-1)P_{t-1}(m,j)+X_{m,j}^t)}{t}$ for each $j \in V$\;
  Send $\{\bar{\mu}_i^m(t)\}_{i =1}^{i=K}$ to all agents in $\mathcal{N}_m(t)$\;
  Receive $\{\bar{\mu}_i^j(t)\}_{i=1}^{i = K}$ from all agents $j \in \mathcal{N}_m(t)$ and store them as $\hat{\bar{\mu}}_{i,j}^m(t)$.
  }
  }
  \For{each agent $m$ and arm $i$}{For agent $1 \leq j \leq M$, let
  $t_{m,j} = \max_{s\geq \tau_1}\{(m,j) \in E_s\}$ or $0$ if such $s$ does not exist \\
  $\Tilde{\mu}_i^m(L+1) = \sum_{j=1}^MP^{\prime}_{m,j}(L) \hat{\bar{\mu}}_{i,j}^m(t_{m,j})$ where $P^{\prime}_{m,j}(L) = \frac{1}{M} \text{ if $P_L(m,j) > 0$ and } 0 \text{ o.w.}$ \;
  }
\end{algorithm2e}

\subsection{Clustering Algorithm}

The full algorithm of the cluster detection described in Section \ref{sec:heter-unknown} is presented below. 

\begin{algorithm2e}[H]
  \SetAlgoLined
  \caption{Iterative Refinement Clustering (IR-LSS)~\cite{braun2022iterative}}\label{alg:IR-LSS}
  \textbf{Input:}  adjacency matrix $A \in \{0,1\}^{M \times M}$, node covariates $V \in \mathbb{R}^{M \times K}$, variance $\sigma^2$, initial cluster assignment $Z^{(0)} \in \{0,1\}^{M \times C}$ and iterations $T > 1$;\par
  \textbf{Output:} a cluster assignment $Z \in \{0,1\}^{M \times C}$ \par
  \For{$t = 0 , 2, \ldots,T-1$}{
  Estimate the model parameters from the current cluster assignment $Z^{(t)}$: $s^{(t)}_n = |c^{(t)}_n|$, $W^{(t)} = Z^{(t)} (D^{(t)})^{-1}$ where $D^{(t)} = \mathrm{diag}(s^{(t)}_1,\dots,s^{(t)}_C)$, $\Pi^{(t)} = (W^{(t)})^{\top} A W^{(t)}$, and $\mu^{(t)}_n = W^{(t)}_n V$; \par
  Refine clustering by assigning each node $i$ to the cluster
  $$z^{(t+1)}_i = \arg\min_{n \in [C]} \|(A_i W^{(t)} + \Pi^{(t)}_n) \sqrt{\Sigma^{(t)}_n}\|_2^2 + \|\mu^{(t)}_n - V_i\|_2^2 /\sigma^2,$$
  where $\Sigma^{(t)}_n = \frac{M}{C(p^{(t)}-q^{(t)})}\log(\frac{p^{(t)}(1-q^{(t)})}{q^{(t)}(1-p^{(t)})})I_C$ with $p^{(t)} = \sum_{n \in [C]} \Pi^{(t)}_{nn}/C$ and $q^{(t)} = \sum_{n \neq m} \Pi^{(t)}_{mn}/(C(C-1))$;\par
  Form the cluster assignment matrix $Z^{(t+1)}$;
  }
\end{algorithm2e}  

\section{Experiment Details}\label{app:exp}

We present the experimental details in this section. We introduce the datasets, including both the synthetic data in Section \ref{app:exp-data} and the real-world dataset, in Section \ref{app:exp-res}. %Then, in Section \ref{app:exp-res}, we outline the computational resources used for the experiments, ensuring reproducibility and providing an illustration of the complexity involved.

\subsection{Synthetic Datasets}\label{app:exp-data}

We now describe the datasets used in the experiments and provide additional details that can be leveraged to reproduce the experiments discussed in Section \ref{sec:exp}.

\textbf{Synthetic Dataset.} We first examine the performance of the algorithms on a synthetic dataset. The data generation process is as follows: for the results shown in Fig. \ref{fig:syn}, we select the number of agents as $M = 10$, the number of clusters as $C = 10$, the inter-cluster probability as $p = 0.5$, and the intra-cluster probability as $q = 0.5$. The length of the game is $5 \cdot 10^5$. %The reward mean values are \hl{?}, and the reward distributions are \hl{?}. 

For Figs. \ref{fig:M}, \ref{fig:C}, \ref{fig:p}, and \ref{fig:q}, we vary $M$, $C$, $p$, and $q$, respectively, while keeping the other parameters fixed.

\subsection{Real-world Datasets}\label{app:exp-res}
\textbf{Real-world Dataset.} 
Besides the synthetic dataset, we evaluate our algorithm and the benchmark algorithms on a real-world dataset, as reported in Fig. \ref{fig:r}, using the well-known Zachary's Karate Club dataset, a widely used benchmark for graph clustering algorithms. This dataset represents the social interactions among 34 members of a university karate club, as observed and documented by Wayne W. Zachary in the 1970s~\cite{zachary1977information}. The dataset consists of 34 nodes (agents), each representing a club member, and 78 unweighted, undirected edges that denote friendships between members.

%\subsection{Computing Resources}\label{app:exp-res}

%Additionally, we outline the computing resources utilized to run the aforementioned experiments. We use a single laptop equipped with \hl{CPUs, GPUs, memory}. It approximately takes \hl{time} to complete 25 runs of the algorithms.

\section{Additional Theoretical Results}\label{app:Theory}

In this section, we present the corollaries of Theorem \ref{thm:2}, \ref{thm:3}, \ref{thm:4}, \ref{thm:5} when the cluster structure is unknown. The proof of them follow the proof of Corollary \ref{cor:6} as in the main body, and as a result, we omit the proof steps here. We use $L + L_1 = O(L+M)$ to denote the length of the burn-in period for the unknown cluster structure setting which contains $L$ steps for Algorithm~\ref{alg:burn-in} and $L_1 = O(M)$ steps to propagate information to the entire graph for cluster recovery. We use $\tau = 1/\mathrm{poly}(M)$ to denote the failure probability of cluster recovery.

\begin{corollary}[Extension of Theorem \ref{thm:2}]
  Let us assume that $\min_{m,n}p(m,n) \geq (\frac{1}{2} + \frac{1}{2}\sqrt{1 - (\frac{\delta}{MT})^{\frac{2}{M-1}}})$. For every $0 < \epsilon < 1$ and $0 < \delta
    < \frac{1}{2} + \frac{1}{4}\sqrt{1 - (\frac{\epsilon}{MT})^{\frac{2}{M-1}}}$, the regret of Algorithm 2 with Rule 1 is upper bounded by with probability $1-7\epsilon - \tau$,
  \begin{align*}
    E[R_T | A_{\epsilon,\delta, \tau}^{\prime}] & \leq L + \textcolor{black}{L_1} + \sum_{i \neq i^*}\Delta_i(\max{\{[\frac{4C_1\log T}{\Delta_i^2}], 2(K^2+MK) \}} +  \frac{2\pi^2}{3P(A_{\epsilon,\delta, \tau}^{\prime})} + K^2 + (2M-1)K)
  \end{align*}
\end{corollary}

\begin{corollary}[Extension of Theorem \ref{thm:3}]
  Let us assume that $p(m,m) = 1$ for any $1 \leq m \leq C$. Let us further assume that $\min_{m,n}p(m,n) \geq \min_{m,n}p(m,n) \geq 1 - (\frac{1}{2} - \frac{1}{2}\sqrt{1 - (\frac{\delta}{CT})^{\frac{2}{C-1}}})^{\nicefrac{C^2}{M^2}}$. The regret bound of Algorithm 2 with Rule 2 reads as with probability $1-7\epsilon - \textcolor{black}{\tau}$
  \begin{align*}
     & E[R_T|A_{\epsilon,\delta, \tau}^{\prime}] & \\
     & \leq L + \textcolor{black}{L_1} + \sum_{i \neq i^*}\Delta_i(\max{\{\frac{C}{M} \cdot [\frac{4C_1\log T}{\Delta_i^2}], 2(K^2+MK) \}} +  \frac{2\pi^2}{3P(A_{\epsilon,\delta, \tau}^{\prime}} + K^2 + (2M-1)K) + l
  \end{align*}
\end{corollary}

\begin{corollary}[Extension of Theorem \ref{thm:4}]
  Let us assume that $p(m,m) = 1$ for any $1 \leq m \leq C$. Let us further assume that $\min_{m,n}p(m,n) \geq  1-(1- \min\{(\frac{1}{2} + \frac{1}{2}\sqrt{1 - (\frac{\delta}{CT})^{\frac{2}{C-1}}}), 1 - \nicefrac{\delta (C-1)}{8CT}\})^{\nicefrac{C^2}{M^2}}$. The regret bound of Algorithm 2 with Rule 2 reads as with probability $1-7\epsilon - \textcolor{black}{\tau}$
  \begin{align*}
     & E[R_T|A_{\epsilon,\delta, \tau}^{\prime}] & \\
     & \leq L + \textcolor{black}{L_1} + \sum_{i \neq i^*}\Delta_i(\max{\{\frac{C}{M} \cdot [\frac{4C_1\log T}{\Delta_i^2}], 2(K^2+MK) \}} +  \frac{2\pi^2}{3P(A_{\epsilon,\delta, \tau}^{\prime}} + K^2 + (2M-1)K) + l
  \end{align*}
\end{corollary}

\begin{corollary}[Extension of Theorem \ref{thm:5}]
Let us assume that $\min_{m}p(m,m) = 1$ for any $1 \leq m \leq C$, and that $\min_{m \neq n}p(m,n) \geq \frac{e}{e-1}\frac{C^2}{M^2}\max{\{\frac{(C-l-1)!}{(C-2)!}(1 - \frac{\delta (C-1)}{8CT}), \frac{(C-l-1)!}{(C-2)!}(\frac{3}{4})^{\frac{1}{l}}\}}$. The regret bound of Algorithm 2 with Rule 2 reads as with probability $1-7\epsilon - \textcolor{black}{\tau}$
  \begin{align*}
     & E[R_T|A_{\epsilon,\delta, \tau}^{\prime}] & \\
     & \leq L + \textcolor{black}{L_1} + \sum_{i \neq i^*}\Delta_i(\max{\{\frac{C}{M} \cdot [\frac{4C_1\log T}{\Delta_i^2}], 2(K^2+MK) \}} +  \frac{2\pi^2}{3P(A_{\epsilon,\delta, \tau}^{\prime}} + K^2 + (2M-1)K) + l
  \end{align*}
\end{corollary}

\section{Proof of Lemmas on Stochastic Block Models}\label{app:proof-lemma}

\begin{lemma*}[Lemma \ref{lem:edge_prob_1}]
    For any pair of vertices $c_m, c_n$ in the sub-graph $G^C_t$, the probability that $c_m$ and $c_n$ is connected in $G^C_t$ is $p(c_m,c_n) = 1-(1-p(m,n))^{M^2/C^2}$. 
\end{lemma*}

\begin{proof}[Proof of Lemma \ref{lem:edge_prob_1}]
    Since all clusters have the same size, the clusters $c_m$ and $c_n$ have $\frac{M}{C}$ agents each. Thus, there are $\frac{M^2}{C^2}$ pairs of vertices between clusters $c_m$ and $c_n$. Since each pair of such vertices is connected with probability $p(m,n)$ independently, the probability that there is at least one edge between clusters $c_m$ and $c_n$ is $1-(1-p(m,n))^{M^2/C^2}$.
    
\end{proof}

\begin{lemma*}[Lemma \ref{lem:edge_prob_2}]
    For any pair of vertices $c_m, c_n$ in the sub-graph $G^C_t$, the probability that $c_m$ and $c_n$ is connected in $G^C_t$ is $p(c_m,c_n) \geq (1-\frac{1}{e})\min\{1, \frac{M^2}{C^2} \cdot p(m,n)\}$. 
\end{lemma*}

\begin{proof}[Proof of Lemma \ref{lem:edge_prob_2}]
    Since all clusters have the same size, the clusters $c_m$ and $c_n$ have $\frac{M}{C}$ agents each. Thus, there are $\frac{M^2}{C^2}$ pairs of vertices between clusters $c_m$ and $c_n$. Since each pair of such vertices is connected with probability $p(m,n)$ independently, the probability that there is at least one edge between clusters $c_m$ and $c_n$ is $1-(1-p(m,n))^{M^2/C^2}$.
    We consider two cases: (1) $p(m,n) \geq \frac{C^2}{M^2}$; (2) $p(m,n) < \frac{C^2}{M^2}$.

    For $p(m,n) \geq \frac{C^2}{M^2}$, we have 
    $$
    1-(1-p(m,n))^{M^2/C^2} \geq 1- (1-\frac{C^2}{M^2})^{M^2/C^2} \geq 1- \frac{1}{e}.
    $$
    For $p(m,n) < \frac{C^2}{M^2}$, we consider $\frac{1-(1-p(m,n))^{M^2/C^2}}{p(m,n)M^2/C^2} $, which is decreasing with $p(m,n) \in (0,1]$. Then, we have
    $$
    \frac{1-(1-p(m,n))^{M^2/C^2}}{p(m,n)M^2/C^2} \geq 1- (1-\frac{C^2}{M^2})^{M^2/C^2} \geq 1- \frac{1}{e}.
    $$
\end{proof}

\begin{lemma*}[Braun et al.~\cite{braun2022iterative}; Lemma \ref{lemma:CSSBM}]
    Consider a CSSBM with $M$ nodes, $C$ clusters, and signal-to-noise ratio $\mathrm{SNR}$. Suppose $\mathrm{SNR} > 2\log M$ and $C^{3} \leq \mathrm{SNR} \cdot \delta$ for a small constant $\delta < 1/2$. Then, with probability at least $1 - 1/\mathrm{poly}(M)$, Algorithm~\ref{alg:IR-LSS} exactly recovers the community structure.
    %\label{lem:recoverCSSBM}
\end{lemma*}

Then, we use the graph information and the reward estimation for each agent from the burn-in period with $L = O(\log T)$ rounds as the input for the above clustering algorithm. 
We assume that the reward for each arm $i \in [K]$ of agent $m \in [M]$ is sampled from a Gaussian distribution $\mathcal{N}(\mu^m_i,\sigma^2)$.
Then, we can recover the cluster structure exactly under the following condition. 

\begin{lemma*}[Lemma \ref{lem:recover}]
    Consider the graph is generated from a stochastic block model with $M$ agents and $C$ balanced clusters with edge connection probability $p(m,m) = p$ for all $m \in [C]$ and $p(m,n) = q$ for all $m \neq n$, where $p = p'\frac{\log M}{M}$ and $q = q' \frac{\log M}{M}$.
    The reward for each arm $i \in [K]$ of agent $m \in [M]$ is sampled from a Gaussian distribution $\mathcal{N}(\mu^m_i, \sigma^2)$.
    Suppose $\mathrm{SNR} > 2 \log M$ and $C^3 \leq \mathrm{SNR} \cdot  \delta$ for a small constant $\delta < 1/2$, where
    $$\mathrm{SNR} = \frac{\log T}{8K\sigma^2} \min_{m\neq n}\|\mu^m - \mu^n\|_2^2 + \frac{\log M \log T}{C}(\sqrt{p'} - \sqrt{q'})^2.$$
    Then, with probability at least $1 - 1/\mathrm{poly}(M)$, Algorithm~\ref{alg:IR-LSS} exactly recovers the community structure.
\end{lemma*}

\begin{proof}[Proof of Lemma~\ref{lem:recover}]
    First, we consider the information from the burn-in period and show that this information can be formed as an instance from the CSSBM.
    The length of the burn-in period is $L = O(\log T)$. Consider the graph $G_L$ on all $M$ agents as follows. For any pair of agents $i,j \in [M]$, there is an edge connecting $i,j$ in $G_L$ if and only if agents $i$ and $j$ are connected for at least once in the burn-in period. 
    Thus, for any two agents $i,j$ in the same cluster, they are connected in $G_L$ with probability $p_L = 1-(1-p)^L$. For any two agents $m,n$ in different clusters, they are connected in $G_L$ with probability $q_L = 1-(1-q)^L$. We consider the regime where $p = p'\log M/M$ and $q = q' \log M/M$ for constants $p'$ and $q'$. Since edge connection probabilities $p$ and $q$ are small, we have $p_L \approx Lp$ and $q_L \approx Lq$. This graph $G_L$ can be seen as generated from the stochastic block model with $M$ agents and $C$ balanced clusters and edge probability $p(m,m) = p_L$ for $m \in [C]$ and $p(m,n) = q_L$ for $m\neq n$.
    For each agent $m \in [M]$, we use the reward local estimates $\bar{\mu}^m(L)$ from the burn-in period as the node covariates. In the burn-in period, each arm is pulled $L/K$ times for each agent. Since we assume the reward is sampled from a Gaussian distribution with variance $\sigma^2$, the reward local estimates $\bar{\mu}^m(L)$ is a random variable from the Gaussian distribution $\mathcal{N}(\mu^m,\sigma^2K/L)$.

    Now, we show that this instance from the CSSBM can be exactly recovered under the given conditions. 
    Note that the Signal-to-noise ratio of the above CSSBM is 
    $$
    \mathrm{SNR} = \frac{\log T}{8K\sigma^2} \min_{m\neq n}\|\mu^m - \mu^n\|_2^2 + \frac{\log M \log T}{C}(\sqrt{p'} - \sqrt{q'})^2.
    $$
    By Lemma~\ref{lemma:CSSBM}, when $\mathrm{SNR} > 2 \log M$ and $C^3 \leq \mathrm{SNR} \cdot  \delta$ for a small constant $\delta < 1/2$, Algorithm~\ref{alg:IR-LSS} exactly recover the cluster structure with probability at least $1-1/\mathrm{poly}(M)$.
    
\end{proof}

\section{Proof of Theorems}\label{app:proof}

\subsection{Proof of Theorem \ref{thm:homo-ucb}}

\begin{proof}[Proof]
    \textbf{Main intuition:} Fix a suboptimal arm \(k\). After the total number of observations for this arm \(k\) exceeds the sample complexity threshold, in expectation, it takes \(\frac{1}{p^{M^2}}\) time slots for all agents to get the information of this arm \(k\).
    After that, no more regret will be incurred on this arm \(k\).

    Below, we present the proof of the algorithm. The proof first makes an assumption to reduce the problem to a standard cooperative UCB for homogeneous agent residing on a complete graph with communication delays.
    Then, we show that this assumption can be fulfilled in the single cluster scenario.

    \textbf{Step 1:}
    Consider a homogeneous multi-agent multi-arm bandit model, where in each time slot, with probability \(q\in(0,1)\), a global synchronization of observations would happen among all agents.
    Applying~\citet[Lemmas 1 and 2]{wang2023achieve} for cooperative UCB yields the upper bound of pulling times for each suboptimal arm as follows,
    \(
    \tilde N_{k,t}\upbra{m} \le \frac{8\log T}{\Delta_k^2},
    \)
    which is a standard property for UCB algorithm.
    With this stochastic global synchronization, we know that in expectation, the global total pulling times of each suboptimal arm \(k\) is at most \(
    \frac{8\log T}{\Delta_k^2} + \frac{M}{q}.
    \)
    Therefore, the regret for multi-agent multi-armed bandits with the synchronization probability \(q\) is upper bounded as follows, \begin{align*}
        \mathbb{E}[R_T] \le O(
        \sum_{k\neq k^*} \frac{\log T}{\Delta_k}
        + \frac{KM}{q}
        ).
    \end{align*}

    % Step 1.1: define type I and type II events, to show that with high probability, the pulling times of arm \(k\) is at most \(\frac{8\log T}{\Delta_k^2}\).

    % Step 1.2: show that the multi-agent system, due to delay communication for lack of global synchronization, incurs at most \(\frac{M}{q}\) additional pulls on the arm \(k\) in expectation.

    \textbf{Step 2:}
    Under the single cluster scenario, the communication graph \(G_t\) is generated by the stochastic block model (SBM) with the edge probability \(p\). With a probability \(p^{M^2}\), all agents in the cluster are connected in the communication graph \(G_t\), which fulfills the stochastic synchronization with probability \(q = p^{M^2}\), which concludes the proof.
    
\end{proof}

\subsection{Proof of Theorem \ref{thm:2}}

\begin{proof}

By specifying $c = \min_{m,n}p(m,n)$ in the Erdos-Renyi model in \citep{xu2023decentralized} and by having $c > \min_{m,n}p(m,n) \geq (\frac{1}{2} + \frac{1}{2}\sqrt{1 - (\frac{\delta}{MT})^{\frac{2}{M-1}}})$ we have that some key results in \citep{xu2023decentralized} hold, which are listed as follows.  

First we characterize the graph topology related to the random graph (a reduced form of Erdos-Renyi (E-R) model), which utilizes the above assumption on $p(m,n)$.  

\textbf{Graph connectivity}

\begin{Proposition}~\label{prop:connectivity_setting_1}
Assume $c$ in setting 1 meets the condition 
\begin{align*}
    1 \geq c \geq \frac{1}{2} + \frac{1}{2}\sqrt{1 - (\frac{\epsilon}{MT})^{\frac{2}{M-1}}}, 
\end{align*}
where $0 < \epsilon < 1$. Then, with probability $1 - \epsilon$, for any $t > 0$,  the graph $G_t$ following the E-R model is connected. 
\end{Proposition}

Next, we present the result regarding the transmission gap, which guarantees that agents can effectively collect one another's reward information within certain time frame with high probability. 

\textbf{Explicit transmission gap}

\begin{Proposition}\label{prop:t_0_L}
    %By the proper choice of $L$ in setting 1 and setting 2 and letting $t_0 = \frac{\ln{\frac{MT}{\eta}}}{\ln{2}}$ in setting 1 and setting 2 with $M \leq 10$ and $t_{0} \geq \frac{\ln(\frac{\eta}{M^2T})}{\ln(1-c)}$ in setting 2 with $M > 10$, 
    We have that with probability $1- \epsilon$,  for any $t > L$ and any $m$, there exists $$t_{0} \geq \frac{\ln(\frac{\epsilon}{M^2T})}{\min_{P(i,j)}\ln(1-P(i,j))}$$ such that
    \begin{align*}
         & t+1 - \min_jt_{m,j} \leq t_0, t_0 \leq c_0\min_{l}n_{l,i}(t+1)
    \end{align*}
    where $c_0$ $=$ $c_0(K, \min_{i \neq i^*}\Delta_i, M, \epsilon, \delta)$.
\end{Proposition}

By the construction of the estimators based on Rule 1, we derive that the global estimator $\Tilde{\mu}_i^m(t)$ is an unbiased estimator of the underlying true global reward mean value.

\textbf{Unbiasedness of the estimator}

\begin{Proposition}\label{prop:unbiased_1}
    Assume the parameter $\delta$ satisfies that $0 < \delta
        < c = f(\epsilon,M,T)$. For any arm $i$ and any agent $m$, at every time step $t$, we have
    \begin{align*}
        E[\Tilde{\mu}_i^m(t) | A_{\epsilon, \delta}] = \mu_i.
    \end{align*}
\end{Proposition}

Again based on Rule 2, we prove that the variance of the global estimator $\Tilde{\mu}^m_i(t)$ decays with $n_{m,i}(t)$ through the characterization of the moment generating function of $\Tilde{\mu}^m_i(t)$. 

\textbf{Variance term}

\begin{Proposition}\label{prop:var_pro_err}
    Assume the parameter $\delta$ satisfies that $0 < \delta
        < c = f(\epsilon,M,T)$. In setting $s_1, s_2, s_3$ where rewards follow sub-gaussian distributions, for any $m,i, \lambda$ and $t > L$ where $L$ is the length of the burn-in period, the global estimator $\Tilde{\mu}^m_i(t)$ is sub-Gaussian distributed. Moreover, the conditional moment generating function satisfies that with $P(A_{\epsilon, \delta}) = 1 - 7\epsilon$,
    \begin{align*}
         & E[\exp{\{\lambda(\Tilde{\mu}^m_i(t)  - \mu_i})\}1_{A_{\epsilon, \delta}} | \sigma(\{n_{m,i}(t)\}_{t,i,m})] \\
         & \leq \exp{\{\frac{\lambda^2}{2}\frac{C\sigma^2}{\min_{j}n_{j,i}(t)}\}}
    \end{align*}
    where $\sigma^2 = \max_{j,i}(\Tilde{\sigma}_i^j)^2$ and $C = \max\{\frac{4(M+2)(1 - \frac{1 - c_0}{2(M+2)})^2}{3M(1-c_0)}, (M+2)(1 + 4Md^2_{m,t})/M\}$.
\end{Proposition}

The unbiasedness and decaying variance of the global estimator $\Tilde{\mu}^m_i(t)$ allows us to show how much difference is there between $\Tilde{\mu}^m_i(t)$ and the unknown groundtruth $\mu_i$ through the following concentration inequality. 

\textbf{Concentration inequality}
\\

\begin{Proposition}\label{prop:concen_ine}
    Assume the parameter $\delta$ satisfies that $0 < \delta
        < c = f(\epsilon,M,T)$. For any $m,i$ and $t > L$ where $L$ is the length of the burn-in period, $\Tilde{\mu}_{m,i}(t)$ satisfies that if if $n_{m,i}(t) \geq 2(K^2+KM+M)$, then with $P(A_{\epsilon, \delta}) = 1 -7\epsilon$,
    \begin{align*}
         & P(\Tilde{\mu}_{m,i}(t) - \mu_i  \geq \sqrt{\frac{C_1\log t}{n_{m,i}(t)}} |A_{\epsilon, \delta}) \leq \frac{1}{P(A_{\epsilon, \delta})}\frac{1}{t^2}, \\
         & P(\mu_i - \Tilde{\mu}_{m,i}(t) \geq \sqrt{\frac{C_1\log t}{n_{m,i}(t)}} | A_{\epsilon, \delta}) \leq \frac{1}{P(A_{\epsilon, \delta})t^2}.
    \end{align*}
\end{Proposition}

Essentially, the above proposition implies that with high probability, we can identify the globally optimal arm by comparing all arms' global estimators $\Tilde{\mu}^m_i(t)$. Subsequently, we next show that the number of pulling these globally sub-optimal arms can be upper bounded by the $\log{T}$ based on the concentration inequality. 

\textbf{Number of pulls of sub-optimal arms}
\\

Upper bounds on $E[n_{m,k}(T) | A_{\epsilon, \delta}]$
\\

\begin{Proposition}\label{prop:n}
    Assume the parameter $\delta$ satisfies that $0 < \delta
        < c = f(\epsilon,M,T)$. An arm $k$ is said to be sub-optimal if $k \neq i^*$ where $i^*$ is the unique optimal arm in terms of the global reward, i.e. $i^* = \arg\max \frac{1}{M}\sum_{j=1}^M\mu_i^j$. Then when the game ends, for every agent $m$, $0 < \epsilon  < 1$ and $T > L$, the expected numbers of pulling sub-optimal arm $k$ after the burn-in period satisfies with $P(A_{\epsilon, \delta}) = 1- 7\epsilon$
    \begin{align*}
         & E[n_{m,k}(T) | A_{\epsilon, \delta}]                                                                                    \\
         & \leq \max{\{[\frac{4C_1\log T}{\Delta_i^2}], 2(K^2+MK+M) \}} +  \frac{2\pi^2}{3P(A_{\epsilon, \delta})} + K^2 + (2M-1)K \\
         & \leq O(\log{T}).
    \end{align*}
\end{Proposition}

The proof of Proposition 1 - 6 is presented in Appendix in \citep{xu2023decentralized} with Erdos-Renyi Models. As a result, we do not repeat the proof details of them herein and refer the proof steps therein. 

With these key results, we proceed to bound the regret. 

\textbf{Regret decomposition}

The optimal arm is denoted as $i^*$ satisfying
\begin{align*}
    i^* = \arg\max_i\sum_{m=1}^M\mu^{m}_i.
\end{align*}

For the proposed regret, we have that for any constant $L$, 
\begin{align*}
    R_T & =   \frac{1}{M}(\max_i\sum_{t=1}^T\sum_{m=1}^M\mu^{m}_i - \sum_{t=1}^T\sum_{m=1}^M\mu^{m}_{a_t^m}) \\
    & = \sum_{t=1}^T\frac{1}{M}\sum_{m=1}^M\mu^{m}_{i^*} - \sum_{t=1}^T\frac{1}{M}\sum_{m=1}^M\mu^{m}_{a_t^m} \\
    & \leq \sum_{t = 1}^{L}|\frac{1}{M}\sum_{m=1}^M\mu^{m}_{i^*} - \frac{1}{M}\sum_{m=1}^M\mu^{m}_{a_t^m}|+ \sum_{t = L + 1}^T(\frac{1}{M}\sum_{m=1}^M\mu^{m}_{i^*} - \frac{1}{M}\sum_{m=1}^M\mu^{m}_{a_t^m}) \\
    & \leq L + \sum_{t = L + 1}^T(\frac{1}{M}\sum_{m=1}^M\mu^{m}_{i^*} - \frac{1}{M}\sum_{m=1}^M\mu^{m}_{a_t^m}) \\
    & = L + \sum_{t = L + 1}^T(\mu_{i^*} - \frac{1}{M}\sum_{m=1}^M\mu^{m}_{a_t^m}) \\
    & = L+ ((T - L) \cdot \mu_{i^*} - \frac{1}{M}\sum_{m=1}^M\sum_{i = 1}^Kn_{m,i}(T)\mu^m_i)
\end{align*}
where the first inequality is by taking the absolute value and the second inequality results from the assumption that $0 < \mu_{i}^j < 1$ for any arm $i$ and agent $j$.

Note that $\sum_{i=1}^K\sum_{m=1}^Mn_{m,i}(T) = M(T-L)$ where 
by definition $n_{m,i}(T)$ is the number of pulls of arm $i$ at agent $m$ from time step $L+1$ to time step $T$, which yields that
\begin{align*}
    R_T
    & \leq L + \sum_{i=1}^K\frac{1}{M}\sum_{m=1}^Mn_{m,i}(T)\mu_{i^*}^m - \sum_{i=1}^K\frac{1}{M}\sum_{m=1}^Mn_{m,i}(T)\mu_i^m \\
    & = L +  \sum_{i=1}^K\frac{1}{M}\sum_{m=1}^Mn_{m,i}(T)(\mu_{i^*}^m -\mu_i^m) \\
    & \leq L +  \frac{1}{M}\sum_{i=1}^K\sum_{m:\mu_{i^*}^m - \mu_i^m > 0}n_{m,i}(T)(\mu_{i^*}^m - \mu_i^m) \\
    & = L +  \frac{1}{M}\sum_{i \neq i*}\sum_{m:\mu_{i^*}^m - \mu_i^m > 0}n_{m,i}(T)(\mu_{i^*}^m - \mu_i^m) .
\end{align*}
where the second inequality uses the fact that $\sum_{m:\mu_{i^*}^m - \mu_i^m \leq 0}n_{m,i}(T)(\mu_{i^*}^m - \mu_i^m) \leq 0$ holds for any arm $i$ and the last equality is true since $n_{m,i}(T)(\mu_{i^*}^m - \mu_i^m) = 0$ for $i = i^*$ and any $m$.

By using Proposition 5 and the fact that 
\begin{align*}
    R_T \leq L +  \frac{1}{M}\sum_{i \neq i*}\sum_{m:\mu_{i^*}^m - \mu_i^m > 0}n_{m,i}(T)(\mu_{i^*}^m - \mu_i^m),  
\end{align*}
we derive that 
\begin{align*}
    E[R_T|A_{\epsilon, \delta}] & \leq L +  \frac{1}{M}\sum_{i \neq i*}\sum_{m:\mu_{i^*}^m - \mu_i^m > 0}E[n_{m,i}(T)](\mu_{i^*}^m - \mu_i^m) \\
    & \leq L + \sum_{i \neq i^*}\Delta_i(\max{\{[\frac{4C_1\log T}{\Delta_i^2}], 2(K^2+MK) \}} +  \frac{2\pi^2}{3P(A_{\epsilon, \delta})} + K^2 + (2M-1)K)
\end{align*}
which completes the proof.

\end{proof}

\subsection{Proof of Theorem \ref{thm:3}}

\begin{proof}
It is worth noting that our framework implies that only the agents in different clusters have different reward distributions, and the agents in the same cluster stay on the same page by the assumption that $p(m,m)=1$. This indicates that as long as there is an edge between one agent in one cluster and another agent in another agents , the two cluster can exchange the heterogeneous reward distributions. Consequently, it is sufficient to pay attention to the sub-graph with respect to the clusters.

To this end, we establish the following proposition regarding the connectivity of the sub-graph. It is worth noting that the edge probability of this sub-graph, is now $c = 1-(1-p(m,n))^{M^2/C^2}$ by Lemma \ref{lem:edge_prob_1}, and the total number of vertex is $C$ instead of $M$.

\textbf{Graph connectivity}

\begin{Proposition}~\label{prop:connectivity_setting_2}
Assume $c$ meets the condition 
\begin{align*}
    1 \geq c \geq \frac{1}{2} + \frac{1}{2}\sqrt{1 - (\frac{\epsilon}{CT})^{\frac{2}{M-1}}}, 
\end{align*}
where $0 < \epsilon < 1$. Then, with probability $1 - \epsilon$, for any $t > 0$,  the sub-graph $G_t^C$ following the E-R model is connected. 
\end{Proposition}

\begin{proof}[Proof of Proposition 7]
    The proof of Proposition 7 follows from that of Proposition 1 in Theorem \ref{thm:2}, based on \citep{xu2023decentralized}. 
    
\end{proof}

Then based on the newly proposed estimator construction as in Rule 2, we have the following results. First, we characterize the consensus regarding arm pulls among the clusters, instead of the agents , since now the communication is on a cluster level. 

\textbf{Information delay}

\begin{lemma}
    For any $m,i,t > L$, if $N_{m,i}(t) \geq 2(K^2+KM+M)$ and subgraph $G_t$ induced by the clusters is connected, then we have
    \begin{align*}
        \hat{N}_{m,i}(t) \leq 2\min_{j}\hat{N}_{j,i}(t).
    \end{align*}
    where the min is taken over all clusters, not just the neighbors.
\end{lemma}

\begin{proof}[Proof of Lemma 7]
    The proof of this lemma follows from Lemma 3 in \citep{zhu2023distributed}, with the exception that now the shared arm information is $N_{m,i}(t)$ instead of $n_{m,i}(t)$. 
    
\end{proof}

Then, we show that the transmission gap with respect to the sub-graph still holds, that clusters can effectively collect one another's reward information within certain time frame with high probability. It is worth noting that when $p(m,m) = 1$, the clusters in one cluster obtain such information as well.  

\textbf{Explicit transmission gap}

\begin{Proposition}
    %By the proper choice of $L$ in setting 1 and setting 2 and letting $t_0 = \frac{\ln{\frac{MT}{\eta}}}{\ln{2}}$ in setting 1 and setting 2 with $M \leq 10$ and $t_{0} \geq \frac{\ln(\frac{\eta}{M^2T})}{\ln(1-c)}$ in setting 2 with $M > 10$, 
    We have that with probability $1- \epsilon$,  for any $t > L$ and any $m$, there exists $$t_{0} \geq \frac{\ln(\frac{\epsilon}{M^2T})}{\min_{P(i,j)}\ln(1-P(i,j))}$$ such that
    \begin{align*}
         & t+1 - \min_jt_{m,j} \leq t_0, t_0 \leq c_0\min_{l}N_{l,i}(t+1)
    \end{align*}
    where $c_0$ $=$ $c_0(K, \min_{i \neq i^*}\Delta_i, M, \epsilon, \delta)$.
\end{Proposition}

\begin{proof}[Proof of Proposition 8]
    The proof of this proposition follows from \citep{xu2023decentralized}, except that the shared information is  $N_{l,i}(t+1)$ instead of $n_{l,i}(t+1)$ since we consider the sub-graph with respect to the clusters (and there is no delay within the cluster). 
    
\end{proof}

It is straightforward to verify that by the construction of the global estimators based on Rule 2, we have that the global estimator $\Tilde{\mu}_i^m(t)$ is an unbiased estimator of $\mu_i$.

\textbf{Unbiasedness of the estimator}

\begin{Proposition}\label{prop:unbiased_2}
    Assume the parameter $\delta$ satisfies that $0 < \delta
        < c = f(\epsilon,M,T)$. For any arm $i$ and any agent $m$, at every time step $t$, we have
    \begin{align*}
        E[\Tilde{\mu}_i^m(t) | A_{\epsilon, \delta}] = \mu_i.
    \end{align*}
\end{Proposition}

\begin{proof}[Proof of Proposition 9]
    The proof of this proposition follows from \citep{xu2023decentralized}. 
    
\end{proof}

To proceed, we examine the variance of the global estimator $\Tilde{\mu}^m_i(t)$ constructed by Rule through the moment generating function.  It is worth noting that the variance is smaller compared to the one based on Rule 1, since we consider cluster-wise information in the decision making and communicate on a cluster-level. More specifically, the upper bound on the moment generating function changes from $\exp{\{\frac{\lambda^2}{2}\frac{C\sigma^2}{\min_{j}n_{j,i}(t)}\}}$ to $\exp{\{\frac{\lambda^2}{2}\frac{C\sigma^2}{\min_{j}N_{j,i}(t)}\}}$, where $N_{j,i} = \sum_{m \in c_j}n_{j,i}$ and thus achieves sample complexity reduction.

\textbf{Variance term}

\begin{Proposition}
    Assume the parameter $\delta$ satisfies that $0 < \delta
        < c = f(\epsilon,M,T)$. In setting $s_1, s_2, s_3$ where rewards follow sub-gaussian distributions, for any $m,i, \lambda$ and $t > L$ where $L$ is the length of the burn-in period, the global estimator $\Tilde{\mu}^m_i(t)$ is sub-Gaussian distributed. Moreover, the conditional moment generating function satisfies that with $P(A_{\epsilon, \delta}) = 1 - 7\epsilon$,
    \begin{align*}
         & E[\exp{\{\lambda(\Tilde{\mu}^m_i(t)  - \mu_i})\}1_{A_{\epsilon, \delta}} | \sigma(\{n_{m,i}(t)\}_{t,i,m})] \\
         & \leq \exp{\{\frac{\lambda^2}{2}\frac{C\sigma^2}{\min_{j}N_{j,i}(t)}\}}
    \end{align*}
    where $\sigma^2 = \max_{j,i}(\Tilde{\sigma}_i^j)^2$ and $C = \max\{\frac{4(M+2)(1 - \frac{1 - c_0}{2(M+2)})^2}{3M(1-c_0)}, (M+2)(1 + 4Md^2_{m,t})/M\}$.
\end{Proposition}

\begin{proof}[Proof of Proposition 10]
    The proof is done by induction as in \citep{xu2023decentralized}.
    
Based on our construction of $A_{\epsilon, \delta}^{\prime}$ and the choice of $\delta$, we have that for $t \geq L$, $|P_t - cE| < \delta < c$ on event $A_{\epsilon, \delta}$. It implies that for any $t \geq  L$, $m$ and $j$, $P_t(m,j) > 0$, and  if $t = L$
\begin{align}\label{eq:5.3_copy}
    P^{\prime}_t(m,j) = \frac{1}{M}
\end{align} 
and if $t > L$
\begin{align}\label{eq:5.3_copy_2}
    P^{\prime}_t(m,j) = \frac{M-1}{M^2}.
\end{align}

Consider the time step $t \leq L+1$. The quantity satisfies that (following from \citep{xu2023decentralized})
\begin{align}\label{eq:1_new}
     & E[\exp{\{\lambda(\Tilde{\mu}^m_i(t) - \mu_i)\}}1_{A_{\epsilon, \delta}}
    | \sigma(\{n_{m,i}(t)\}_{t,i,m}) \notag \\
     & = E[\exp{\{\lambda(\Tilde{\mu}^m_i(L+1)- \mu_i)\}}1_{A_{\epsilon, \delta}} | \sigma(\{n_{m,i}(t)\}_{t,i,m}) \notag \\
     & = E[\exp{\{\lambda(\sum_{j=1}^MP^{\prime}_{m,j}(L) \hat{\bar{\mu}}_{i,j}^m(t_{m,j}) - \mu_i) \}}1_{A_{\epsilon, \delta}} | \sigma(\{n_{m,i}(t)\}_{t,i,m})] \notag \\
     & = E[\exp{\{\lambda(\sum_{j=1}^M\frac{1}{M}\hat{\bar{\mu}}_{i,j}^m(t_{m,j}) - \mu_i)\}}1_{A_{\epsilon, \delta}} | \sigma(\{n_{m,i}(t)\}_{t,i,m})] \notag \\
     & = E[\exp{\{\lambda\sum_{j=1}^M\frac{1}{M}(\hat{\bar{\mu}}_{i,j}^m(t_{m,j})- \mu_i^j)\}}1_{A_{\epsilon, \delta}} | \sigma(\{n_{m,i}(t)\}_{t,i,m})] \notag \\
     & \leq \Pi_{j=1}^M(E[(\exp{\{(\lambda\frac{1}{M}(\hat{\mu}_i^j(t_{m,j}) - \mu_i^j)\}}1_{A_{\epsilon, \delta}})^M | \sigma(\{n_{m,i}(t)\}_{t,i,m}))])^{\frac{1}{M}}
\end{align}
where the third equality holds by (\ref{eq:5.3_copy}), the fourth equality uses the definition $\mu_i = \frac{1}{M}\sum_{i=1}^M\mu_i^j$, and the last inequality results from the generalized hoeffding inequality and the fact that $\hat{\bar{\mu}}_{i,j}^m(t_{m,j}) = \bar{\mu}_i^j(t_{m,j})$.

Note that for any agent $j$, we have
\begin{align}
    & E[(\exp{\{(\lambda\frac{1}{M}(\hat{\mu}_i^j(t_{m,j}) - \mu_i^j) \}}1_{A_{\epsilon, \delta}})^M | \sigma(\{n_{m,i}(t)\}_{t,i,m}))] \notag \\
    & = E[\exp{\{(\lambda(\hat{\mu}_i^j(t_{m,j})- \mu_i^j)\}}1_{A_{\epsilon, \delta}}) | \sigma(\{n_{m,i}(t)\}_{t,i,m})] \notag \\
    & = E[\exp{\{(\lambda\frac{\sum_s(r^j_i(s) - \mu_i^j)}{N_{j,i}(t_{m,j})}\}}1_{A_{\epsilon, \delta}}) | \sigma(\{n_{m,i}(t)\}_{t,i,m})] \notag \\
    & = E[\exp{\{\sum_s(\lambda\frac{(r^j_i(s) - \mu_i^j)}{N_{j,i}(t_{m,j})}\}}1_{A_{\epsilon, \delta}}) | \sigma(\{n_{m,i}(t)\}_{t,i,m})].  
\end{align}

We observe that based on the reward generation mechanism, given $s$, $r^j_i(s)$ does not dependend on anything else, which implies that 
\begin{align}
   & E[(\exp{\{(\lambda\frac{1}{M}(\hat{\mu}_i^j(t_{m,j}) - \mu_i^j) \}}1_{A_{\epsilon, \delta}})^M | \sigma(\{n_{m,i}(t)\}_{t,i,m}))] \notag \\ \textit{}
    & = \Pi_sE[\exp{\{\lambda\frac{(r^j_i(s) - \mu_i^j)}{N_{j,i}(t_{m,j})}\}}1_{A_{\epsilon, \delta}} | \sigma(\{n_{m,i}(t)\}_{t,i,m})] \notag \\
    & = \Pi_sE[\exp{\{\lambda\frac{(r^j_i(s) - \mu_i^j)}{N_{j,i}(t_{m,j})}\}}| \sigma(\{N_{m,i}(t)\}_{t,i,m})]\cdot E[1_{A_{\epsilon, \delta}} | \sigma(\{n_{m,i}(t)\}_{t,i,m})] \notag \\ 
    & = \Pi_sE_r[\exp{\{\lambda\frac{(r^j_i(s) - \mu_i^j)}{N_{j,i}(t_{m,j})}\}}] \cdot E[1_{A_{\epsilon, \delta}} | \sigma(\{N_{m,i}(t)\}_{t,i,m})] \notag \\
    & \leq \Pi_s\exp{\{\frac{(\frac{\lambda}{{N_{j,i}(t_{m,j})}})^2\sigma^2}{2}\}} \cdot E[1_{A_{\epsilon, \delta}} | \sigma(\{N_{m,i}(t)\}_{t,i,m})] \notag \\
    & \leq (\exp{\{\frac{(\frac{\lambda}{{N_{j,i}(t_{m,j})}})^2\sigma^2}{2}\}})^{N_{j,i}(t_{m,j})} \notag \\
    & = \exp{\{\frac{\frac{\lambda^2}{{N_{j,i}(t_{m,j})}}\sigma^2}{2}\}}  \notag \\
    & \leq  \exp{\{\frac{\lambda^2\sigma^2}{2\min_jN_{j,i}(t_{m,j})}\}}
\end{align}
where the first inequality holds by the definition of sub-Gaussian random variables $r^j_i(s) - \mu_i^j$ with an mean value $0$, the second inequality results from $1_{A_{\epsilon, \delta}} \leq 1$, and the last inequality uses $N_{j,i}(t_{m,j}) \geq \min_jN_{j,i}(t_{m,j})$ for any $j$. 

Therefore, we obtain that
\begin{align*}
    (\ref{eq:1_new}) & \leq \Pi_{j=1}^M( \exp{\{\frac{\lambda^2\sigma^2}{2\min_jN_{j,i}(t_{m,j})}\}})^\frac{1}{M} 
    \\
    & = (( \exp{\{\frac{\lambda^2\sigma^2}{2\min_jN_{j,i}(t_{m,j})}\}})^\frac{1}{M})^M \\
    & = \exp{\{\frac{\lambda^2\sigma^2}{2\min_jN_{j,i}(t_{m,j})}\}}
\end{align*}
which concludes the basis step.

Now we proceed to the induction step. Let us assume that for any $s < t+1$ where $t \geq L$, the following holds 
\begin{align}\label{eq:3}
    & E[\exp{\{\lambda(\Tilde{\mu}^m_i(s) - \mu_i)\}}1_{A_{\epsilon, \delta}} | \sigma(\{n_{m,i}(s)\}_{s,i,m})] \notag \\
    & \leq \exp{\{\frac{\lambda^2}{2}\frac{C\sigma^2}{\min_{j}N_{j,i}(s)}\}}. 
\end{align}

By the derivation of (25) in \citep{xu2023decentralized}, we derive that 
\begin{align}
    & E[\exp{\{\lambda(\Tilde{\mu}^m_i(t+1)-\mu_i)\}}1_{A_{\epsilon, \delta}} | \sigma(\{n_{m,i}(s)\}_{s,i,m})]  \notag \\
    & \leq \Pi_{j=1}^M(\exp{\{\frac{\lambda^2(P^{\prime}_t(m,j))^2(M+2)^2}{2}\frac{C\sigma^2}{\min_{j}N_{j,i}(t_{m,j})}\}})^{\frac{1}{M+2}} \cdot \notag \\
    & \qquad \qquad \Pi_{j \in N_m(t)}\Pi_s(E_r[\exp{\{\lambda 
 d_{m,t}(M+2) \frac{(r^j_i(s) -\mu_i^j)}{N_{j,i}(t)}\}}] \cdot E[1_{A_{\epsilon, \delta}} | \sigma(\{n_{m,i}(t)\}_{t,i,m})])^{\frac{1}{M+2}} \cdot \notag \\
    & \qquad \qquad \Pi_{j \not\in N_m(t)}\Pi_s(E_r[\exp{\{\lambda d_{m,t}(M+2)\frac{(r^j_i(s) -\mu_i^j)}{N_{j,i}(t_{m,j})}\}}] \cdot E[1_{A_{\epsilon, \delta}} | \sigma(\{n_{m,i}(t)\}_{t,i,m})])^{\frac{1}{M+2}} 
\end{align}

Meanwhile, we derive the following bound for the last two terms by the sub-Gaussian property of $(r^j_i(s) -\mu_i^j)$
\begin{align*}
       & E[\exp{\{\lambda(\Tilde{\mu}^m_i(t+1)-\mu_i)\}}1_{A_{\epsilon, \delta}} | \sigma(\{n_{m,i}(s)\}_{s,i,m})] \\
    & \leq (\exp{\{\frac{\lambda^2(P^{\prime}_t(m,j))^2(M+2)^2}{2}\frac{C\sigma^2}{\min_{j}N_{j,i}(t_{m,j})}\}})^{\frac{M}{M+2}} \cdot \\
 & \qquad \qquad \Pi_{j \in N_m(t)}\Pi_s(\exp{\frac{\lambda^2d_{m,t}^2(M+2)^2\sigma^2}{2N_{j,i}^2(t)}} \cdot E[1_{A_{\epsilon, \delta}} | \sigma(\{n_{m,i}(t)\}_{t,i,m})])^{\frac{1}{M+2}} \cdot \\
    & \qquad \qquad \Pi_{j \not\in N_m(t)}\Pi_s(\exp{\frac{\lambda^2d_{m,t}^2(M+2)^2\sigma^2}{2N_{j,i}^2(t_{m,j})}} \cdot E[1_{A_{\epsilon, \delta}} | \sigma(\{N_{m,i}(t)\}_{t,i,m})])^{\frac{1}{M+2}} \\
    & = (\exp{\{\frac{\lambda^2(P^{\prime}_t(m,j))^2(M+2)^2}{2}\frac{C\sigma^2}{\min_{j}N_{j,i}(t_{m,j})}\}})^{\frac{M}{M+2}} \cdot \\
 & \qquad \qquad \Pi_{j \in N_m(t)}\exp{\{\frac{N_{j,i}(t)}{M+2} \frac{\lambda^2d_{m,t}^2(M+2)^2\sigma^2}{2N_{j,i}^2(t)}\}} \cdot E[1_{A_{\epsilon, \delta}} | \sigma(\{N_{m,i}(t)\}_{t,i,m})] \cdot \\
 & \qquad \qquad \Pi_{j \not\in N_m(t)}\exp{\{\frac{N_{j,i}(t_{m,j})}{M+2} \frac{\lambda^2d_{m,t}^2(M+2)^2\sigma^2}{2N_{j,i}^2(t_{m,j})}\}} \cdot E[1_{A_{\epsilon, \delta}} | \sigma(\{n_{m,i}(t)\}_{t,i,m})] 
\end{align*}

Subsequently, we obtain 
\begin{align*}
 & E[\exp{\{\lambda(\Tilde{\mu}^m_i(t+1)-\mu_i)\}}1_{A_{\epsilon, \delta}} | \sigma(\{n_{m,i}(s)\}_{s,i,m})] \\
 & \leq (\exp{\{\frac{\lambda^2(P^{\prime}_t(m,j))^2(M+2)^2}{2}\frac{C\sigma^2}{\min_{j}N_{j,i}(t_{m,j})}\}})^{\frac{M}{M+2}} \cdot \\
 & \qquad \qquad (\exp{\{\frac{\lambda^2d^2_{m,t}(M+2)\sigma^2}{2\min_jN_{j,i}(t)}\}})^{|N_m(t)|} \cdot E[1_{A_{\epsilon, \delta}} | \sigma(\{n_{m,i}(t)\}_{t,i,m})] \cdot \\
 & \qquad \qquad (\exp{\{\frac{\lambda^2d^2_{m,t}(M+2)\sigma^2}{2\min_jN_{j,i}(t_{m,j})}\}})^{|M - N_m(t)|} \cdot E[1_{A_{\epsilon, \delta}} | \sigma(\{n_{m,i}(t)\}_{t,i,m})] \\
 & = E[(\exp{\{\frac{\lambda^2(P^{\prime}_t(m,j))^2M(M+2)}{2}\frac{C\sigma^2}{\min_{j}N_{j,i}(t_{m,j})}\}}) \cdot (\exp{\{\frac{\lambda^2d^2_{m,t}(M+2)|N_m(t)|}{2\min_jn_{j,i}(t)}\}}) \\
 & \qquad \qquad \cdot(\exp{\{\frac{\lambda^2d^2_{m,t}(M+2)\sigma^2|M - N_m(t)|}{2\min_jN_{j,i}(t_{m,j})}\}})  1_{A_{\epsilon, \delta}} | \sigma(\{n_{m,i}(t)\}_{t,i,m})]  \\
 & \leq E[(\exp{\{\frac{\lambda^2(P^{\prime}_t(m,j))^2M(M+2)}{2(1-c_0)}\frac{C\sigma^2}{\min_{j}N_{j,i}(t+1)}\}}) \cdot (\exp{\{\frac{\lambda^2d^2_{m,t}(M+2)|N_m(t)|\sigma^2}{2\frac{L/K}{L/K + 1}\min_jN_{j,i}(t+1)}\}}) \\
 & \qquad \qquad \cdot(\exp{\{\frac{\lambda^2d^2_{m,t}(M+2)|M - N_m(t)|\sigma^2}{2(1-c_0)\min_jN_{j,i}(t+1)}\}}) 1_{A_{\epsilon, \delta}} | \sigma(\{n_{m,i}(t)\}_{t,i,m})] 
\end{align*}

After organizing the terms in the above objective, we have  
 \begin{align*}
 & E[\exp{\{\lambda(\Tilde{\mu}^m_i(t+1)-\mu_i)\}}1_{A_{\epsilon, \delta}} | \sigma(\{n_{m,i}(s)\}_{s,i,m})] \\ 
 &  = E[(\exp\{ \frac{\lambda^2\sigma^2}{2\min_{j}N_{j,i}(t+1)} \cdot (\frac{C(P^{\prime}_t(m,j))^2M(M+2)}{2(1-c_0)} + \\
 & \qquad \qquad \frac{d^2_{m,t}(M+2)|N_m(t)|}{\frac{L/K}{L/K + 1}} + \frac{d^2_{m,t}(M+2)|M - N_m(t)|}{(1-c_0)})\} 1_{A_{\epsilon, \delta}} | \sigma(\{n_{m,i}(t)\}_{t,i,m})] \\
 & \leq E[\exp\{ \frac{C\lambda^2\sigma^2}{2\min_{j}N_{j,i}(t+1)}\} 1_{A_{\epsilon, \delta}} | \sigma(\{n_{m,i}(t)\}_{t,i,m})] \\
 & \leq \exp\{ \frac{C\lambda^2\sigma^2}{2\min_{j}N_{j,i}(t+1)}\}
\end{align*}
where the first inequality is true because of the specification of the parameters $P^{\prime}_t(m,j), d_{m,t}, L , c_0$ and $C$ and the second inequality holds true by the observation that $1_{A_{\epsilon, \delta}} \leq 1$ and $\min_{j}N_{j,i}(t+1) \in \sigma(\{n_{m,i}(t)\}_{t,i,m})$. 

The completion of this induction step subsequently completes the proof of Proposition 10.

\end{proof}

As in the proof of Theorem \ref{thm:2}, we next show how much difference between $\Tilde{\mu}^m_i(t)$ and $\mu_i$ by establishing the following concentration inequality. 

\textbf{Concentration inequality}
\\

\begin{Proposition}
    Assume the parameter $\delta$ satisfies that $0 < \delta
        < c = f(\epsilon,M,T)$. For any $m,i$ and $t > L$ where $L$ is the length of the burn-in period, $\Tilde{\mu}_{m,i}(t)$ satisfies that if if $N_{m,i}(t) \geq 2(K^2+KM+M)$, then with $P(A_{\epsilon, \delta}) = 1 -7\epsilon$,
    \begin{align*}
         & P(\Tilde{\mu}_{m,i}(t) - \mu_i  \geq \sqrt{\frac{C_1\log t}{N_{m,i}(t)}} |A_{\epsilon, \delta}) \leq \frac{1}{P(A_{\epsilon, \delta})}\frac{1}{t^2}, \\
         & P(\mu_i - \Tilde{\mu}_{m,i}(t) \geq \sqrt{\frac{C_1\log t}{N_{m,i}(t)}} | A_{\epsilon, \delta}) \leq \frac{1}{P(A_{\epsilon, \delta})t^2}.
    \end{align*}
\end{Proposition}

\begin{proof}[Proof of Proposition 11]
    The proof of this proposition follows from \citep{xu2023decentralized}. 
    
\end{proof}

Essentially, the above proposition implies that with high probability, we can identify the globally optimal arm by comparing all arms' global estimators $\Tilde{\mu}^m_i(t)$ with smaller sample complexity (by having $N_{m,i}(t)$ instead of $n_{m,i}(t)$). Subsequently, we next show that the number of pulling these globally sub-optimal arms can be upper bounded by the $\log{T}$ based on the concentration inequality. 

\textbf{Number of pulls of sub-optimal arms}
\\

Upper bounds on $E[n_{m,k}(T) | A_{\epsilon, \delta}]$
\\

\begin{Proposition}
    Assume the parameter $\delta$ satisfies that $0 < \delta
        < c = f(\epsilon,M,T)$. An arm $k$ is said to be sub-optimal if $k \neq i^*$ where $i^*$ is the unique optimal arm in terms of the global reward, i.e. $i^* = \arg\max \frac{1}{M}\sum_{j=1}^M\mu_i^j$. Then when the game ends, for every agent $m$, $0 < \epsilon  < 1$ and $T > L$, the expected numbers of pulling sub-optimal arm $k$ after the burn-in period satisfies with $P(A_{\epsilon, \delta}) = 1- 7\epsilon$
    \begin{align*}
         & E[n_{m,k}(T) | A_{\epsilon, \delta}]                                                                                    \\
         & \leq \max{\{\frac{C}{M} \cdot [\frac{4C_1\log T}{\Delta_i^2}], 2(K^2+MK+M) \}} +  \frac{2\pi^2}{3P(A_{\epsilon, \delta})} + K^2 + (2M-1)K \\
         & \leq O(\log{T}).
    \end{align*}
\end{Proposition}

\begin{proof}[Proof of Proposition 12]
    It is easy to observe that by repeating the proof steps of Proposition 6 with $N_{m,i}(t)$ instead of $n_{m,i}(t)$, we obtain that 
        \begin{align*}
         & E[N_{m,k}(T) | A_{\epsilon, \delta}]                                                                                    \\
         & \leq \max{\{ [\frac{4C_1\log T}{\Delta_i^2}], 2(K^2+MK+M) \}} +  \frac{2\pi^2}{3P(A_{\epsilon, \delta})} + K^2 + (2M-1)K \\
         & \leq O(\log{T}).
    \end{align*}

    It is worth noting that the decision rule of each agent relies on the cluster-wise information, i.e. $\tilde{\mu}_{m,i}(t)$ and $N_{m,i}(t)$, which is the same for all agents within one cluster. That being said, the agents within one cluster are pulling the same arm, formally written as 
    \begin{align*}
        N_{m,k}(T) = |c_M|n_{m,i}(t) = \frac{M}{C}n_{m,i}(t)
    \end{align*}
    since we assume a balanced cluster structure. 

    Subsequently, we derive that 
        \begin{align*}
         & E[n_{m,k}(T) | A_{\epsilon, \delta}] \\ 
         & \leq E[\frac{C}{M} \cdot N_{m,k}(T) | A_{\epsilon, \delta}]
         & \leq \max{\{\frac{C}{M} \cdot [\frac{4C_1\log T}{\Delta_i^2}], 2(K^2+MK+M) \}} +  \frac{2\pi^2}{3P(A_{\epsilon, \delta})} + K^2 + (2M-1)K \\
         & \leq O(\log{T}).
    \end{align*}
    which concludes the proof of Proposition 12. 
    
\end{proof}

Now we are ready to proceed to the proof of the main theorem by characterizing the regret. Likewise, we again perform regret decomposition. 

\textbf{Regret decomposition}

For the proposed regret, we have that for any constant $L$, 
\begin{align*}
    R_T & =   \frac{1}{M}(\max_i\sum_{t=1}^T\sum_{m=1}^M\mu^{m}_i - \sum_{t=1}^T\sum_{m=1}^M\mu^{m}_{a_t^m}) \\
    & = \sum_{t=1}^T\frac{1}{M}\sum_{m=1}^M\mu^{m}_{i^*} - \sum_{t=1}^T\frac{1}{M}\sum_{m=1}^M\mu^{m}_{a_t^m} \\
    & \leq \sum_{t = 1}^{L}|\frac{1}{M}\sum_{m=1}^M\mu^{m}_{i^*} - \frac{1}{M}\sum_{m=1}^M\mu^{m}_{a_t^m}|+ \sum_{t = L + 1}^T(\frac{1}{M}\sum_{m=1}^M\mu^{m}_{i^*} - \frac{1}{M}\sum_{m=1}^M\mu^{m}_{a_t^m}) \\
    & \leq L + \sum_{t = L + 1}^T(\frac{1}{M}\sum_{m=1}^M\mu^{m}_{i^*} - \frac{1}{M}\sum_{m=1}^M\mu^{m}_{a_t^m}) \\
    & = L + \sum_{t = L + 1}^T(\mu_{i^*} - \frac{1}{M}\sum_{m=1}^M\mu^{m}_{a_t^m}) \\
    & = L+ ((T - L) \cdot \mu_{i^*} - \frac{1}{M}\sum_{m=1}^M\sum_{i = 1}^Kn_{m,i}(T)\mu^m_i)
\end{align*}
where the first inequality is by taking the absolute value and the second inequality results from the assumption that $0 < \mu_{i}^j < 1$ for any arm $i$ and agent $j$.

Note that $\sum_{i=1}^K\sum_{m=1}^Mn_{m,i}(T) = M(T-L)$ where 
by definition $n_{m,i}(T)$ is the number of pulls of arm $i$ at agent $m$ from time step $L+1$ to time step $T$, which yields that
\begin{align*}
    R_T
    & \leq L + \sum_{i=1}^K\frac{1}{M}\sum_{m=1}^Mn_{m,i}(T)\mu_{i^*}^m - \sum_{i=1}^K\frac{1}{M}\sum_{m=1}^Mn_{m,i}(T)\mu_i^m \\
    & = L +  \sum_{i=1}^K\frac{1}{M}\sum_{m=1}^Mn_{m,i}(T)(\mu_{i^*}^m -\mu_i^m) \\
    & \leq L +  \frac{1}{M}\sum_{i=1}^K\sum_{m:\mu_{i^*}^m - \mu_i^m > 0}n_{m,i}(T)(\mu_{i^*}^m - \mu_i^m) \\
    & = L +  \frac{1}{M}\sum_{i \neq i*}\sum_{m:\mu_{i^*}^m - \mu_i^m > 0}n_{m,i}(T)(\mu_{i^*}^m - \mu_i^m) .
\end{align*}
where the second inequality uses the fact that $\sum_{m:\mu_{i^*}^m - \mu_i^m \leq 0}n_{m,i}(T)(\mu_{i^*}^m - \mu_i^m) \leq 0$ holds for any arm $i$ and the last equality is true since $n_{m,i}(T)(\mu_{i^*}^m - \mu_i^m) = 0$ for $i = i^*$ and any $m$.

By using Proposition 12 and the above inequality, we have that
\begin{align*}
    R_T \leq L +  \frac{1}{M}\sum_{i \neq i*}\sum_{m:\mu_{i^*}^m - \mu_i^m > 0}n_{m,i}(T)(\mu_{i^*}^m - \mu_i^m),  
\end{align*}

Consequently, we obtain that 
\begin{align*}
    E[R_T|A_{\epsilon, \delta}] & \leq L +  \frac{1}{M}\sum_{i \neq i*}\sum_{m:\mu_{i^*}^m - \mu_i^m > 0}E[n_{m,i}(T)](\mu_{i^*}^m - \mu_i^m) \\
    & \leq L + \sum_{i \neq i^*}\Delta_i(\max{\{\frac{C}{M} \cdot [\frac{4C_1\log T}{\Delta_i^2}], 2(K^2+MK) \}} +  \frac{2\pi^2}{3P(A_{\epsilon, \delta})} + K^2 + (2M-1)K)
\end{align*}
which concludes the proof.

\end{proof}

\subsection{Proof of Theorem \ref{thm:4}}

\begin{proof}
We first demonstrate that as long as the following lower bound on $p(m,n)$ holds, then we also have the claims about the connectivity of the sub-graphs and the claim about the delay induced by the random sub-graph. 

To this end, we establish the following proposition regarding the connectivity of the sub-graph. It is worth noting that the edge probability of this sub-graph, is now $c = \frac{e}{e-1}\frac{M^2}{C^2} \cdot \min_{m \neq n}P(m,n)$ based on Lemma \ref{lem:edge_prob_2}, and the total number of vertex is $C$ instead of $M$. 

We start with the graph connectivity. 

\textbf{Graph connectivity}

\begin{Proposition}~\label{prop:connectivity_setting_3}
Assume $c$ meets the condition 
\begin{align*}
    1 \geq c \geq (1 - \frac{\delta (C-1)}{8CT}), 
\end{align*}
where $0 < \epsilon < 1$. Then, with probability $1 - \epsilon$, for any $t > 0$,  the sub-graph $G_t^C$ following the E-R model is connected. 
\end{Proposition}

\begin{proof}[Proof of Proposition 13]

    We would like to highlight that this result can help solving the existing problems on random graphs, not limited to the bandit setting studied herein. It is worth noting that for any two clusters $j_1 \neq j_2 \neq m$, $1_{E_{m,j_1}}$ and $1_{(m,j_2) \in E_t^C}$ are identical random variables. Meanwhile, we note that the variance of $1_{(m,j_1) \in E_t^C}$ is no more than $1$. Subsequently, based on the Chebyshev's inequality, we obtain
    \begin{align*}
         & P(d_m \leq \frac{C-1}{2})                                                                                       \\
         & = P(d_m - A_{M-2}^{l-1}p^l(M-1) \leq \frac{C-1}{2} - A_{M-2}^{l-1}p^l(M-1))                                     \\
         & \leq P((d_m - A_{M-2}^{l-1}p^l(M-1))^2 \geq (\frac{C-1}{2} - A_{M-2}^{l-1}p^l(M-1))^2)                             \\
         & \leq \frac{Var(d_m)}{(C-1)^2(\frac{C-1}{2} - A_{M-2}^{l-1}p^l(M-1))^2}                                          \\
         & \leq \frac{(C-1)^2 \cdot A_{M-2}^{l-1}p^l(1- A_{M-2}^{l-1}p^l)}{(C-1)^2(\frac{C-1}{2} - A_{M-2}^{l-1}p^l(M-1))^2} \\
         & =  \frac{1}{(\frac{1}{2} - A_{M-2}^{l-1}p^l)^2} \cdot (1 - A_{M-2}^{l-1}p^l)                     \\
         & \leq 8 \cdot (1 - A_{M-2}^{l-1}p^l)                                                               \\
         & \leq \frac{\delta}{T}
    \end{align*}
    when we specify that $p \geq (1 - \frac{\delta (C-1)}{8T})$.

    Hence, we have
    \begin{align}\label{eq:d_j}
        P(d_m \leq \frac{C-1}{2}) \leq \frac{\delta}{T}. 
    \end{align}
    In other words, with probability at least $1- \frac{\delta}{T}$, we have that $d_m > \frac{C-1}{2}$.

    It is well known that if $\delta(G_t) \geq \frac{C-1}{2}$, then we have that the sub-graph $G_{t}^C$ is connected where $\delta(G_t) = \min_m{d_m}$.

    As a result, consider the probability and we obtain that
    \begin{align*}
         & P(\text{graph }  G_{t-l+1} \cdot \ldots, G_{t} \text{ is connected}) \\
         & \geq P(\min_j{d_j} \geq \frac{C-1}{2})                               \\
         & = P(\bigcap_j \{d_j \geq \frac{C-1}{2}\})                            \\
         & = 1 - P(\bigcup_j \{d_j < \frac{C-1}{2}\})                           \\
         & \geq 1 - \sum_jP( d_j < \frac{C-1}{2})                               \\
         & = 1 - CP( d_j < \frac{C-1}{2})                                       \\
         & \geq 1 - C\frac{\delta}{T} = 1 - \frac{C\delta}{T}
    \end{align*}
    where the second inequality holds by the Bonferroni's inequality and the third inequality uses (\ref{eq:d_j}).

    Consequently, we obtain
    \begin{align*}
         & P(\text{graph } G_t \text{ is connected})                  \\
         & = P(\cap_t\{G_t \text{ is connected}\})                    \\
         & \geq 1 - \sum_tP(G_t \text{ is not connected})             \\
         & = 1 - \sum_t(1 - P(G_t \text{ is connected}))              \\
         & \geq  1 - \sum_t(1 - (1-\frac{C\delta}{T}) = 1 - C\delta
    \end{align*}

    This indicates that with probability at least $1- \epsilon$, for any time step $t$, the corresponding composition graph is connected, which concludes the proof of Proposition 13. 
    
\end{proof}

Henceforth, we derive the following claim about the graph connectivity, by combining Proposition 13 with Proposition 7, which reads as follows. 

\begin{Proposition}~\label{prop:connectivity_setting_4}
Assume $c$ meets the condition 
\begin{align*}
    1 \geq c \geq \min{\{\frac{1}{2} + \frac{1}{2}\sqrt{1 - (\frac{\epsilon}{CT})^{\frac{2}{M-1}}}, (1 - \frac{\delta (C-1)}{8CT})\}}, 
\end{align*}
i.e.
\begin{align*}
    1 \geq \min_{m,n}p(m,n) \geq \frac{C^2}{M^2}\min{\{\frac{1}{2} + \frac{1}{2}\sqrt{1 - (\frac{\epsilon}{CT})^{\frac{2}{M-1}}}, (1 - \frac{\delta (C-1)}{8CT})\}}, 
\end{align*}
where $0 < \epsilon < 1$. Then, with probability $1 - \epsilon$, for any $t > 0$,  the sub-graph $G_t^C$ following the E-R model is connected. 
\end{Proposition}

\begin{proof}[Proof of Proposition 14]
    This is a direct result of merging Proposition 7 and Proposition 13. 
\end{proof}

With the characterization of the graph topology, we next demonstration  the consensus regarding arm pulls among the clusters, instead of the agents , since now the communication is on a cluster level. This is determined by how much information delay we have across the clusters (again, within one cluster, there is no such information delay).  

\textbf{Information delay}

\begin{lemma}
    For any $m,i,t > L$, if $N_{m,i}(t) \geq 2(K^2+KM+M)$ and subgraph $G_t$ induced by the clusters is connected, then we have
    \begin{align*}
        \hat{N}_{m,i}(t) \leq 2\min_{j}\hat{N}_{j,i}(t).
    \end{align*}
    where the min is taken over all clusters, not just the neighbors.
\end{lemma}

\begin{proof}[Proof of Lemma 7]
    The proof of this lemma again follows from Lemma 3 in \citep{zhu2023distributed}, with the exception that now the shared arm information is $N_{m,i}(t)$ instead of $n_{m,i}(t)$. 
    
\end{proof}

To proceed, we observe that since there are no modifications to the algorithm, the following results on the explicit transmission gap, unbiasedness of the estimator, variance of the estimator, concentration inequality, number of pulls of sub-optimal arms, hold. As a result, we omit the proof steps and refer to the proof of Theorem \ref{thm:3} for details. 

\textbf{Explicit transmission gap}

\begin{Proposition}
    %By the proper choice of $L$ in setting 1 and setting 2 and letting $t_0 = \frac{\ln{\frac{MT}{\eta}}}{\ln{2}}$ in setting 1 and setting 2 with $M \leq 10$ and $t_{0} \geq \frac{\ln(\frac{\eta}{M^2T})}{\ln(1-c)}$ in setting 2 with $M > 10$, 
    We have that with probability $1- \epsilon$,  for any $t > L$ and any $m$, there exists $$t_{0} \geq \frac{\ln(\frac{\epsilon}{M^2T})}{\min_{P(i,j)}\ln(1-P(i,j))}$$ such that
    \begin{align*}
         & t+1 - \min_jt_{m,j} \leq t_0, t_0 \leq c_0\min_{l}N_{l,i}(t+1)
    \end{align*}
    where $c_0$ $=$ $c_0(K, \min_{i \neq i^*}\Delta_i, M, \epsilon, \delta)$.
\end{Proposition}

\textbf{Unbiasedness of the estimator}

\begin{Proposition}\label{prop:unbiased_3}
    Assume the parameter $\delta$ satisfies that $0 < \delta
        < c = f(\epsilon,M,T)$. For any arm $i$ and any agent $m$, at every time step $t$, we have
    \begin{align*}
        E[\Tilde{\mu}_i^m(t) | A_{\epsilon, \delta}] = \mu_i.
    \end{align*}
\end{Proposition}

\textbf{Variance term}

\begin{Proposition}
    Assume the parameter $\delta$ satisfies that $0 < \delta
        < c = f(\epsilon,M,T)$. In setting $s_1, s_2, s_3$ where rewards follow sub-gaussian distributions, for any $m,i, \lambda$ and $t > L$ where $L$ is the length of the burn-in period, the global estimator $\Tilde{\mu}^m_i(t)$ is sub-Gaussian distributed. Moreover, the conditional moment generating function satisfies that with $P(A_{\epsilon, \delta}) = 1 - 7\epsilon$,
    \begin{align*}
         & E[\exp{\{\lambda(\Tilde{\mu}^m_i(t)  - \mu_i})\}1_{A_{\epsilon, \delta}} | \sigma(\{n_{m,i}(t)\}_{t,i,m})] \\
         & \leq \exp{\{\frac{\lambda^2}{2}\frac{C\sigma^2}{\min_{j}N_{j,i}(t)}\}}
    \end{align*}
    where $\sigma^2 = \max_{j,i}(\Tilde{\sigma}_i^j)^2$ and $C = \max\{\frac{4(M+2)(1 - \frac{1 - c_0}{2(M+2)})^2}{3M(1-c_0)}, (M+2)(1 + 4Md^2_{m,t})/M\}$.
\end{Proposition}

\textbf{Concentration inequality}
\\

\begin{Proposition}
    Assume the parameter $\delta$ satisfies that $0 < \delta
        < c = f(\epsilon,M,T)$. For any $m,i$ and $t > L$ where $L$ is the length of the burn-in period, $\Tilde{\mu}_{m,i}(t)$ satisfies that if if $N_{m,i}(t) \geq 2(K^2+KM+M)$, then with $P(A_{\epsilon, \delta}) = 1 -7\epsilon$,
    \begin{align*}
         & P(\Tilde{\mu}_{m,i}(t) - \mu_i  \geq \sqrt{\frac{C_1\log t}{N_{m,i}(t)}} |A_{\epsilon, \delta}) \leq \frac{1}{P(A_{\epsilon, \delta})}\frac{1}{t^2}, \\
         & P(\mu_i - \Tilde{\mu}_{m,i}(t) \geq \sqrt{\frac{C_1\log t}{N_{m,i}(t)}} | A_{\epsilon, \delta}) \leq \frac{1}{P(A_{\epsilon, \delta})t^2}.
    \end{align*}
\end{Proposition}

\textbf{Number of pulls of sub-optimal arms}
\\

Upper bounds on $E[n_{m,k}(T) | A_{\epsilon, \delta}]$
\\

\begin{Proposition}
    Assume the parameter $\delta$ satisfies that $0 < \delta
        < c = f(\epsilon,M,T)$. An arm $k$ is said to be sub-optimal if $k \neq i^*$ where $i^*$ is the unique optimal arm in terms of the global reward, i.e. $i^* = \arg\max \frac{1}{M}\sum_{j=1}^M\mu_i^j$. Then when the game ends, for every agent $m$, $0 < \epsilon  < 1$ and $T > L$, the expected numbers of pulling sub-optimal arm $k$ after the burn-in period satisfies with $P(A_{\epsilon, \delta}) = 1- 7\epsilon$
    \begin{align*}
         & E[n_{m,k}(T) | A_{\epsilon, \delta}]                                                                                    \\
         & \leq \max{\{\frac{C}{M} \cdot [\frac{4C_1\log T}{\Delta_i^2}], 2(K^2+MK+M) \}} +  \frac{2\pi^2}{3P(A_{\epsilon, \delta})} + K^2 + (2M-1)K \\
         & \leq O(\log{T}).
    \end{align*}
\end{Proposition}

Lastly, by the aforementioned regret decomposition, we obtain 
\begin{align*}
    E[R_T|A_{\epsilon, \delta}] & \leq L +  \frac{1}{M}\sum_{i \neq i*}\sum_{m:\mu_{i^*}^m - \mu_i^m > 0}E[n_{m,i}(T)](\mu_{i^*}^m - \mu_i^m) \\
    & \leq L + \sum_{i \neq i^*}\Delta_i(\max{\{\frac{C}{M} \cdot [\frac{4C_1\log T}{\Delta_i^2}], 2(K^2+MK) \}} +  \frac{2\pi^2}{3P(A_{\epsilon, \delta})} + K^2 + (2M-1)K)
\end{align*}
which concludes the proof. 

\end{proof}

\subsection{Proof of Theorem \ref{thm:5}}

\begin{proof}
We prove the result following a similar analytical order. 

We start by examining the graph connectivity, with respect to the sub-graph induced by the clusters. Instead of requiring that the sub-graph is connected for every time step, it holds true that as long as the sub-graph is $l$-periodically connected, we guarantee the same consensus among the clusters (represented by the information delay). And the assumption on $\min_{m,n}p(m,n)$ in Theorem \ref{thm:4} guarantees that with high probability, the sub-graph is $l$-periodically connected at every time step, which is formally presented as follows.  

\textbf{$l$-periodically connectivity}

\begin{Proposition}
    Let us assume that $\min_{m,n}p(m,n) \geq \frac{C^2}{M^2}\max{\{\frac{(C-l-1)!}{(C-2)!}(1 - \frac{\delta (C-1)}{8CT}), \frac{(C-l-1)!}{(C-2)!}(\frac{3}{4})^{\frac{1}{l}}\}}$. Then with probability at least $1-\delta$, for any $t$, the sequence starting $G_t^C$ is a $l$-periodically connected graph. 
\end{Proposition}

\begin{proof}[Proof of Proposition 20]

We would like to highlight that this result can help solving the existing problems on random graphs, not limited to the bandit setting studied herein. The probability of having $l$ periodically connected graph, i.e. the composition of $G_1, G_2, \ldots, G_l$ is a connected graph, which means that for any two clusters $m,n$, the  probability of having a path among them during the $l$ steps. Formally, let us define $E_{m,n} = \exists a_1, a_2, \ldots, a_{l-1}, s.t. 1_{(m, a_1) \in G_1, (a_1, a_2) \in G_2, \ldots, (a_{l-1},n) \in G_l} = 1$

    Let us denote $p =  c$ as $\frac{M^2}{C^2}p_{m,n}$, which again represents the edge probability of the sub-graph.

    \begin{align*}
         & P(E_{m,n})                                                                                                                             \\
         & = P(\exists a_1, a_2, \ldots, a_{l-1}, s.t. 1_{(m, a_1) \in G_1, (a_1, a_2) \in G_2, \ldots, (a_{l-1},n) \in G_l} = 1)                 \\
         & = P(\exists a_1, a_2, \ldots, a_{l-1}, s.t. 1_{(m, a_1) \in G_1} = 1, 1_{(a_1, a_2) \in G_2} = 1, \ldots, 1_{(a_{l-1},n) \in G_l} = 1) \\
         & = A_{M-2}^{l-1}p^l
    \end{align*}

    Let us define the degree of cluster $m$ as $d_m$. We then derive that
    \begin{align*}
         & E[d_m]                                            \\
         & = E[\sum_{n=1}^{C}1_{n \neq m} \cdot 1_{E_{m,n}}] \\
         & = (C-1) \cdot A_{M-2}^{l-1}p^l
    \end{align*}

    Likewise, we derive that
    \begin{align*}
         & Var(d_m)                                                \\
         & \leq (C-1)\sum_{n=1}^{C}Var(1_{n \neq m} \cdot 1_{E_{m,n}}]) \\
         & \leq (C-1)^2 \cdot A_{M-2}^{l-1}p^l(1- A_{M-2}^{l-1}p^l)
    \end{align*}

    Let us assume that
    \begin{align*}
        A_{M-2}^{l-1}p^l \geq \frac{3}{4}
    \end{align*}
    which implies that $p \geq (\frac{4}{3A_{M-2}^{l-1}})^{\frac{1}{l}}$.

    It is worth noting that for any two clusters $j_1 \neq j_2 \neq m$, $1_{E_{m,j_1}}$ and $1_{E_{m,j_2}}$ are dependent but identical random variables. Meanwhile, we note that the variance of $1_{E_{m,j_1}}$ is no more than $1$. Subsequently, based on the Chebyshev's inequality, we obtain
    \begin{align*}
         & P(d_m \leq \frac{C-1}{2})                                                                                       \\
         & = P(d_m - A_{M-2}^{l-1}p^l(M-1) \leq \frac{C-1}{2} - A_{M-2}^{l-1}p^l(M-1))                                     \\
         & \leq P((d_m - A_{M-2}^{l-1}p^l(M-1))^2 \geq (\frac{C-1}{2} - A_{M-2}^{l-1}p^l(M-1))^2)                             \\
         & \leq \frac{Var(d_m)}{(C-1)^2(\frac{C-1}{2} - A_{M-2}^{l-1}p^l(M-1))^2}                                          \\
         & \leq \frac{(C-1)^2 \cdot A_{M-2}^{l-1}p^l(1- A_{M-2}^{l-1}p^l)}{(C-1)^2(\frac{C-1}{2} - A_{M-2}^{l-1}p^l(M-1))^2} \\
         & =  \frac{1}{(\frac{1}{2} - A_{M-2}^{l-1}p^l)^2} \cdot (1 - A_{M-2}^{l-1}p^l)                     \\
         & \leq 8 \cdot (1 - A_{M-2}^{l-1}p^l)                                                               \\
         & \leq \frac{\delta}{T}
    \end{align*}
    when we specify that $p \geq \frac{(C-l-1)!}{(C-2)!}(1 - \frac{\delta (C-1)}{8T})$.

    In other words, with probability at least $1- \frac{\delta \cdot l}{T}$, we have that $d_m > \frac{C-1}{2}$.

    It is well known that if $\delta(G_t) \geq \frac{C-1}{2}$, then we have that the composition graph $G_{t-l+1} \cdot \ldots, G_{t}$ is connected where $\delta(G_t) = \min_m{d_m}$.

    As a result, consider the probability and we obtain that
    \begin{align*}
         & P(\text{graph }  G_{t-l+1} \cdot \ldots, G_{t} \text{ is connected}) \\
         & \geq P(\min_j{d_j} \geq \frac{C-1}{2})                               \\
         & = P(\bigcap_j \{d_j \geq \frac{C-1}{2}\})                            \\
         & = 1 - P(\bigcup_j \{d_j < \frac{C-1}{2}\})                           \\
         & \geq 1 - \sum_jP( d_j < \frac{C-1}{2})                               \\
         & = 1 - CP( d_j < \frac{C-1}{2})                                       \\
         & \geq 1 - C\frac{\epsilon}{T} = 1 - \frac{C\epsilon}{T}
    \end{align*}
    where the second inequality holds by the Bonferroni's inequality and the third inequality uses (\ref{eq:d_j}).

    Consequently, we obtain
    \begin{align*}
         & P(\text{graph } G_t \text{ is connected})                  \\
         & = P(\cap_t\{G_t \text{ is connected}\})                    \\
         & \geq 1 - \sum_tP(G_t \text{ is not connected})             \\
         & = 1 - \sum_t(1 - P(G_t \text{ is connected}))              \\
         & \geq  1 - \sum_t(1 - (1-\frac{C\epsilon}{T}) = 1 - C\epsilon
    \end{align*}

    This indicates that with probability at least $1- \epsilon$, for any time step $t$, the corresponding composition graph of $G_t^C$ is connected. By definition, we conclude that the sequence $G_t^C$ is $l$-periodically connected, which completes the proof.

\end{proof}

\textbf{Information delay}

\begin{lemma}
    For any $m,i,t > L$, if $n_{m,i}(t) \geq 2(K^2+KM+M)$ and subgraph $G_t$ is $l$-periodically connected in the sense that the composition of $l$ consecutive graphs is a connected graph, then we have
    \begin{align*}
        \hat{N}_{m,i}(t) \leq 2\min_{j}\hat{N}_{j,i}(t).
    \end{align*}
    where the min is taken over all clusters, not just the neighbors.
\end{lemma}

\begin{proof}[Proof of Lemma 9]

We consider Lemma 10 in \cite{zhu2023distributed}, and thus establish that
    \begin{align*}
        \hat{N}_{m,i}(t) \leq 2\min_{j}\hat{N}_{j,i}(t).
    \end{align*}

We refer the full proof to the proof of Lemma 10 in \cite{zhu2023distributed}. 
    
\end{proof}

Then with the information delay results and the same update rule (Rule 2) as in Theorem \ref{thm:3}, we can again establish the unbiasedness of the global estimator $\tilde{\mu}_{m,i}(t)$, the variance of $\tilde{\mu}_{m,i}(t)$, the concentration inequality with respect to $\tilde{\mu}_{m,i}(t)$, the upper bound on the number of pulls of sub-optimal arms, which are presented as follows (the proof of them is referred to the proof of Theorem \ref{thm:3}).

\textbf{Unbiasedness of the estimator}

\begin{Proposition}\label{prop:unbiased_4}
    Assume the parameter $\delta$ satisfies that $0 < \delta
        < c = f(\epsilon,M,T)$. For any arm $i$ and any agent $m$, at every time step $t$, we have
    \begin{align*}
        E[\Tilde{\mu}_i^m(t) | A_{\epsilon, \delta}] = \mu_i.
    \end{align*}
\end{Proposition}

\textbf{Variance term}

\begin{Proposition}
    Assume the parameter $\delta$ satisfies that $0 < \delta
        < c = f(\epsilon,M,T)$. In setting $s_1, s_2, s_3$ where rewards follow sub-gaussian distributions, for any $m,i, \lambda$ and $t > L$ where $L$ is the length of the burn-in period, the global estimator $\Tilde{\mu}^m_i(t)$ is sub-Gaussian distributed. Moreover, the conditional moment generating function satisfies that with $P(A_{\epsilon, \delta}) = 1 - 7\epsilon$,
    \begin{align*}
         & E[\exp{\{\lambda(\Tilde{\mu}^m_i(t)  - \mu_i})\}1_{A_{\epsilon, \delta}} | \sigma(\{n_{m,i}(t)\}_{t,i,m})] \\
         & \leq \exp{\{\frac{\lambda^2}{2}\frac{C\sigma^2}{\min_{j}N_{j,i}(t)}\}}
    \end{align*}
    where $\sigma^2 = \max_{j,i}(\Tilde{\sigma}_i^j)^2$ and $C = \max\{\frac{4(M+2)(1 - \frac{1 - c_0}{2(M+2)})^2}{3M(1-c_0)}, (M+2)(1 + 4Md^2_{m,t})/M\}$.
\end{Proposition}

\textbf{Concentration inequality}
\\

\begin{Proposition}
    Assume the parameter $\delta$ satisfies that $0 < \delta
        < c = f(\epsilon,M,T)$. For any $m,i$ and $t > L$ where $L$ is the length of the burn-in period, $\Tilde{\mu}_{m,i}(t)$ satisfies that if if $N_{m,i}(t) \geq 2(K^2+KM+M)$, then with $P(A_{\epsilon, \delta}) = 1 -7\epsilon$,
    \begin{align*}
         & P(\Tilde{\mu}_{m,i}(t) - \mu_i  \geq \sqrt{\frac{C_1\log t}{N_{m,i}(t)}} |A_{\epsilon, \delta}) \leq \frac{1}{P(A_{\epsilon, \delta})}\frac{1}{t^2}, \\
         & P(\mu_i - \Tilde{\mu}_{m,i}(t) \geq \sqrt{\frac{C_1\log t}{N_{m,i}(t)}} | A_{\epsilon, \delta}) \leq \frac{1}{P(A_{\epsilon, \delta})t^2}.
    \end{align*}
\end{Proposition}

\textbf{Number of pulls of sub-optimal arms}
\\

Upper bounds on $E[n_{m,k}(T) | A_{\epsilon, \delta}]$
\\

\begin{Proposition}
    Assume the parameter $\delta$ satisfies that $0 < \delta
        < c = f(\epsilon,M,T)$. An arm $k$ is said to be sub-optimal if $k \neq i^*$ where $i^*$ is the unique optimal arm in terms of the global reward, i.e. $i^* = \arg\max \frac{1}{M}\sum_{j=1}^M\mu_i^j$. Then when the game ends, for every agent $m$, $0 < \epsilon  < 1$ and $T > L$, the expected numbers of pulling sub-optimal arm $k$ after the burn-in period satisfies with $P(A_{\epsilon, \delta}) = 1- 7\epsilon$
    \begin{align*}
         & E[n_{m,k}(T) | A_{\epsilon, \delta}]                                                                                    \\
         & \leq \max{\{\frac{C}{M} \cdot [\frac{4C_1\log T}{\Delta_i^2}], 2(K^2+MK+M) \}} +  \frac{2\pi^2}{3P(A_{\epsilon, \delta})} + K^2 + (2M-1)K \\
         & \leq O(\log{T}).
    \end{align*}
\end{Proposition}

As a concluding step, we again use the aforementioned regret decomposition that leads to the following 
\begin{align*}
    E[R_T|A_{\epsilon, \delta}] & \leq L +  \frac{1}{M}\sum_{i \neq i*}\sum_{m:\mu_{i^*}^m - \mu_i^m > 0}E[n_{m,i}(T)](\mu_{i^*}^m - \mu_i^m) \\
    & \leq L + \sum_{i \neq i^*}\Delta_i(\max{\{\frac{C}{M} \cdot [\frac{4C_1\log T}{\Delta_i^2}], 2(K^2+MK) \}} +  \frac{2\pi^2}{3P(A_{\epsilon, \delta})} + K^2 + (2M-1)K). 
\end{align*}

This completes the proof of Theorem \ref{thm:5}.

\end{proof}

\subsection{Proof of Theorem \ref{thm:6}}

\begin{proof}

In a like manner, we approach the regret analysis by decomposing the proof into the following phases, in the order indicated below. 

We start with characterizing the sub-graph connectivity (referring to the sub-graphs including the one with respect to the agents within one cluster and the other with respect to the clusters). 

\textbf{Graph connectivity}

\begin{Proposition}
Assume the edge probabilities meet the condition 
\begin{align*}
    1 \geq \min_{m \neq n}p(m,n) \geq \frac{C^2}{M^2}\max{\{\frac{(C-l-1)!}{(C-2)!}(1 - \frac{\delta (C-1)}{8CT}), \frac{(C-l-1)!}{(C-2)!}(\frac{3}{4})^{\frac{1}{l}}\}}, 
\end{align*}
and 
\begin{align*}
    1 \geq \min_{m,m}p(m,m) \geq \max{\{\frac{(c_M-l-1)!}{(c_M-2)!}(1 - \frac{\delta (c_M-1)}{8c_MT}), \frac{(c_M-l-1)!}{(c_M-2)!}(\frac{3}{4})^{\frac{1}{l}}\}}, 
\end{align*}
where $0 < \epsilon < 1$. Then, with probability $1 - \epsilon$, for any $t > 0$,  the sub-graph $G_t^C$ induced by the clusters following the E-R model is $l$-periodically connected and the sub-graph $G_t^{c_M}$ induced by the agents within one cluster (since we consider a balanced cluster) is also $l$-periodically connected. 
\end{Proposition}

\begin{proof}[Proof of Proposition 25]
    The first part of the statement follows from Proposition 16 in the proof of Theorem \ref{thm:3}. 

    For the second part, we repeat the proof of Proposition by treating the sub-graph as the sub-graph induced by the agents in one cluster, which has $c_M$ vertex and the corresponding edge set determined by $\{p(m,m)\}_{m}$, instead of the sub-graph induced by the clusters. As a result, we omit the proof steps here and refer to Proposition 16 for a detailed version of the proof. 
    
\end{proof}

Then we consider the information delay due to the randomness in both the cluster-wise sub-graph $G_t^C$ and the within-cluster sub-graph $G_t^{c_M}$. 

\textbf{Information delay}

The following lemma characterize the first case. 

\begin{lemma}
    For any $m,i,t > L$, if $N_{m,i}(t) \geq 2(K^2+KM+M)$ and subgraph $G_t$ is $l$-periodically connected in the sense that the composition of $l$ consecutive graphs is a connected graph, then we have
    \begin{align*}
        \hat{N}_{m,i}(t) \leq 2\min_{j}\hat{N}_{j,i}(t).
    \end{align*}
    where the min is taken over all clusters, not just the neighbors.
\end{lemma}

\begin{proof}[Proof of Lemma 10]

We consider Lemma 10 in \cite{zhu2023distributed}, and thus establish that
    \begin{align*}
        \hat{N}_{m,i}(t) \leq 2\min_{j}\hat{N}_{j,i}(t).
    \end{align*}
    
\end{proof}

Additionally, we have the following proposition

\begin{lemma}
    For any $m,i,t > L$, if $n_{m,i}(t) \geq 2(K^2+KM+M)$ and subgraph $G_t^{c_M}$ is $l$-periodically connected, then we have that for any $t$ and $m \in c_m$, 
    \begin{align*}
        & N_{m,i}(t)(t) - c_M \cdot l \leq \min_{m \in c_m}N_{m,i}(t-l) \leq N_{m,i}(t) \\
        & N_{m,i}(t-l) \geq \hat{N}_{m,i}(t) - K(K+2M) - c_M \cdot l 
    \end{align*}
    where the min is taken over all agents in one cluster.
\end{lemma}

\begin{proof}[Proof of Lemma 11]

We observe that all agents in one cluster will collect each other's information, after at most $l$ steps, and only after that, they update the cluster information $\hat{\mu}_m^i(t)$ that aggregates all agents ' information and thus is the same for all agents in one cluster. And in between, the agents use the most recent cluster estimator, which is the same for agents in one cluster.

This implies that $n_{m,i}(t) = n_{j,i}(t)$ for $m,j \in c_m$, as well as $\min_{m \in c_m}N_{m,i}(t-l) = N_{m,i}(t-l) = N_{j,i}(t-l) \geq N_{j,i}(t) - c_M \cdot l$

Meanwhile, based on Lemma 1 in \citep{zhu2023distributed}, we have that 
\begin{align*}
    N_{m,i}(t) \geq \hat{N}_{m,i}(t) - K(K+2M)
\end{align*}
and thus 
\begin{align*}
    N_{m,i}(t-l) \geq \hat{N}_{m,i}(t) - K(K+2M) - c_M \cdot l.
\end{align*}

This completes the proof.

\end{proof}

\textbf{Unbiasedness}

\begin{Proposition}\label{prop:unbiased_5}
    Assume the parameter $\delta$ satisfies that $0 < \delta
        < c = f(\epsilon,M,T)$. For any arm $i$ and any agent $m$, at every time step $t$, we have
    \begin{align*}
        E[\Tilde{\mu}_i^m(t) | A_{\epsilon, \delta}] = \mu_i.
    \end{align*}
\end{Proposition}

\begin{proof}[Proof of Proposition 26]

It is worth noting that the information delay of length of $l$ does not change the expected value of the global estimator, since the delayed estimator is also unbiased, the same as before. Hence, the proof of Proposition 16 for Theorem \ref{thm:3} holds herein, and thus we refer to the proof there. 
    
\end{proof}

\textbf{Variance term}

\begin{Proposition}
    Assume the parameter $\delta$ satisfies that $0 < \delta
        < c = f(\epsilon,M,T)$. In setting $s_1, s_2, s_3$ where rewards follow sub-gaussian distributions, for any $m,i, \lambda$ and $t > L$ where $L$ is the length of the burn-in period, the global estimator $\Tilde{\mu}^m_i(t)$ is sub-Gaussian distributed. Moreover, the conditional moment generating function satisfies that with $P(A_{\epsilon, \delta}) = 1 - 7\epsilon$,
    \begin{align*}
         & E[\exp{\{\lambda(\Tilde{\mu}^m_i(t)  - \mu_i})\}1_{A_{\epsilon, \delta}} | \sigma(\{n_{m,i}(t)\}_{t,i,m})] \\
         & \leq \exp{\{\frac{\lambda^2}{2}\frac{C\sigma^2}{\min_{j}N_{j,i}(t-l)}\}}
    \end{align*}
    where $\sigma^2 = \max_{j,i}(\Tilde{\sigma}_i^j)^2$ and $C = \max\{\frac{4(M+2)(1 - \frac{1 - c_0}{2(M+2)})^2}{3M(1-c_0)}, (M+2)(1 + 4Md^2_{m,t})/M\}$.
\end{Proposition}

\begin{proof}[Proof of Proposition 27]
    Based on Lemma 11, we observe that there is $l$ delay in the within-cluster information, which implies that the quantity $N_{m,i}(t)$ is at least $ N_{m,i}(t-l)$. Also, we note that the term $\exp{\{\frac{\lambda^2}{2}\frac{C\sigma^2}{\min_{j}N_{j,i}(t)}\}}$ is monotone decreasing in $N_{j,i}(t)$, which means that we can use this term $N_{j,i}(t-l)$ in characterizing the moment generating function of $\Tilde{\mu}^m_i(t)$, as well as the variance. 

    With $N_{m,i}(t-l)$, we repeat the proof of Proposition 10, following the proof in \citep{xu2023decentralized}, which concludes the result. 
    
\end{proof}

\textbf{Concentration inequality}

\begin{Proposition}
    Assume the parameter $\delta$ satisfies that $0 < \delta
        < c = f(\epsilon,M,T)$. For any $m,i$ and $t > L$ where $L$ is the length of the burn-in period, $\Tilde{\mu}_{m,i}(t)$ satisfies that if if $N_{m,i}(t) \geq 2(K^2+KM+M)$, then with $P(A_{\epsilon, \delta}) = 1 -7\epsilon$,
\begin{align*}
         & P(\Tilde{\mu}_{m,i}(t) - \mu_i  \geq \sqrt{\frac{C_1\log t}{N_{m,i}(t-l)}} |A_{\epsilon, \delta}) \leq \frac{1}{P(A_{\epsilon, \delta})}\frac{1}{t^2}, \\
         & P(\mu_i - \Tilde{\mu}_{m,i}(t) \geq \sqrt{\frac{C_1\log t}{N_{m,i}(t-l)}} | A_{\epsilon, \delta}) \leq \frac{1}{P(A_{\epsilon, \delta})t^2}.
\end{align*}
\end{Proposition}

\begin{proof}[Proof of Proposition 28]

It is worth noting that in Algorithm 2, we specify the update frequency as $\tau = l$, which means between $t-l$ and $t$, the agents in the same cluster do not update the information in order to make sure that they stay on the same page. Also, using Lemma  we obtain that with high probability $1 - \delta$, after at most $l$ steps, any two agents in one cluster communicate, and any two cluster communicate.  At the end of $t$, they already collect all the information of agents within the cluster and they update the information. Based on the concentration inequality we obtained for the cluster-wise information as in the case where the within-cluster graph is a complete graph, we obtain   
\begin{align*}
         & P(\Tilde{\mu}_{m,i}(t) - \mu_i  \geq \sqrt{\frac{C_1\log t}{N_{m,i}(t-l)}} |A_{\epsilon, \delta}) \leq \frac{1}{P(A_{\epsilon, \delta})}\frac{1}{t^2}, \\
         & P(\mu_i - \Tilde{\mu}_{m,i}(t) \geq \sqrt{\frac{C_1\log t}{N_{m,i}(t-l)}} | A_{\epsilon, \delta}) \leq \frac{1}{P(A_{\epsilon, \delta})t^2}.
\end{align*}

\end{proof}

\textbf{Number of pulls of sub-optimal arms}

\begin{Proposition}
    Assume the parameter $\delta$ satisfies that $0 < \delta
        < c = f(\epsilon,M,T)$. An arm $k$ is said to be sub-optimal if $k \neq i^*$ where $i^*$ is the unique optimal arm in terms of the global reward, i.e. $i^* = \arg\max \frac{1}{M}\sum_{j=1}^M\mu_i^j$. Then when the game ends, for every agent $m$, $0 < \epsilon  < 1$ and $T > L$, the expected numbers of pulling sub-optimal arm $k$ after the burn-in period satisfies with $P(A_{\epsilon, \delta}) = 1- 7\epsilon$
    \begin{align*}
         & E[n_{m,k}(T) | A_{\epsilon, \delta}]                                                                                    \\
         & \leq \max{\{\frac{C}{M} \cdot [\frac{4C_1\log T}{\Delta_i^2}], 2(K^2+MK+M) \}} +  \frac{2\pi^2}{3P(A_{\epsilon, \delta})} + K^2 + (2M-1)K + M \cdot l \\
         & \leq O(\log{T}).
    \end{align*}
\end{Proposition}

\begin{proof}[Proof of Proposition 29]
It is worth noting that by time $t$, each agent $m \in c_m$ has received the information $\mathcal{F}_{t-l}$ of all agents in the same cluster. In other words, $N_{m,i}(t) = \sum_{m \in c_m}n_{m,i}(t-l)$, which implies that $N_{m,i}(t) = N_{j,i}(t)$ for any agent $m, j \in c_m$ for any $t > 2l$.  

Likewise, we obtain that $\Tilde{\bar{\mu}}_{m,i}(t) = \Tilde{\bar{\mu}}_{j,i}(t)$ for any $m, j \in c_m$, since we update $\Tilde{\bar{\mu}}_{m,i}(t) = \sum_{j \in c_m}\frac{\Tilde{\mu}_{j,i}(t-l)}{|C_m|} = \hat{\mu}_{j,i}(t)$. And the agents across the cluster exchange $\hat{\mu}_{j,i}(t)$ if they do not belong to the same cluster, and update the estimation towards  $\Tilde{\mu}_{j,i}(t)$.

In light of the UCB decision rule, $n_{m, i}(t)$ and  $n_{j, i}(t)$ only depend on $N_{m,i}(t), \Tilde{\bar{\mu}}_{m,i}(t)$. Therefore, we derive that $n_{m, i}(t) = n_{j, i}(t)$ for any $t$ for any $m,j \in c_m$, on event $A_{\epsilon,\delta}$.

Also, it is worth mentioning that $n_{m,i}(t) \leq  \frac{N_{m,i}(t)}{|C_m|}$, concluded from the above statement and the fact that the cluster structure is balanced.  

By considering 4 different cases regarding the possible values of $N_{m,i}(t)$ as in \citep{xu2023decentralized} and noticing that $N_{m,i}(t) \leq N_{m,i}(t-l) + c_M \cdot l$ and the fact that $n_{m,i}(t) \leq  \frac{N_{m,i}(t)}{|c_M|}$, we obtain that 
\begin{align}\label{eq:en_bound}
    & E[n_{m,i}(T) | A_{\epsilon, \delta}] \notag \\
    & \leq \frac{E[N_{m,i}(T) | A_{\epsilon, \delta}]}{|C_m|} + l \notag \\
    & \leq \max{\{[\frac{4C_1\log T}{|C_m|\Delta_i^2}], 2(K^2+MK+M) \}} +  \frac{2\pi^2}{3} + K^2 + (2M-1)K + l.
\end{align}
\end{proof}

Next, we proceed to the regret decomposition and derive the upper bound on the regret. 

\textbf{Regret decomposition}

For the proposed regret, we have that for any constant $L$, 
\begin{align*}
    R_T & =   \frac{1}{M}(\max_i\sum_{t=1}^T\sum_{m=1}^M\mu^{m}_i - \sum_{t=1}^T\sum_{m=1}^M\mu^{m}_{a_t^m}) \\
    & = \sum_{t=1}^T\frac{1}{M}\sum_{m=1}^M\mu^{m}_{i^*} - \sum_{t=1}^T\frac{1}{M}\sum_{m=1}^M\mu^{m}_{a_t^m} \\
    & \leq \sum_{t = 1}^{L}|\frac{1}{M}\sum_{m=1}^M\mu^{m}_{i^*} - \frac{1}{M}\sum_{m=1}^M\mu^{m}_{a_t^m}|+ \sum_{t = L + 1}^T(\frac{1}{M}\sum_{m=1}^M\mu^{m}_{i^*} - \frac{1}{M}\sum_{m=1}^M\mu^{m}_{a_t^m}) \\
    & \leq L + \sum_{t = L + 1}^T(\frac{1}{M}\sum_{m=1}^M\mu^{m}_{i^*} - \frac{1}{M}\sum_{m=1}^M\mu^{m}_{a_t^m}) \\
    & = L + \sum_{t = L + 1}^T(\mu_{i^*} - \frac{1}{M}\sum_{m=1}^M\mu^{m}_{a_t^m}) \\
    & = L+ ((T - L) \cdot \mu_{i^*} - \frac{1}{M}\sum_{m=1}^M\sum_{i = 1}^Kn_{m,i}(T)\mu^m_i)
\end{align*}
where the first inequality is by taking the absolute value and the second inequality results from the assumption that $0 < \mu_{i}^j < 1$ for any arm $i$ and agent $j$.

Note that $\sum_{i=1}^K\sum_{m=1}^Mn_{m,i}(T) = M(T-L)$ where 
by definition $n_{m,i}(T)$ is the number of pulls of arm $i$ at agent $m$ from time step $L+1$ to time step $T$, which yields that
\begin{align*}
    R_T
    & \leq L + \sum_{i=1}^K\frac{1}{M}\sum_{m=1}^Mn_{m,i}(T)\mu_{i^*}^m - \sum_{i=1}^K\frac{1}{M}\sum_{m=1}^Mn_{m,i}(T)\mu_i^m \\
    & = L +  \sum_{i=1}^K\frac{1}{M}\sum_{m=1}^Mn_{m,i}(T)(\mu_{i^*}^m -\mu_i^m) \\
    & \leq L +  \frac{1}{M}\sum_{i=1}^K\sum_{m:\mu_{i^*}^m - \mu_i^m > 0}n_{m,i}(T)(\mu_{i^*}^m - \mu_i^m) \\
    & = L +  \frac{1}{M}\sum_{i \neq i*}\sum_{m:\mu_{i^*}^m - \mu_i^m > 0}n_{m,i}(T)(\mu_{i^*}^m - \mu_i^m) .
\end{align*}
where the second inequality uses the fact that $\sum_{m:\mu_{i^*}^m - \mu_i^m \leq 0}n_{m,i}(T)(\mu_{i^*}^m - \mu_i^m) \leq 0$ holds for any arm $i$ and the last equality is true since $n_{m,i}(T)(\mu_{i^*}^m - \mu_i^m) = 0$ for $i = i^*$ and any $m$.

Meanwhile, by the choices of $\delta$ such that $\delta < c = f(\epsilon,M,T)$, we apply Proposition~\ref{prop:n} which leads to for any agent $m$ and arm $i \neq i^*$, 

As a result, the upper bound on $R_T$ can be derived as by taking the conditional expectation over $R_T$ on $A_{\epsilon, \delta}$
\begin{align}
    & E[R_T| A_{\epsilon, \delta}] \notag \\
    & \leq L +  \sum_{i \neq i^*}\sum_{j=1}^{C}\sum_{m \in c_j}E[n_{m,i}(T)|A_{\epsilon, \delta}](\mu_{i^*}^m - \mu_i^m) \label{eq:er_a_sub} \notag \\
    & \leq \sum_{i \neq i^*}\sum_{j=1}^{C}\sum_{m \in c_j}\max{\{[\frac{4C_1\log T}{|C_m|\Delta_i^2}], 2(K^2+MK+M) \}} +  \frac{2\pi^2}{3} + K^2 + (2M-1)K + l \notag \\
    & \leq \sum_{i \neq i^*}\sum_{j=1}^{C}\max{\{[\frac{4C_1\log T}{\Delta_i^2}], 2(K^2+MK+M) \}} +  \frac{2\pi^2}{3} + K^2 + (2M-1)K + l \notag \\
    & \leq \sum_{i \neq i^*}C\max{\{[\frac{4C_1\log T}{\Delta_i^2}], 2(K^2+MK+M) \}} +  \frac{2\pi^2}{3} + K^2 + (2M-1)K + l
\end{align}
where the second inequality holds by plugging in~(\ref{eq:en_bound}).

%\hl{$n_{j,i}(t)$ relies on the sample complexity of the homogeneous case. }

%\hl{$n_{j,i}(t) = n_{m,i}(t) + l|j|$}

Hence, the regret can be upper bounded by 
\begin{align*}
    & E[R_T|A_{\epsilon, \delta}] \\
    & \leq L +  \sum_{i \neq i^*}C(\Delta_i + 1)(\max{\{[\frac{C}{M} \cdot \frac{4C_1\log T}{\Delta_i^2}], 2(K^2+MK+M) \}} +  \frac{2\pi^2}{3} + K^2 + (2M-1)K) + l\\
    & = O(\max\{L,\log{T}\})
\end{align*}

This completes the proof of Theorem \ref{thm:6}.

\end{proof}

\end{document}